\documentclass{article}

\usepackage{amsmath,amsthm,bm,physics}
\usepackage{amssymb,graphicx,mathtools,setspace,mdframed,verbatim, dsfont,pifont,soul,wrapfig}
\usepackage[mathscr]{euscript}
\usepackage[dvipsnames]{xcolor} 
\usepackage{enumerate}
\usepackage[shortlabels]{enumitem}
\usepackage{titletoc}
\usepackage[page, header]{appendix}

\usepackage{tikz}
\usepackage{physics}
\usepackage{caption}

\usepackage{subcaption}
\allowdisplaybreaks
\usepackage{natbib}
\usepackage{geometry}
\usepackage{bold-extra}
\usepackage{nicefrac}
\usepackage[colorlinks=true, linkcolor=blue, urlcolor=black, citecolor=blue]{hyperref}
\usepackage{cleveref}
\usepackage{diy}
\usepackage{custom}

\usepackage[ruled,vlined]{algorithm2e}
\usepackage[normalem]{ulem}
\usepackage{adjustbox}
\makeatother
\usetikzlibrary{decorations.pathreplacing}

\definecolor{citeblue}{HTML}{2A6F97}
\definecolor{myorange}{RGB}{245,156,74}
\definecolor{citepurple}{RGB}{95,35,135}
\hypersetup{
  colorlinks=true,
  citecolor=citepurple,
  filecolor=blue,
  linkcolor=citepurple,
  urlcolor=blue
}

\crefname{definition}{Definition}{Definitions}
\Crefname{definition}{Definition}{Definitions}
\crefname{assumption}{Assumption}{Assumptions}
\Crefname{assumption}{Assumption}{Assumptions}
\crefname{equation}{Eq.}{Eqs.}
\Crefname{equation}{Equation}{Equations}
\crefname{section}{Section}{Sections}
\Crefname{section}{Section}{Sections}
\crefname{subsection}{Section}{Sections} % Make subsection refs look like Section X.Y
\Crefname{subsection}{Section}{Sections}
\crefname{figure}{Fig.}{Figs.}
\Crefname{figure}{Figure}{Figures}
\crefname{table}{Table}{Tables}
\Crefname{table}{Table}{Tables}
\crefname{align}{Eq.}{Eqs.}
\Crefname{align}{Equation}{Equations}

\newtheorem{theorem}{Theorem}[section]
\newtheorem{lemma}{Lemma}[section]
\newtheorem{induction}{Induction}[section]
\newtheorem{corollary}{Corollary}[section]

\newtheorem{remark}{Remark}[section]
\theoremstyle{definition}
\newtheorem{definition}{Definition}[section]
\newtheorem{assumption}{Assumption}[section]
\newtheorem{proposition}{Proposition}[section]
\newtheorem{fact}{Fact}[section]

\usepackage[textsize=tiny]{todonotes}

\title{On the Emergence of Implicit Curriculum in RLVR Learning Dynamics}
\author{
  \begin{tabular}{c}
  Yu Huang\footnote{Equal contributions.} \thanks{Department of Statistics and Data Science, Wharton School, University of Pennsylvania.} \hspace{1.5em}
    Zixin Wen\footnotemark[1] \thanks{ Machine Learning Department, Carnegie Mellon University.}\hspace{1.5em}
        Yuejie Chi\thanks{  Department of Statistics and Data Science, Yale University.}  \\[1em]
        Yuting Wei\footnotemark[2] \hspace{1.5em}
    Aarti Singh\footnotemark[3] \hspace{1.5em}
    Yingbin Liang\thanks{Department of Electrical and Computer Engineering, The Ohio State University.}\hspace{1.5em}
    Yuxin Chen\footnotemark[2] \vspace{0.5em}
  \end{tabular}
}

\date{\today}
\begin{document}
\maketitle

\begin{abstract}
  Reinforcement learning with verifiable rewards (RLVR) has been a main driver of recent breakthroughs in large reasoning models. Yet it remains a mystery how rewards based solely on final outcomes can help overcome the long-horizon barrier to extended reasoning. To understand this, we develop a theory of the training dynamics of RLVR for transformers on compositional reasoning tasks. Our theory shows that mixed-difficulty training naturally induces an \textbf{implicit curriculum}: without any explicit schedule, easier problems become learnable first and shape the frontier for harder ones, creating a learning progression from easy to hard during optimization. The effectiveness of this curriculum is governed by the smoothness of the difficulty spectrum. When the spectrum is smooth, training dynamics enter a well-behaved \textit{relay} regime, in which persistent gradient signals on easier problems make slightly harder ones tractable and keep training at the edge of competence. When the spectrum contains abrupt discontinuities, training undergoes grokking-type phase transitions with prolonged plateaus before progress recurs. As a technical contribution, our analysis develops and adapts techniques from Fourier analysis on finite groups to our setting. We validate the predicted mechanisms empirically via controlled synthetic experiments and real-model RLVR runs.
\end{abstract}

\section{Introduction}\label{sec:intro}

Large language models (LLMs) such as OpenAI-o3~\citep{openai2024o1card} and DeepSeek~\citep{deepseekai2025deepseekr1} %, and Kimi k1.5~\citep{Team2025KimiKS} 
have shown striking performance in complex reasoning tasks. A key enabler of this recent progress is reinforcement learning with verifiable rewards (RLVR)~\citep{shao2024deepseekmath, Lambert2024TLU3P, gao2024designing}, which fine-tunes pre-trained base models via reinforcement learning (RL) using automatically verifiable, outcome-based feedback, such as a binary signal indicating whether the final answer is correct or not.

This raises an immediate question: if RLVR relies on outcome-based feedback only at the end of a reasoning trajectory, how can such a sparse reward mechanism drive effective learning on long-horizon problems? As the horizon grows, RL algorithms encounter an inherent search barrier, because useful signals are buried within an exponentially expanding space of trajectories. While recent studies have sought to understand the mechanism of RLVR~\citep{yeo2025demystifying, wu2025invisible, Yue2025DoesRL, Sun2025RLGR, yuan2025f, wen2025reinforcement}, existing findings often provide mixed, inconclusive results across different tasks and setups. It remains unclear under what conditions outcome-only rewards are sufficient to enable effective RL. 

A recent controlled study~\citep{zhang2025interplay} offered an important insight: RLVR is most effective near the model's \emph{edge of competence}, where problems are neither already solved nor impossible. They conducted experiments showing that RLVR improves model's reasoning performance the most when the data sits right above the model's capability frontier, while other difficulty levels induce early plateaus. The work \citep{zhang2025interplay} suggests picking problems close to the model's capability frontier, but did not answer why RLVR in practice can work on the general mixture of problems. In practice, RLVR is not done on a single fixed difficulty level but on a spectrum of problems, in which some are already learnable and others still beyond reach \citep{shao2024deepseekmath, deepseekai2025deepseekr1}. This motivates the following question:
\vspace{5pt}
\begin{center}
    \it How does the difficulty spectrum of data shape RLVR learning dynamics?
\end{center}
\vspace{5pt}

To address this question, we study how policy-gradient dynamics evolve across controlled difficulty spectra. We adopt a minimal transformer model~\citep{Vaswani2017AttentionIA}, which is the backbone architecture of most LLMs. It consists of a softmax-based attention layer followed by a multilayer perceptron (MLP) layer. We freeze the MLP layer in a way that perfectly implements the atomic operation, modeling the regime where the model already possesses the requisite atomic skills and RLVR only needs to learn how to compose these skills~\citep{yuan2025f}. We study RL training on this task under outcome-based rewards via the standard policy gradient algorithm \textsf{REINFORCE}~\citep{williams1992simple}. Within this setting, we track the learning dynamics of the transformer model and identify the factors that govern progress across increasing horizons. Our main answer is that mixed-difficulty training naturally gives rise to an \emph{implicit curriculum} (illustrated in \Cref{fig:implicit-curriculum-qwen}): if the difficulty spectrum is smooth enough, the gradient signal is relayed from easier problems to slightly harder ones, keeping optimization near the edge of competence; if the spectrum contains large jumps, the same curriculum develops delayed handoffs and training exhibits grokking-like plateaus before abrupt progress. An overview of our main contributions is provided below.
% To address this question, we study how policy-gradient dynamics evolve across controlled difficulty spectra. Our main answer is that mixed-difficulty training naturally gives rise to an \emph{implicit curriculum}: if the difficulty spectrum is smooth enough, gradient signal is relayed from easier problems to slightly harder ones, keeping optimization near the edge of competence; if the spectrum contains large jumps, the same curriculum develops delayed handoffs and training exhibits grokking-like plateaus before abrupt progress. \Cref{fig:implicit-curriculum-qwen} illustrates this phenomenon in real-size models.

\begin{figure*}[t]
    \centering
    \includegraphics[width=0.85\textwidth]{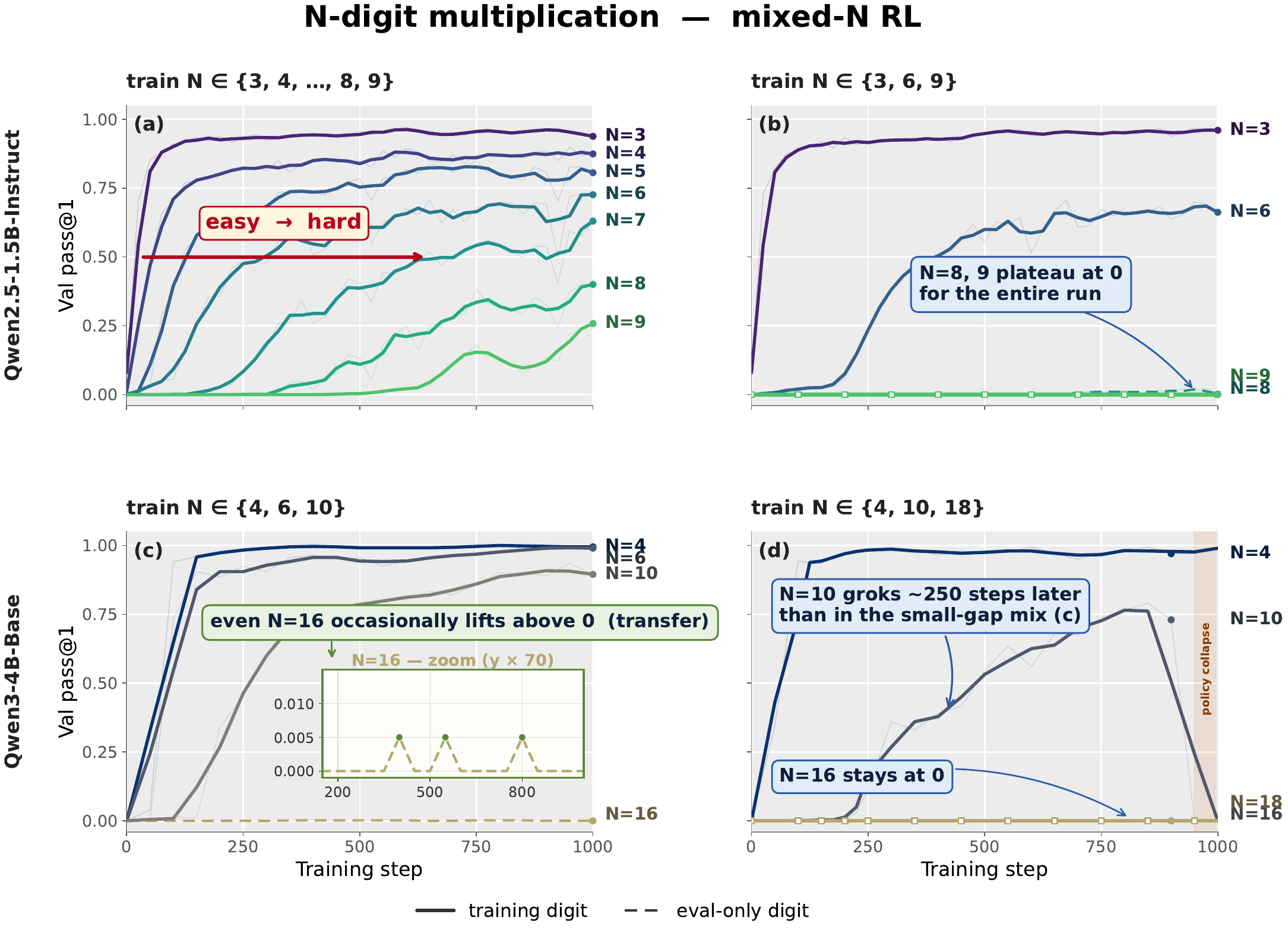}
    \caption{\small \textbf{Mixed-difficulty RL induces an implicit curriculum.}
We train on different digit mixtures of $N$-digit by $N$-digit multiplication tasks, where larger $N$ indicates greater difficulty. Without any explicit schedule, easier digits become learnable first and then help unlock harder digits. A smooth ladder of difficulties yields a steady easy-to-hard progression, whereas sparse anchors leave harder digits silent for long periods before delayed jumps or persistent plateaus. Detailed setup and discussion are given in \Cref{sec:experiments}.}
    \label{fig:implicit-curriculum-qwen}
\end{figure*}

% We formalize this mechanism in a minimal transformer model~\citep{Vaswani2017AttentionIA}, which is the backbone architecture of most LLMs. It consists of a softmax-based attention layer followed by a multilayer perceptron (MLP) layer. We freeze the MLP layer in a way that perfectly implements the atomic operation, modeling the regime where the model already possesses the requisite atomic skills and RLVR only needs to learn how to compose these skills~\citep{yuan2025f}. We study RL training on this task under outcome-based rewards via the standard policy gradient algorithm \textsf{REINFORCE}~\citep{williams1992simple}. Within this setting, we track the learning dynamics of the transformer model and identify the factors that govern progress across increasing horizons. An overview of our main contributions is provided below. 
\begin{enumerate}

\item \textbf{A comparative study between short-horizon learning and long-horizon barriers.} We show that with outcome-based rewards, \textsf{REINFORCE}-style policy gradient algorithms provably learn short-horizon compositions. Meanwhile, beyond a critical horizon, the gradient field at initialization becomes exponentially flat, indicating an optimization barrier for near-random policies. In contrast, supervised fine-tuning (SFT) can provably learn beyond this critical horizon by providing intermediate feedback for sequential compositional reasoning.

\item \textbf{A theoretical characterization of the learning dynamics in RLVR.} On an easy-to-hard mixture over horizons, we establish polynomial-time convergence guarantees for outcome-based RL training. We show that an \emph{implicit curriculum} emerges in RL: in the smooth regime, progress transfers across adjacent horizons through a {\it relay effect}, which keeps training near the edge of competence and yields steady progress toward harder problems; when the spectrum contains large discontinuities, the same curriculum instead exhibits delayed progress, long plateaus, and abrupt, grokking-like phase transitions~\citep{Sun2025RLGR}.

\item \textbf{Novel techniques from Fourier analysis on groups.} We introduce a Fourier analysis~\citep{terras1999fourier} framework that transforms the problem of trajectory-level success conditioning into tractable calculations based on Fourier analysis for convolutions of measures. Our new framework allows us to compute the magnitudes of policy gradients in long-horizon group composition problems by resorting to the spectral properties of the group representations, which greatly simplifies the analysis of combinatorial event probabilities.
 \end{enumerate}

\paragraph{The implicit curriculum: induced learning dynamics over difficulty spectra.}
Collectively, our results provide an optimization-based explanation for how RLVR improves model's capability frontier. The high-level phenomenon is an \emph{implicit curriculum}: when the training distribution displays a spectrum of difficulties, optimization naturally progresses from already-learnable problems to slightly harder ones without any explicit schedule. When the difficulty spectrum is smooth, the underlying mechanism is a \textit{relay effect}: easier horizons continue to provide non-vanishing gradient signal while the next harder horizon is becoming learnable, so gradient dominance is handed off before training stalls. When the gap between adjacent difficulty levels is too large, i.e., when the spectrum is non-smooth, the same curriculum becomes delayed and unpredictable: the handoff arrives only after a prolonged plateau, followed by a grokking-style jump. Thus, we conclude that RLVR's effectiveness is governed not only by exploration efficiency, but also by whether the difficulty spectrum supports a smooth learning progression. \Cref{fig:two-panels} depicts the contrast between these two regimes.

\newcommand{\SubFigBox}[1]{%
  \begin{adjustbox}{width=\linewidth,height=0.78\linewidth,keepaspectratio,center}%
    #1%
  \end{adjustbox}%
}

\begin{figure}[t]
  \centering
  \begin{subfigure}[t]{0.32\textwidth}
    \centering
    \SubFigBox{
    \begin{tikzpicture}
    \path[use as bounding box] (-0.35,-0.30) rectangle (10.15,4.15);
    
    % ---- parameters (tune these)
    \def\A{3.2}      % peak height
    \def\xzero{2.5}  % phase shift (sets where peaks land)
    \def\w{45}       % frequency in degrees per x (zeros spaced by 4 when w=45)
    \def\n{2.5}       % sharpness / flatness at peaks (bigger => flatter peaks, sharper transitions)
    \def\m{7}        % valley-bottom flatness (bigger => flatter valleys)
    
    % axes
    \draw[->, black] (-0.2,0) -- (10.0,0) node[below] {$t$};
    \draw[->, black] (0,-0.2) -- (0,4.0) node[left] {$\dfrac{\mathrm{d}r}{\mathrm{d}t}$};
    
    % curve (smooth but with flat peaks + sharp transitions)
    \draw[blue!75!black, very thick, domain=0:10, samples=900]
      plot (\x,{
        \A*pow(
          max(0, 1 - pow(abs(sin(\w*(\x-\xzero))), \n)),
          \m
        )
      });
    
    % key x-locations (for w=45, peaks at xzero+4k, valleys at xzero+2+4k)
    \foreach \xp/\lab in {2.5/peak, 6.5/peak}{
      \fill (\xp,\A) circle (1.6pt);
      \draw[densely dashed, gray!60] (\xp ,0) -- (\xp,\A);
      \node[font=\large, above=2pt] at (\xp,\A) {\lab};
    }

    \draw[densely dashed, gray!60] (0,\A) -- (10,\A)
  node[pos=0.98, above, font=\small, black] {};
    
    \foreach \xv/\lab in {4.5/plateau, 8.5/plateau}{
      \fill (\xv,0) circle (1.6pt);
      % \draw[densely dashed, gray!60] (\xv,0) -- (\xv,0.9);
      \draw[|-|, black!70] (\xv-0.9,-0.18) -- (\xv + 0.9,-0.18) node[midway, below=4pt, font=\large, black] {};
      \node[font=\large, above=2pt] at (\xv,0) {\lab};
    }
    % \def\xL{3.8}
% \def\xR{5.2}
% \draw[<->, gray!70, thick] (\xL,-0.18) -- (\xR,-0.18)
%   node[midway, below=4pt, font=\large, black] {};
    % brace for intermediate basin
    % \draw[decorate,decoration={brace,amplitude=4pt},gray!60]
    %   (5.5,-0.18) -- (7.5,-0.18)
    %   node[midway,below=6pt,font=\large,black!90] {};
    
    \end{tikzpicture}
    }
    \caption{\small If the difficulty ratio \(R = L_{k+1}/L_k \gg 1\), learning exhibits phase transitions in between difficulty levels.}
    \label{fig:left}
  \end{subfigure}\hfill
  \begin{subfigure}[t]{0.32\textwidth}
    \centering
    \SubFigBox{
    \begin{tikzpicture}
    \path[use as bounding box] (-0.35,-0.30) rectangle (10.15,4.15);
    
    % ---- parameters (tune these)
    \def\A{3.2}       % peak height
    \def\B{1.9}       % NEW: elevated valley level (min after warm-up)
    \def\xzero{2.5}   % phase shift (sets where peaks land)
    \def\w{90}        % frequency in degrees per x (zeros spaced by 4 when w=45)
    
    \def\n{3}         % sharper valleys + flatter peaks (increase for sharper/narrower dips)
    \def\m{1.2}       % reduce to avoid flat valley bottoms (closer to 1 => pointier valleys)
    
    \def\tstart{1.1}  % NEW: stay at exactly 0 until this t
    \def\trise{0.9}   % NEW: smooth ramp length after tstart
    
    % axes
    \draw[->, black!90] (-0.2,0) -- (10.0,0) node[below] {$t$};
    \draw[->, black!90] (0,-0.2) -- (0,4.0) node[left] {$\dfrac{\mathrm{d}r}{\mathrm{d}t}$};
    
    % curve: flat peaks, elevated sharper valleys, AND initial zero plateau
    \draw[green!60!black, very thick, domain=0:10, samples=900]
      plot (\x,{
        % u = clamp((x-tstart)/trise, 0, 1)
        ( % gate(u) = 6u^5 - 15u^4 + 10u^3
          (
            % define u inline
            (min(1, max(0, (\x-\tstart)/\trise)))
          )^3
          * (10 - 15*(min(1, max(0, (\x-\tstart)/\trise))) + 6*(min(1, max(0, (\x-\tstart)/\trise)))^2)
        )
        *
        (
          \B + (\A-\B) * pow(
            max(0, 1 - pow(abs(sin(\w*(\x-\xzero))), \n)),
            \m
          )
        )
      });
    
    % peaks
    \foreach \xp/\lab in {2.5/peak, 4.5/peak, 6.5/peak}{
      % \draw[densely dashed, gray!60] (\xp - 0.2 ,0) -- (\xp - 0.2,\A-0.1);
      % \draw[densely dashed, gray!60] (\xp + 0.2 ,0) -- (\xp + 0.2,\A-0.1);
      \node[font=\large, above=2pt] at (\xp,\A) {\lab};
    }
    
    % valleys (lifted after first peak)
    \foreach \xv/\lab in {3.5/relay, 5.6/relay, 7.7/relay}{
      \node[font=\large, below=1pt] at (\xv,\B) {\lab};
    }

    \draw[densely dashed, gray!60] (0,\A) -- (10,\A)
  node[pos=0.98, above, font=\small, black] {};

    \draw[densely dashed, gray!60] (0,\B) -- (10,\B)
  node[pos=0.98, below, font=\small, black] {};

    \end{tikzpicture}
    }
    \caption{\small If the difficulty ratio \(R = L_{k+1}/L_k = O(1)\), the implicit curriculum proceeds smoothly across adjacent difficulty levels via relay from \(L_1\) to \(L_{\max}\).}
    \label{fig:right}
  \end{subfigure}
\hfill
\begin{subfigure}[t]{0.32\textwidth}
  \centering
  \SubFigBox{
  \begin{tikzpicture}
  \path[use as bounding box] (-0.35,-0.30) rectangle (10.15,4.15);
    % ---- safe smooth step: S(x;c,W) goes 0->1 over width W centered near c
    % u = clamp((x-(c-W/2))/W, 0, 1),  S = 3u^2 - 2u^3
    \pgfmathdeclarefunction{S}{3}{%
      \pgfmathsetmacro{\xx}{#1}%
      \pgfmathsetmacro{\cc}{#2}%
      \pgfmathsetmacro{\WW}{#3}%
      \pgfmathsetmacro{\uu}{min(1,max(0,(\xx-(\cc-\WW/2))/\WW))}%
      \pgfmathparse{ (3*\uu*\uu - 2*\uu*\uu*\uu) }%
    }

    % axes
    \draw[->, black] (-0.2,0) -- (10.0,0) node[below] {$t$};
    \draw[->, black] (0,-0.15) -- (0,4.0) node[left] {$r(t)$};

    \def\C{1.8}
    % ---- (a) reward: step-like with plateaus (integral of pulsed dr/dt)
    % choose widths that look nice and are stable
    \def\Wstep{0.9}
    \draw[very thick, blue!75!black, domain=0:10, samples=800]
      plot (\x,{ \C* (
        0.45*S(\x,2.5,\Wstep)
      + 0.4*S(\x,6.5,\Wstep)
      + 0.4*S(\x,10,\Wstep) )
      });

    % ---- (b) reward: smooth monotone growth with small oscillations
    % warm-up gate * (linear trend + small sine wiggle)
    \def\Wgate{1.2}
    \draw[very thick, green!60!black, domain=0:10, samples=800]
      plot (\x,{ 
        \C * S(\x,1.2,\Wgate) * (0.19*\x + 0.05*sin(180*(\x-1.5)))
      });

    % legend
    \begin{scope}[shift={(0.3,2.6)}]
      \draw[blue!75!black, very thick] (0,0) -- (0.9,0)
        node[right, black, font=\large] {grokking (a)};
      \draw[green!60!black, very thick] (0,-0.45) -- (0.9,-0.45)
        node[right, black, font=\large] {relay (b)};
    \end{scope}
  \end{tikzpicture}
  }
  \caption{\small Reward trajectories $r(t)$ induced by different ratios \(R\) in (a) the delayed-handoff grokking regime and (b) the smooth implicit-curriculum relay regime.}
  \label{fig:reward-curves}
\end{subfigure}
\caption{\small \textbf{Reward dynamics in mixed-difficulty RL.}
A schematic illustration of the reward growth rate $\mathrm{d}r/\mathrm{d}t$ and \(r(t)\) for mixed-difficulty RL, showing how the difficulty ratio $R=L_{k+1}/L_k$ determines whether the implicit curriculum at the edge of competence proceeds smoothly or through delayed plateau-and-jump transitions. Small ratios yield smooth relay-based handoffs across horizons, while large ratios produce delayed handoffs and grokking-type phase transitions.}
%\vspace{-0.2cm}
%   \caption{\textbf{Schematic comparison of reward dynamics under mixed difficulty.}
% Illustration of the (qualitative) reward growth rate $\mathrm{d}r/\mathrm{d}t$ over training, highlighting how the difficulty ratio $R=L_{k+1}/L_k$ controls whether optimization exhibits long stalls grokking or smooth
% relay.}
  %\caption{Overall caption for the two panels.}
  \label{fig:two-panels}
\end{figure}

\section{Related Work}

\paragraph{Empirical understanding of RLVR mechanisms.} With RLVR’s recent success and apparent scalability, there has been growing interest in understanding its mechanics, sparking an active debate: \emph{what does RLVR actually teach base LLMs beyond pre-training?} The existing literature provides mixed evidence. Some works characterize RL primarily as a capability refiner or reranker~\citep{yeo2025demystifying, wu2025invisible, Yue2025DoesRL, Zhao2025EchoCR}, whereas others argue that RL can induce substantial reasoning gains beyond the base model~\citep{Sun2025RLGR, yuan2025f, wen2025reinforcement, liu2025prorl}. Related findings suggest that RLVR can yield notable improvements even from spurious rewards~\citep{Shao2025SpuriousRR} or extremely limited RL data (e.g., one-shot)~\citep{Wang2025ReinforcementLF}. Complementary evidence further highlights an entropy-based mechanism as a potential driver of such gains~\citep{cui2025entropy, Wang2025BeyondT8}. Recently, \cite{zhang2025interplay} helps reconcile these views by showing that genuine capability gains arise mainly when there is sufficient headroom beyond pre-training and the RL data are calibrated to the model’s edge of competence.

\paragraph{Theory of RL training for LLMs.}  The empirical success of RLVR has spurred theoretical studies from various perspectives~\citep{chen2025coverage, Chen2025ReshapingRI, zhu2025path, ranmilo2026outcomebasedrp, Rad2026RateOF, Lyu2025TransformersWR, Bu2025ProvableBO, Tsilivis2025HowRL,davis2025objective}. Of particular relevance are gradient-based analyses with transformer policies~\citep{
Lyu2025TransformersWR,Bu2025ProvableBO,ranmilo2026outcomebasedrp}. \cite{Bu2025ProvableBO} formalize the benefits of curriculum-style RL post-training, while \cite{Lyu2025TransformersWR} studies learnability in settings with intermediate supervision. In contrast, we focus on outcome-based RL where no dense feedback is available. A closely related work~\citep{ranmilo2026outcomebasedrp} shows that outcome-based RL can induce step-by-step reasoning under gradient flow, with asymptotic guarantees under static data conditions. We go beyond this by analyzing realistic gradient dynamics along the full learning trajectory and characterizing phase transitions in both reward and gradient.

\paragraph{Curriculum and difficulty-aware RL.} Curriculum learning in RL studies how sequencing tasks or samples can make hard problems learnable~\citep{narvekar2020curriculum}. Recent LLM reasoning work uses easy-to-hard schedules~\citep{parashar2025curriculum}, learned curriculum policies~\citep{chen2025self}, and online filtering or frontier-of-learnability sampling to focus training on intermediate-difficulty problems with useful success signals~\citep{bae2025online,foster2025learning}. Since generalization across difficulty levels is limited~\citep{kordi2026revisiting}, broad and smooth difficulty spectra can be preferable to relying only on easy or hard subsets. Related hard-problem exploration methods reduce effective difficulty through privileged or off-policy prefixes, as in POPE~\citep{qu2026pope} and PrefixRL~\citep{setlur2026reuse}. These works support the real-world relevance of our mechanism: sparse terminal rewards become useful when data or prompting keeps training near the current competence frontier.

\paragraph{Grokking in supervised learning and RL.} Grokking characterizes delayed generalization in supervised learning,
where performance stays flat for long periods before improving abruptly~\citep{Power2022GrokkingGB}. Similar plateau-to-jump dynamics also appear in RLVR training~\citep{Sun2025RLGR}, and have been described as ``aha moments'' in RL systems such as DeepSeek-R1-Zero~\citep{deepseekai2025deepseekr1}. Prior work investigates mechanisms behind such phase changes~\citep{nanda2023progress,kumar2024grokking,liu2022towards,tian2025provable}, while our theory on the learning dynamics provides a mechanistic explanation for both the plateau and the subsequent transition in RLVR. 

\section{Problem Setup}
\label{sec:problem-setup}
In this section, we present the formal setup for our theoretical investigation. We first define the compositional reasoning task, and then describe a minimal transformer architecture as well as the policy gradient objective used to study RL training dynamics.

\paragraph{Notation.} For functions $h,g$, write $h(x)=\Omega(g(x))$ (resp.\ $O(g(x))$) if there exist universal constants
$C>0$ and $a$ such that $|h(x)|\ge C|g(x)|$ (resp.\ $\le C|g(x)|$) for all $x\ge a$; write
$h(x)=\Theta(g(x))$ if both bounds hold. We write $h(x)=o(g(x))$ if $\lim_{x\to\infty}\frac{h(x)}{g(x)}=0$, and $h(x)=\omega(g(x))$ if
$\lim_{x\to\infty}\frac{h(x)}{g(x)}=\infty$.
Let $\1\{\cdot\}$ be the indicator and $[L]=\{1,\dots,L\}$. We use $\widetilde{O}$,  $\widetilde{\Theta}$ and $\widetilde{\Omega}$ to suppress logarithmic factors, and use $\poly(d)$ and $\polylog(d)$ for polynomials in $d$ and $\log d$, respectively.

\subsection{Compositional Reasoning}\label{sec-comp-task}

To study the mechanistic challenges of multi-step reasoning, we consider the state-tracking task that has received much recent attention~\citep{liu2023transformers,merrill2024illusion,huang2025cot}. This setting serves as a simple example of compositional reasoning: while each individual step is computationally simple, the task requires the precise sequential composition of $L$ transitions in order to compute the final result.

\begin{definition}[$L$-step compositional reasoning]
Let \(\mathcal Y\) be a finite set (the \emph{state space})  and \(\mathcal G\)  a finite group,
acting on \(\mathcal Y\) via \((g,y)\mapsto g (y)\). For any initial state \(y_0\in\mathcal Y\) and sequence of transitions \(g_1,\dots,g_L\in\mathcal G\),
we can obtain a trajectory \(y_\ell = g_\ell (y_{\ell-1})
\) where \( \ell=1,\dots,L\). The goal is to predict the final state \(y_L\) given the sequence $(y_0, g_1, \dots, g_L)$.
We define $L$ as the \emph{length} or \emph{horizon} of the problem. 
\end{definition}

Our analysis rests on some structural assumptions on the group actions of \(\cG\) on \(\cY\), described below.

\begin{assumption}[Group structure and action]\label{assum-group}
    We assume $\cG$ is a finite non-abelian simple group that acts simply transitively on the set $\cY$, which implies that there is a bijective correspondence between group elements and states, such that $|\cG| = |\cY| = d$. We focus on the asymptotic regime where the state-space size tends to infinity, i.e., $d\rightarrow \infty$.
\end{assumption}

\begin{assumption}[Bounded-horizon collision sparsity]\label{assump-collision}
For every fixed horizon $L=O(1)$, let $g_1,\dots,g_L$ be sampled uniformly without replacement from $\cG$. Define
\begin{align}
\rho_L(d)
\triangleq
\sum_{\mathbf{i}=(i_1,\dots,i_L)\in[L]^L:\,\mathbf{i}\neq(1,\dots,L)}
\mathbb{P}\left(
g_{i_L}\circ \cdots \circ g_{i_1}
=
g_L\circ \cdots \circ g_1
\right).
\end{align}
We assume $\rho_L(d)=o(1)$ as $d=|\cG|\to\infty$ for every fixed $L$.
\end{assumption}

\begin{remark}\label{remark-group}
    The non-abelian condition in \Cref{assum-group} rules out the most degenerate source of path collisions, since abelian groups identify all orderings of the same multiset of operations. However, the outcome-based RL task also needs a mild quantitative separation between the intended trajectory and accidental shortcut trajectories. \Cref{assump-collision} captures this separation: for any fixed horizon, using the sampled prompt operators in a noncanonical order, or with repetitions, should almost never reach the same final state. Thus a terminal reward is, with high probability, evidence of following the intended composition rather than an accidental shortcut. Many standard group families satisfy this type of sparsity; see \Cref{rem:two-step-collisions}.
\end{remark}

Next, we specify the format of reasoning data and the distribution over such instances.

\begin{definition}[Reasoning problems]\label{def:reasoning-data}
Fix a set of positional identifiers $\cX\subset\R^{d}$ consisting of mutually
orthogonal vectors.
We encode each token as a position--symbol pair $(x,s)$ where
$x\in\cX$ and $s\in\cG\cup\cY$. A length-$L$ reasoning instance is $Z^L=(Z_p^L,Z_a^L)$ consisting of:
\begin{itemize}
\item \textit{problem description (prompt):} a sequence of $L$ transition tokens:
\[
Z_p^L
=
\big((x_{p,1},g_1),(x_{p,2},g_2),\dots,(x_{p,L},g_L)\big),
\]
where $\{x_{p,\ell}\}_{1\leq \ell \leq L}\subset \cX$ are distinct and $\{g_\ell\}_{1\leq \ell \leq L} \subset \cG$.

\item \textit{compositional solution:} a sequence of
$(L+1)$ state tokens:
\[
Z_a^L
=
\big((x_{a,0},y_0),(x_{a,1},y_1),\dots,(x_{a,L},y_L)\big),
\]
where $y_\ell=g_\ell(y_{\ell-1})$ for all $\ell\in[L]$.

\item \textit{position alignment:} the prompt and solution positions are related
by a fixed \textit{unknown} permutation $\mathfrak{s}:\cX\to\cX$ such that  \(x_{a,\ell-1}=\mathfrak{s}(x_{p,\ell})\), \(\forall \ell\in[L]\).
\end{itemize}
\end{definition}
Since the network we will use is a transformer (\cref{sec-transformer}),
token positions carry no intrinsic meaning in the absence of absolute
positional information. Hence, the prompt/solution format encodes only a relative correspondence and
does not provide any positional shortcut.  Now we are ready to define the data distribution for the reasoning problems.

\begin{definition}[Data distribution $\mathcal{D}^{L}$] \label{def:data-distribution}

Given a problem length $L$, we sample a reasoning instance $Z^L = (Z_p^L, Z_a^L)$ as in \Cref{def:reasoning-data} by the following process:
\begin{enumerate}
    \item $g_1, \dots, g_L$ is sampled from $\mathcal{G}$ uniformly without replacement;
    \item the  initial state $y_0$ is sampled uniformly at random from the set $\mathcal{Y}$;
    \item we sample distinct prompt identifiers $\{x_{p,1},\dots,x_{p,L}\}$
      uniformly from $\cX$;
    \item  Set $  x_{a, \ell-1} = \mathfrak{s}(x_{p, \ell})$ for all $\ell \in \{1, \dots, L\}.$ Additionally, $x_{a,L}$ is  sampled from  $\cX\setminus\{x_{a,k}\}_{k=0}^{L-1}$. 
    \item the intermediate states are computed via the group action: $y_\ell = g_\ell(y_{\ell-1})$.
\end{enumerate}
An instance $Z^L \triangleq (Z_p^L, Z_a^L)$ generated by the above procedure is said to be sampled from  $\mathcal{D}^{L}$. Since both $\{g_\ell\}_{\ell=1}^L$ and $\{x_{p,\ell}\}_{\ell=1}^L$ are sampled
without replacement, the length is bounded by
$L\le \min\{|\cX|-1,d\}$.
\end{definition}

\begin{assumption}\label{assump-x}
    We assume $|\cX|=\Theta(d^{c_x})$ for some constant \(
    c_x \in (0.1,1)\). We denote by $ L_{\max}=|\cX|-1$ 
    %as the unique symbols in \(\cX\) defines 
    the maximum problem length.
\end{assumption}

\paragraph{Embeddings and tokenizer.} We embed each group symbol $g\in\cG$ and each state symbol $y\in\cY$ into $\R^{2d}$ via a shared orthonormal map: $\mathsf{emb}:\cG\cup\cY\to\R^{2d}$.   With a slight abuse of notation, whenever a symbol $g\in\cG$ or $y\in\cY$
is used as a model input or appears in the network computation, we write the symbol itself to denote its embedding $\mathsf{emb}(g)$ or $\mathsf{emb}(y)$, respectively, whenever this causes no ambiguity.  In the sequel, we use $s\in\cG\cup\cY$ to refer generically to a symbol from either set.
% We embed each symbol $s\in\cG\cup\cY$ into $\R^{2d}$ via an orthonormal map
% $\mathsf{emb}:\cG\cup\cY\to\R^{2d}$.
% With a slight abuse of notation, when a symbol $s$ is used as a model input
% (or appears inside network computation), we write $s$ to refer to its embedding
% $\mathsf{emb}(s)\in\R^{2d}$ as long as it is clear from the context. 
We also fix a bijective tokenizer $\tau:\cY\to[d]$ to index states for the
next-state prediction objective, so that predicting $y\in\cY$ is equivalent to predicting the class label $\tau(y)\in[d]$.

\subsection{Transformer Architecture}\label{sec-transformer}
Building on 
\Cref{def:reasoning-data,def:data-distribution}, we now define a simple transformer, aimed at predicting the next state in the solution trace given the prompt and the current state token.
\begin{definition}[One-layer transformer~\citep{Vaswani2017AttentionIA}] \label{def:transformer-arch}
    We consider a simplified transformer consisting of a single attention module followed by a one-hidden-layer MLP.  For a prompt
$Z_p^L=((x_{p,\ell},g_\ell))_{\ell=1}^L$
and the current solution token
$Z_{a,k}=(x_{a,k},y_k)$,  the  transformer with parameter  $\theta = (W, Q)$ outputs an (unnormalized) score vector  over the next-state index in $[d]$:  
    \begin{align}
        \label{eq:defn-TF}
        \mathsf{TF}_{\theta}(Z_{a,k}, Z_p^L) = \mathsf{MLP}_W\left( \mathsf{Attn}_{Q}(Z_{a,k}, Z_p^L) \right)\in \mathbb{R}^d,
    \end{align}   
where the operators $\mathsf{Attn}_{Q}$ and $\mathsf{MLP}_{W}$ are defined shortly in \cref{eq-attn-layer} and \cref{eq:defn-MLP-layer}, respectively.  
\end{definition}
    
    \paragraph{Attention Layer.} The attention module uses the current solution position $x_{a,k}$ to form weights
over the prompt positions $\{x_{p,\ell}\}_{\ell=1}^L$, and returns a vector that aggregates the transition embeddings. Specifically,  given a query weight $Q \in \RR^{d \times d}$, we define
\begin{subequations}
\label{eq-attn-layer}
    \begin{align}
    \textstyle    \mathsf{Attn}_Q(Z_{a,k}, Z_p^L) \triangleq \frac{1}{2}\Big( y_k+ \sum_{\ell=1}^L \attn_{a,k \to p,\ell}(Z_{a,k}, Z_p^L) \cdot g_\ell\Big) \in \RR^{2d}, 
    \end{align}
    where the attention weight $\attn_{a,k \to p,\ell}(Z_{a,k}, Z_p^L) $ is obtained by softmax-normalizing the inner
products $\langle Qx_{a,k},x_{p,\ell}\rangle$:
    \begin{align} 
        \mathsf{softmax}\Big(
\big(\langle Qx_{a,k},x_{p,1}\rangle,\dots,\langle Qx_{a,k},x_{p,L}\rangle \big)
\Big)_\ell.
        \label{eq-attn-score}
    \end{align}
    \end{subequations}
    Here, for any $u\in\R^{n}$, we denote the $i$-th entry of the softmax output as $\mathsf{softmax}(u)_i\triangleq \frac{\exp(u_i)}{\sum_{j=1}^{n}\exp(u_j)}$.  In standard transformer architectures, the score typically takes the form
$\langle W_Q x_{a,k}, W_K x_{p,\ell}\rangle$ rather than using a single matrix $Q$.
We adopt the equivalent reparameterization commonly used in theoretical studies to simplify analysis without
changing expressivity~\citep{huang2023context,yang2024context,zhang2024trained}. The factor $1/2$ normalizes the combined contribution of the residual term $y_k$ and the attention output.

\begin{remark}
We compute attention scores using only positional identifiers
$x_{a,k},x_{p,\ell}$, a standard simplification in theoretical analyses of transformers~\citep{jelassi2022vision,huang2025a,wen2025sparse,kim2025transformersprovably,cheng2026demystifying}. This decoupling separates token association from reasoning with vector embeddings (states and transitions), which will be carried out by the subsequent MLP computation.  Since the attention scores depend only on $(x_{a,k},x_{p,1:L})$, we often suppress the arguments and write $\attn_{a,k \to p,\ell}$ instead of $\attn_{a,k \to p,\ell}(Z_{a,k}, Z_p^L) $ when it is clear from the context.
\end{remark}

\paragraph{MLP layer.}
Given the attention output in $\R^{2d}$, the MLP operator with parameter $W$ maps it to logits in $\R^{d}$
for next-state prediction (indexed by $\tau:\cY\to[d]$). With $m$ hidden units and ReLU activation $\sigma(z)=\max\{0,z\}$, for each $j\in[d]$,
\begin{align}
\label{eq:defn-MLP-layer}
\textstyle\bigl[\mathsf{MLP}_W(Z)\bigr]_j
=
\sum_{r=1}^m \sigma\left(\langle W_{j,r}, Z\rangle\right),
\quad
W_{j,r}\in\RR^{2d}.
\end{align}
In the present work, we keep $W$ fixed and assume that the MLP has already
acquired pre-trained \emph{atomic skills} for one-step transitions.
Specifically, given a state $y$ and a transition symbol $g$ appearing in the
form $\frac{1}{2}(g + y)$, the MLP implements the map $y \mapsto g(y)$.
The structural details of the MLP and the existence of such an
implementation are deferred to \Cref{sec-atom}. Accordingly, our focus is on how the attention module enables
long-horizon composition once the model already possesses this
one-step atomic skill. This perspective aligns with recent RLVR studies that investigate compositional reasoning beyond single-step transitions~\citep{yuan2025f,Park2025HowDR}.

\paragraph{Induced next-state distribution.}
Given the transformer's output
\(\mathsf{TF}_\theta(Z_{a,k},Z_p^L)\), we define
the induced next-state distribution (policy) by softmax normalization: 
\begin{align}
&\pi_\theta\big(j \mid Z_{a,k}, Z_p^L\big)
\triangleq
\mathsf{softmax}\Big(\mathsf{TF}_\theta(Z_{a,k}, Z_p^L)\Big)_j
\end{align}
for all $j\in[d]$.  Equivalently, we write 
$\pi_\theta(y\mid Z_{a,k}, Z_p^L)\triangleq
\pi_\theta(\tau(y)\mid Z_{a,k}, Z_p^L)$ for $y\in\cY$.
In our reasoning format, the answer positions $(x_{a,1},\dots,x_{a,L})$ are part
of the instance and are not predicted; the model only predicts the next state
symbol at each step.
Starting from the prefix $Z^{L,0}=[Z_{a,0},Z_p^L]$, the induced distribution
over the generated state sequence $\hat y^L=(\hat y_1,\dots,\hat y_L)$
factorizes autoregressively as
\begin{align}
    \pi_\theta(\hat y^L \mid Z^{L,0})
&=
\prod_{k=0}^{L-1}
\pi_\theta(\hat y_{k+1}\mid \hat Z_{a,k}, Z_p^L),\label{eq-policy}
\\
\text{ where }& \hat Z_{a,k}=(x_{a,k},\hat y_k),\ \hat y_0=y_0. \notag
\end{align}
When it is clear from the context, we abbreviate the conditioning as $(y_0, G^L)$, where
$G^L=(g_1,\dots,g_L)$.
Formally, $\pi_\theta^L(\cdot\mid y_0,G^L)$ still conditions on the full
instance $Z^L$ (including $(x_{p,1:L},x_{a,0:L})$ and the fixed permutation $\mathfrak{s}$); we simply suppress these positional variables in the notation.

\subsection{Pretrained Atomic Skills}\label{sec-atom}
As alluded to previously, we assume that the MLP module provides a pre-trained atomic skill for single-step transitions, and we keep its parameter $W$ fixed throughout RL training. This allows us to focus on long-horizon composition in the attention dynamics.

For each output index $j\in[d]$ and hidden neuron $r\in[m]$, let us define the feature magnitude
\[
V_{j,r}(s)\triangleq \langle W_{j,r}, s\rangle,
\qquad s\in\cG\cup\cY.
\]
For each pair $(g,y)$, let $j=\tau(g(y))$ be the correct next-state index. Within the neuron group $\{W_{j,r}\}_{r\in[m]}$, we designate a unique neuron $r_{g\cdot y}\in[m]$ associated with this pair\footnote{Our theory can also accommodate the setting where, for each feature, there is a non-overlapping group of neurons
$r_1,r_2,\dots,r_k\in[m]$ jointly satisfying the same activation pattern (e.g.,
$\sum_{i=1}^k \sigma\!\big(V_{j,r_i}(g)+V_{j,r_i}(y)\big)=2B$).
We do not adopt this setting to simplify the proof.}.
We define 
\begin{align}
B=C_B\log d
\end{align}
for some sufficiently large integer $C_B=O(1)$ and $\sigma_0=d^{-1/2}$, and assume the features satisfy
\begin{subequations}
\label{eq:choice-mlp-compact}
\begin{align}
&V_{j,r_{g\cdot y}}(g)=B,
\qquad
V_{j,r_{g\cdot y}}(y)=B+2\sigma_0;
\label{eq:choice-mlp-compact-pos}
\\
&V_{j,r_{g\cdot y}}(s)=-B,
\quad &&\forall s\in(\cG\cup\cY)\setminus\{g,y\};
\label{eq:choice-mlp-compact-neg}
\\
&V_{j,r}(s)=0, %\leq O(\sigma_0),
\quad
&&\forall r\notin\{r_{g\cdot y}\}_{\tau(g(y))=j},\ \forall s\in\cG\cup\cY.
\label{eq:choice-mlp-compact-zero}
\end{align}
\end{subequations}
\begin{proposition}\label{prop-atomic}
    Under Assumptions~\ref{assum-group}-\ref{assump-x}, if the MLP weight $W$ satisfies \cref{eq:choice-mlp-compact-pos}--\cref{eq:choice-mlp-compact-zero},  then given any $Q$, for any  $y_0\in\cY$ and $G^{1}=(g_1)$ with $g_1\in\cG$, we have
   % we have $\cJ_{1}(\theta)=1-\frac{1}{\poly d}$.
    \begin{align*}
        \pi_{\theta}\big(g_1(y_0)\,\big|\, y_0, G^{1}\big)=1-\frac{1}{\poly d}.
    \end{align*}
\end{proposition}
Note that for $L=1$, the model necessarily attends to the only prompt $Z_{p,1}$. Combined with the residual connection, the MLP receives an aggregate input
$\frac{1}{2}(g_1+y_0)$. \Cref{prop-atomic} thus guarantees that an MLP equipped with the above structural properties can perfectly implement the atomic group action.
\begin{remark}
At a high level,  when the MLP input contains the correct pair $(g,y)$, a unique activated neuron creates a large positive margin for the correct logit, while mismatched symbols
produce canceling (negative) contributions.
Consequently, the MLP predicts $g(y)$ with near-perfect accuracy. Prior analysis~\citep{huang2025cot} shows that an MLP can learn such a
feature-separated structure under supervised training with suitable
initialization.
\end{remark}

Therefore, we formalize this one-step capability into the following structural assumption on $W$. 
\begin{assumption}[Pretrained MLP]
\label{ass:pretrained-mlp}
The MLP weight $W$ is fixed and satisfies
\cref{eq:choice-mlp-compact-pos}--\cref{eq:choice-mlp-compact-zero}.
\end{assumption}

\paragraph{How does the transformer reason sequentially?}  Solving the $L$-step state-tracking task reduces to carrying out the \emph{single-step} operation \(L\) times: retrieve the required transition $g_\ell$ from the prompt and apply it to the current predicted state $\hat y_{\ell-1}$. Under the above assumption, the pre-trained MLP contains all the atomic skills for one-step transitions, and the remaining challenge lies in \emph{association}: the attention layer must find the correct $g_\ell$ for the current reasoning step, which is learned by RL or SFT.
% By \Cref{def:transformer-arch}, the attention output takes the form
% \begin{align}
% \textstyle \hat y_{\ell-1}+\sum_{\ell'=1}^{L}\attn_{a,\ell-1\to p,\ell'} g_{\ell'}.
% \label{eq-attn-mech}
% \end{align}
% From the above assumption, the pre-trained MLP acts as an atomic executor: when its input is close to $y+g$, it produces logits sharply peaked at the correct next state $g(y)$. Hence, multi-step success is governed by attention purity at every step, i.e., placing a dominant %$\attn_{a,\ell-1\to p,\ell}$ on the target transition $g_\ell$.
% \paragraph{How does the transformer reason?}  Solving the $L$-step state-tracking task reduces to reliably performing the same
% {single-step} operation at each $\ell\in[L]$: retrieve the required
% transition $g_\ell$ from the prompt and apply it to the current predicted state $\hat y_{\ell-1}$.
% By \Cref{def:transformer-arch}, the attention output takes the form
% \begin{align}
% \textstyle \hat y_{\ell-1}+\sum_{\ell'=1}^{L}\attn_{a,\ell-1\to p,\ell'} g_{\ell'}.
% \label{eq-attn-mech}
% \end{align}
% From the above assumption, the pre-trained MLP acts as an atomic executor: when its input is close to $y+g$, it produces logits sharply peaked at the correct next state $g(y)$. Hence, multi-step success is governed by attention purity at every step, i.e., placing a dominant $\attn_{a,\ell-1\to p,\ell}$ on the target transition $g_\ell$.

\subsection{Outcome-based RL Objective}
\label{subsec:rl-objective} 
We train the induced policy $\pi_\theta$ using an outcome-based RL objective with a terminal reward.
Given an instance $Z^L\sim\cD^L$ (equivalently, $(y_0,G^L)$), we generate a state sequence
$\hat y^L=(\hat y_1,\dots,\hat y_L)$ with the policy \(\pi_\theta\) via \cref{eq-policy}.  We assign reward $1$ if the final prediction \(\Hat{y}_L\) matches the true final state $y_L$, and $0$ otherwise:
\[
r(\hat y^L \mid y_0, G^L)
\triangleq
\1\big\{\hat y_L = y_L\big\}, 
\]
where $y_{L}=g_L(\cdots  g_1(y_0))$. The RL objective is defined as the expected terminal reward:
\begin{equation}\label{eq-obj}
\mathcal J_L(\theta)
=
\E_{Z^L}
\Big[
\E_{\hat y^L\sim \pi_\theta^L(\cdot\mid y_0, G^L)}
\big[ 
r(\hat y^L \mid y_0, G^L)
\big]
\Big].
\end{equation}
For comparison, we also consider an SFT-type objective~\citep{Chu2025SFTMR}. Unlike the outcome-based RL objective, which provides a terminal reward only after the full rollout, this supervised objective employs teacher forcing~\citep{huang2025cot, kim2025transformersprovably, wen2025sparse,yang2025multi}: at each step $k$ we condition on the ground-truth current state $y_{k-1}$ and apply immediate supervision to the next state $y_k$.
Formally, the SFT objective is written as
\begin{equation}
\label{eq-obj-sft}
\mathsf{Loss}_L(\theta)
\triangleq
\E_{Z^L}
\bigg[
\frac{1}{L}\sum_{k=1}^{L} -\log \pi_{\theta}( y_{k} \mid y_{k-1}, G^L)
\bigg].
\end{equation}

\paragraph{Learning algorithm.} We consider the \textsf{REINFORCE} algorithm of policy gradient~\citep{williams1992simple}:
\begin{align}
    &\nabla \mathcal J_L(\theta) =
\E_{Z^L, \hat y^L}
\Big[
r(\hat y^L \mid y_0, G^L) 
\nabla \log \pi_\theta^L(\hat y^L \mid y_0, G^L)
\Big].\notag
\end{align}
Given that $W$ is kept fixed, we study  gradient ascent on $Q$ with \textit{length-normalized}~\citep{he2025vl,gao2024designing} %\footnote{Length normalization is a common trick to stabilize policy gradient in RL for LLMs~\citep{he2025vl,gao2024designing}.} 
policy gradient: %$\nabla \mathcal %J_L(\theta)$: %to optmize $\mathcal J_L(\theta)$:
\begin{equation}
Q^{(t+1)}
=
Q^{(t)} + \eta\nabla_Q \widetilde\cJ_L\big(\theta^{(t)}\big),
\label{eq:pg-update}
\end{equation}
where $\widetilde\cJ_L\big(\theta^{(t)}\big)=\frac{1}{L} \cJ_{L}\big(\theta^{(t)}\big)$ and  $\eta>0$ represents the step size. %Here, length normalization is a common trick to stabilize policy gradient in RL for LLMs~\citep{he2025vl,gao2024designing}. 
In addition, we consider  gradient descent for optimizing the SFT loss:
\begin{equation}
Q^{(t+1)}
=
Q^{(t)} - \eta \nabla_Q \Loss_L\big(\theta^{(t)}\big).
\label{eq:pg-update-sft}
\end{equation}
\begin{assumption}[Initialization] \label{assump-init}
At $t=0$,  $Q^{(0)}$ is initialized to be the zero matrix $\mathbf{0}_{d\times d}$.
\end{assumption}
For simplicity, in the ensuing discussions, we let $A^{(t)}$ represent the value of $A$ at iteration $t$, dropping the explicit dependence on the $\theta^{(t)}$ or $Q^{(t)}$ as long as it is clear from the context. 

\section{Learning Short-horizon Compositional Reasoning}
\label{sec:constant-horizon}
In this section, we examine the RL dynamics of transformers on short-horizon compositional tasks.  We prove that RL successfully learns compositional reasoning up to a critical horizon, beyond which a flat-gradient barrier emerges due to the nature of sparse, outcome-based rewards. In contrast, we demonstrate that SFT can overcome this limitation by leveraging immediate supervision.

\subsection{RL for Short-horizon Compositions}
Following the setup in \Cref{sec-atom}, the transformer
$\TF_{\theta^{(0)}}$ at initialization can execute the atomic
one-step skills. We assume that the attention is approximately uniform at $t=0$,
so the induced policy does not reliably implement multi-step
compositions. Our first result shows that, for any short horizon $L \le C_B = O(1)$ where $C_B$ is determined by the pretrained MLP parameters, policy-gradient RL successfully learns the $L$-step composition and yields the desired
attention concentration pattern.

\begin{theorem}[RL for short-horizon problems]\label{thm:rl-constant}
Suppose Assumptions~\ref{assum-group}-\ref{ass:pretrained-mlp} hold, and assume that  \(\TF_{\theta^{(0)}}\) is initialized according to \Cref{assump-init}. Consider any \(L \in [2, C_B]\) (where $C_B>0$ is some sufficiently large constant as mentioned in \Cref{sec-atom}), \(\eta = \frac{1}{\poly(d)}\) and \(\epsilon \in \big(\frac{1}{\log^{\Omega(1)}(d)}, \frac{1}{4}\big)\). Then the transformer \(\TF_{\theta^{(t)}}\) trained via \cref{eq:pg-update} on the objective $\cJ_{L}$ after \(T_{L,\epsilon} = O\big(\frac{L_{\max} \log({L}/{\epsilon}) }{ \eta\log d }\cdot d^{(1 - \epsilon)C_B-1}\big)\) iterations attains:
    \begin{enumerate}[(a), topsep=3pt]
        \item \textbf{Reward optimality:} At \(t = T_{L,\epsilon}\), the reward is optimal in the sense that
        \begin{align*}
            \cJ_{L}^{(t)} \geq 1 - O\qty (\frac{1}{d^{C_B(1-\epsilon) - 1}}).
        \end{align*}
        \item \textbf{Optimal short-horizon attention:} At \(t = T_{L,\epsilon}\), for any \(\ell \le L\), we have 
        \begin{align*}
            \attn_{a,\ell-1 \to p,\ell}^{(t)} \geq 1 - \epsilon.
        \end{align*}
    \end{enumerate}
\end{theorem}

%\paragraph{Significance of the result.} 
Theorem~\ref{thm:rl-constant} provides the first provable guarantee that a
transformer can learn multi-step \emph{compositional} reasoning via
outcome-only policy gradients, even though the initialization implements only the atomic one-step skill.
Our theory unveils an explicit short-horizon regime ($L \le C_B$) in
which learning is provably efficient. Beyond reward optimality, this theorem also uncovers an emergent
attention concentration pattern, providing a mechanistic characterization of how the
learned transformer implements the $L$-step composition.

\subsection{Critical Horizon and Exponentially Flat Region}
When $L$ exceeds the critical threshold $C_B$, the near-uniform attention at initialization yields little informative signal to the MLP layer. As a consequence,  the model behaves nearly randomly on long compositional instances. Under 
outcome-only rewards, the resulting policy-gradient signal becomes exponentially small even when a nonzero reward is received, as formalized in the result below.

\begin{proposition}[Exponentially flat region]\label{prop:flat-region}
Suppose Assumptions~\ref{assum-group}-\ref{ass:pretrained-mlp} hold, and assume that the \(\TF_{\theta^{(0)}}\) is initialized according to \Cref{assump-init}. Then for any horizon \(L > 2C_B\), whenever the feature magnitude 
  \(\max_{x,x'\in\cX} \vbrack{ Q^{(t)}  x, x'} \leq 0.01 \), we have $\cJ_{L}^{(t)}=\frac{1}{d}(1\pm o(1))$, and 
    \begin{align*}
       \max_{x,x'\in\cX} \bigg|\Big\langle\big[\nabla_{Q} \widetilde\cJ_L^{(t)}\big] x, x'\Big\rangle\bigg| \leq \tilde{O}\left( \frac{1}{L_{\max}} \right)\cdot d^{-\Omega(L)}.
    \end{align*}

\end{proposition}

\paragraph{Why is the landscape flat for RL initially?} Conceptually, the initial training horizon controls the concentration of signals the model could learn from each sample, which dilutes due to the \(O(d^L)\) possible trajectories if the model uniformly traverses the actions specified by the problem instance. In this case, outcome-based rewards, due to \emph{lack of process-level feedback}, make it extremely difficult to extract out sufficient signals from policy gradients.

\subsection{SFT Succeeds Beyond the Critical Horizon}
For comparison purposes, we also consider an SFT objective as in \cref{eq-obj-sft}. Since SFT provides intermediate supervision rather than only an outcome-based reward, it remains effective even for long-horizon problems.

\begin{theorem}[SFT provably escapes initial flat region]\label{thm:sft-constant}
Suppose Assumptions~\ref{assum-group}-\ref{ass:pretrained-mlp} hold, and assume that \(\TF_{\theta^{(0)}}\) is initialized according to \Cref{assump-init}. Then for any length \( 2\le L\le \polylog (d)\), \(\eta = \frac{1}{\poly(d)}\) and \(\epsilon \in (\frac{1}{\log^{\Omega(1)}(d)}, \frac{1}{4})\), the transformer \(\TF_{\theta^{(t)}}\) trained via \cref{eq:pg-update-sft} on the objective \(\Loss_L\) (cf.~\cref{eq-obj-sft}) for 
\begin{align*}
    T_{L,\epsilon} = O\Big(\frac{L_{\max}\log({L}/{\epsilon}) d^{(1 - \epsilon)C_B-1}}{\eta \epsilon \log d  }+\frac{L_{\max} L}{\eta \log d}\Big)
\end{align*}
iterations, satisfies:
    \begin{enumerate}[(a)]
        \item \textbf{Loss convergence:} At \(t = T_{L,\epsilon}\), the loss converges in the sense that 
        \begin{align*}
            \Loss_{L}^{(t)} \leq  O \qty(\frac{1}{d^{C_B(1-\epsilon) - 1}}).
        \end{align*}
        \item \textbf{Optimal attention:} For any \(\ell \le L\), we have 
        \begin{align*}
            \attn_{a,\ell-1 \to p,\ell}^{(t)} \geq 1 - \epsilon.
        \end{align*}
    \end{enumerate}
\end{theorem}
In words, \Cref{thm:sft-constant} demonstrates that SFT can successfully train the transformer to solve the composition reasoning beyond the critical horizon.

\section{Learning Dynamics of Mixed-difficulty RL}\label{sec:main-results}
In practice, RL datasets typically contain instances of mixed complexity, due to either heterogeneous data collection or explicit curriculum strategies~\citep{zeng2025simplerl,parashar2025curriculum,chen2025self}, which can fundamentally affect optimization dynamics. Motivated by this observation, we study policy-gradient training under \emph{mixed-difficulty} distributions, which is close to how large-scale LLM RL is performed in production \cite{deepseekai2025deepseekr1}. Our central message is that large gaps in difficulty level lead to delayed, grokking-like handoffs, whereas smoother spectra enable a relay effect that transfers learning progress from easier to harder horizons.

\subsection{Easy-to-Hard Mixture}\label{sec:easy-to-hard-mix}
%\paragraph{Mixture over reasoning horizons.} 
To model mixed difficulty, we consider a mixture over multiple reasoning horizons. This geometric horizon schedule serves as our formalization of a \emph{difficulty spectrum}: adjacent horizons play the role of neighboring difficulty levels through which the implicit curriculum unfolds.
Let us choose \emph{difficulty ratio} \(R >1\) and fix a starting horizon \(L_1 \geq 2\). 
Define a set of horizons $\cL_R=\{L_1,L_2,\ldots,L_K\}$ by the recursion:
\[
L_k = 
\begin{cases}
    \lceil R L_{k-1}\rceil, & \text{if } L_{\max} > \lceil R L_{k-1}\rceil \\
    L_{\max}, & \text{otherwise }
\end{cases}
\]
where $K=\min \{k:L_k = L_{\max}\}$ so that $L_K=L_{\max}$.
We then define the mixed-difficulty objective as the uniform mixture
\begin{align}
\cJ_{\mix,R}(\theta)
=\E_{L\sim\mathrm{Unif}(\cL_R)}\big[\cJ_L(\theta)\big].
\label{eq-mix-obj}
\end{align}
To optimize $\cJ_{\mix,R}$, we consider a {length-normalized} update: 
\begin{align}
    Q^{(t+1)}=Q^{(t)}+\eta \nabla_Q \widetilde{\mathcal{J}}_{\mix,R}(\theta^{(t)}), \label{eq-mix-grad}
\end{align}
where  $\nabla_Q \widetilde{\mathcal{J}}_{\mix,R}(\theta) = \E_{L \sim \mathrm{Unif}(\cL_R)} [\nabla_Q \widetilde{\mathcal{J}}_L(\theta)]$, and recall $\widetilde\cJ_L\big(\theta\big)=\frac{1}{L} \cJ_{L}\big(\theta\big)$.

\subsection{Grokking Dynamics Under Large Difficulty Ratios}\label{sec:grokking}
Intuitively, under mixed-difficulty training, the shorter-horizon tasks are simpler and are therefore solved first, after which learning attempts to extend to longer horizons. When the gap between adjacent difficulty levels is large, i.e., when the mixture ratio $R$ is large, this inter-horizon handoff becomes severely delayed: the implicit curriculum turns jagged because the next horizon remains too far beyond the current competence frontier. As a result, the policy can spend an extended period with near-zero reward on the next horizon before making noticeable improvements.
This plateau-and-jump pattern resembles an empirical phase-transition phenomenon dubbed \emph{grokking}~\citep{Power2022GrokkingGB,Sun2025RLGR}: after a prolonged phase of receiving near-zero reward, the policy abruptly rises to near-perfect accuracy.
Our next theorem formalizes this behavior in the mixed-horizon setting by quantitatively characterizing
the length of the plateau and the ensuing transition.

To state our result, we introduce two observable states for each horizon that capture
(i) when progress first becomes noticeable and (ii) when the horizon is essentially solved.
\paragraph{Stopping times for mastery and visible states.} For any horizon $L_k\in\cL_R$, we say that the task at horizon $L_k$ has \emph{visible return} at iteration $t$ if \(\cJ_{L_k}^{(t)} \ge 0.01\):
\begin{equation}\label{eq:visible-def}
 T_{\mathsf{vis},k} \triangleq \min\big\{t:\cJ_{L_k}^{(t)}\ge 0.01\big\}.
\end{equation}
We say that the horizon $L_k$ is \emph{mastered} at iteration $t$ if \(\cJ_{L_k}^{(t)} \ge 0.99\):
\begin{equation}\label{eq:mastered-def}
T_{\mathsf{mas},k} \triangleq \min\big\{t:\cJ_{L_k}^{(t)}\ge 0.99\big\}.
\end{equation}

\begin{theorem}[Grokking dynamics]\label{thm:grokking-1}
Let $\cJ_{\mix,R}$ be the mixed-difficulty objective with ratio $\omega(1)\le R\le \frac{L_{\max}}{2C_B}$ and starting horizon $L_1=C_B$. Under Assumptions~\ref{assum-group}--\ref{assump-init}, consider the RL training process under the length-normalized update~\cref{eq-mix-grad} with step size
$\eta=1/\poly(d)$. Then for each $1 \le k \le K-2$, the following hold:
 \begin{enumerate}[(a)]
 \item \textbf{Long inter-difficulty plateaus.}
Before the next horizon $L_{k+1}$ makes noticeable progress (i.e., before it enters the visible-return state), the inter-horizon plateau length satisfies
\begin{align}
T_{\mathsf{vis},k+1}-T_{\mathsf{mas},k}= \widetilde{\Theta}\Big(\frac{L_{\max}}{\eta}\Big)\cdot d^{C_B-1}.%\triangleq     \cT_{\mathsf{plat}, k}.
\end{align}

\item \textbf{Grokking-like phase transitions.} Once $L_{k+1}$ enters the visible reward state, it reaches mastery quickly: $T_{\mathsf{mas},k+1}-T_{\mathsf{vis},k+1}\leq \widetilde{O}\big(\frac{L_{\max}}{\eta}\big)\cdot L_{k+1}$.
 \end{enumerate}

\end{theorem}

\Cref{thm:grokking-1} shows that each transition $L_k\!\to\! L_{k+1}$ consists of a long near-zero-return plateau followed by a rapid rise to mastery once return becomes visible. Aggregating these transitions yields a time-to-mastery bound for the longest horizon, in which the total runtime is dominated by the plateaus.
\begin{corollary}\label{cor-grok}
Under the  assumptions of \Cref{thm:grokking-1}, suppose $c_x<\frac{C_B-2}{C_B+2}$, where $c_x$ is defined in Assumption~\ref{assump-x}.  Then the first time the longest horizon
$L_{\max}$ reaches mastery satisfies
\begin{align}
    T_{\mathsf{mas}, K}= \widetilde{\Theta}\Big(\frac{L_{\max}}{\eta}\Big)\cdot d^{C_B-1}\triangleq \cT_{\mathsf{plat}}.
\end{align}
\end{corollary}

\paragraph{Why does grokking happen in RL?}  For long-horizon tasks, reward can either come from fully correct traces or from rare lucky guesses that reach the correct final answer despite intermediate mistakes. Before the policy can reliably generate correct traces at long horizons, the reward stays near-zero, and the gradient signal mainly consists of those from the lucky guesses, which are random and uninformative. Meanwhile, gradient updates from shorter horizons keep sharpening the internal features long after their rewards have saturated. In other words, shorter horizons continue to supply hidden structural progress, but this progress is handed off to the next horizon only after a long delay because the difficulty gap is too large. Once the longer horizon finally becomes learnable, rewards improve rapidly within a few iterations. Our analysis in \Cref{lem-grad-char-tech} makes explicit this delayed-handoff mechanism by tracking how the hidden feature sharpening eventually translates into long-horizon policy improvement.

\subsection{Relay Dynamics under Moderate Difficulty Ratios}\label{sec:relay}
The grokking dynamics above highlight that when the mixture contains well-separated horizons, training can stall at each new horizon: even after $L_k$ is mastered,
the next horizon $L_{k+1}$ may remain in the near-zero-reward regime for a long plateau
before its reward becomes visible. We now show the smooth counterpart: when the difficulty spectrum is sufficiently smooth (i.e., when $R$ is a moderate constant), the implicit curriculum enters a well-behaved relay regime. The underlying optimization mechanism is a \emph{relay} effect, in which progress on easier horizons continuously supports the next harder horizon, preventing prolonged plateaus. 

Our next theorem formalizes the relay regime by providing an upper bound on $T_{\mathsf{vis},k+1}-T_{\mathsf{mas},k}$, which will be significantly smaller than the long plateaus in the regime with large difficulty gaps.

\begin{theorem}[Relay dynamics]\label{thm:relay-1}
Let $\cJ_{\mix,R}$ be the mixed-difficulty objective with ratio $2\le R\le O(1)$
and starting horizon $L_1=C_B$. Under Assumptions~\ref{assum-group}--\ref{assump-init},
consider the RL training process under the length-normalized update~\cref{eq-mix-grad}
with step size $\eta=1/\poly(d)$. Then for each $k\le K-2$,
before $L_{k+1}$ enters the visible-return state, the inter-horizon plateau length satisfies
\[
T_{\mathsf{vis},k+1}-T_{\mathsf{mas},k}
\le
\widetilde{O}\Big(\frac{L_{\max}}{\eta}\Big)\cdot d^{(1-\frac{C_B}{C_B+R})C_B-1}.
\]
Moreover, once $L_{k+1}$ enters the visible-return state, it reaches mastery rapidly: $T_{\mathsf{mas},k+1}-T_{\mathsf{vis},k+1}
\le
\widetilde{O}\Big(\frac{L_{\max}}{\eta}\Big)\cdot L_{k+1}.$
\end{theorem}
Compared with the grokking regime, \Cref{thm:relay-1} shortens each inter-horizon plateau by a factor $d^{\Theta(1)}$. Although a smoother spectrum induces more horizons, we have $K=O(\log_R(L_{\max}/C_B))\le O(\log d)$. Therefore, the total time to reach mastery at the longest horizon is still governed by the (much shorter) relay plateaus, leading to a strictly faster overall convergence bound than in the large-ratio regime.
\begin{corollary}\label{cor-relay}
Under the assumptions of \Cref{thm:relay-1}, suppose $c_x<\frac{C_B-2}{C_B+2}$, where $c_x$ is defined in Assumption~\ref{assump-x}.  Then the first time that the longest horizon
$L_{\max}$ reaches mastery satisfies $T_{\mathsf{mas},K}\le \cT_{\mathsf{relay}}$, where
\[
\cT_{\mathsf{relay}}
\le
\widetilde{O}\Big(\frac{L_{\max}}{\eta}\Big)\cdot d^{(1-\frac{C_B}{C_B+R})C_B-1}
\le
\widetilde{O}\Big(d^{-\frac{C_B^2}{C_B+R}}\Big)\,\cT_{\mathsf{plat}}.
\]
\end{corollary}

\paragraph{Implicit curriculum at the edge of competence.}
\Cref{thm:grokking-1} and \Cref{thm:relay-1} show that the mixed-horizon dynamics are governed by
how quickly training can move from mastering $L_k$ to making noticeable progress on $L_{k+1}$. These are two regimes of the same implicit curriculum.
When $R=O(1)$, the handoff is overlapping and smooth: progress on $L_k$ starts benefiting $L_{k+1}$ before $L_k$ fully saturates, so the next horizon gains a visible return much sooner and training relays steadily across horizons. This is the relay regime. The policy remains ``just competent enough'' on $L_k$ that success on $L_{k+1}$ is no longer purely random, while $L_k$ still provides a strong learning signal; the two horizons therefore improve in tandem, keeping optimization near the edge of competence~\citep{zhang2025interplay}. In the large-$R$ regime, by contrast, the same curriculum has a delayed handoff: the policy must become overwhelmingly competent on $L_k$ before $L_{k+1}$ ceases to be random, so training exhibits a long plateau followed by a grokking-like phase transition. See \Cref{sec-key-lemma} for a more detailed, gradient-level mechanism explanation.

\section{Proof Overview}
This section explains the main proof ideas behind our learning-dynamic analysis.
The central technical ingredient is a characterization of the
policy-gradient signal as a function of a \emph{step-wise probability}, which
reveals an explicit long-horizon thresholding mechanism.

\subsection{Technical Preliminaries}
We first map our reasoning mechanism to some step-invariant quantities. At step $\ell$, attention weights over prompt tokens are given by a softmax of
scores $\langle Qx_{a,\ell-1},x_{p,\ell'}\rangle$. Since
$x_{a,\ell-1}=\mathfrak{s}(x_{p,\ell})$, correct retrieval means that the \emph{aligned}
prompt token $x_{p,\ell}$  receives a strictly larger score than all misaligned tokens $x_{p,\ell'}$ with
$\ell'\neq \ell$.
With initialization $Q^{(0)}=\mathbf 0$ and the symmetry of $\mathcal D^L$, the expected policy gradient update preserves a two-level score
structure: for all $x\in\cX$ and $x'\neq x$, 
\begin{align*}
\big\langle Q^{(t)}\mathfrak{s}(x),x\big\rangle = q^{(t)}, \quad  
\big\langle Q^{(t)}\mathfrak{s}(x),x'\big\rangle = r^{(t)}.
\end{align*}
Consequently, the attention weight is  step-invariant: $$\attn^{(t)}_{a,\ell-1\to p,\ell}\equiv \attn_L^{(t)}.$$ The following lemma translates $\attn_L^{(t)}$ into the corresponding three classes of step-invariant next-state probabilities.

\begin{lemma}[Step-invariant probability]
\label{lem:step-invariant-prob}
Given $\attn_L^{(t)}$, with a fixed MLP in Assumption~\ref{ass:pretrained-mlp}, the next-state distribution assigns step-invariant probability masses to:
\begin{itemize}[itemsep=0pt]
    \item  Target: $p_{L,1}^{(t)}
\triangleq
\pi_{\theta^{(t)}} \big(g_\ell (\hat y_{\ell-1}) \mid \hat y_{\ell-1},G^L\big)$;
\item Context
distractor: $p_{L,2}^{(t)}
\triangleq
\pi_{\theta^{(t)}} \big(g_{\ell'}(\hat y_{\ell-1}) \big|\ \hat y_{\ell-1},G^L\big)$ for any   $\ell'\neq \ell$;
\item Other distractor: $p_{L,3}^{(t)}
\triangleq
\pi_{\theta^{(t)}} \big(y\ \big|\ \hat y_{\ell-1},G^L\big)$ {for any } $y\notin\{g_{\ell'}(\hat y_{\ell-1})\mid \ell'\in[L]\}$. 
\end{itemize}
Moreover, it holds that $ p_{L,1}^{(t)}\propto d^{C_B\attn_L^{(t)}}$. For later use, define the effective probability margin:
\begin{align}
\Delta^{(t)}_{L}:=p^{(t)}_{L,1}-p^{(t)}_{L,3}
\qquad \text{and}\qquad 
\delta^{(t)}_{L}:=p^{(t)}_{L,2}-p^{(t)}_{L,3}.
\end{align}
\end{lemma}

\subsection{Key Lemma for Gradient Estimation}\label{sec-key-lemma}
The preceding discussion reduces the model's one-step behavior to two scalars $(q^{(t)},r^{(t)})$. We now state the main technical result: for length-$L$ tasks, the policy-gradient
signal in the $(q^{(t)},r^{(t)})$ coordinates admits an explicit characterization
in terms of the step-level margin $\Delta_{L}^{(t)}$.
This turns the learning dynamics into an effective one-dimensional evolution of
attention concentration.
\begin{lemma}[Gradient characterization]\label{lem-grad-char-tech}
Throughout the mixed-training process, given $L\in\cL_{R}$, if the step-wise probability satisfies $\frac{\delta^{(t)}_{L}}{\Delta^{(t)}_{L}}\ll \frac{1}{L^2}(1-\Delta^{(t)}_{L})$ and $\frac{p^{(t)}_{L,2}}{p^{(t)}_{L,1}}\ll 1-p_{L,1}^{(t)}$, then we have
    \begin{align}
  \nabla_q \widetilde\cJ^{(t)}_{L} &\propto   (\Delta^{(t)}_{L})^{L} (1-\Delta^{(t)}_{L}),\\
\cJ^{(t)}_L
&=
\frac{1}{d}
+\Bigl(1-\frac{1}{d}\Bigr)(1\pm o(1))\cdot (\Delta^{(t)}_{L})^{L}. 
\end{align}
Moreover, it holds that  $| \nabla_r \widetilde\cJ^{(t)}_{L}|\leq O(\frac{1}{L_{\max}} )\nabla_q \widetilde\cJ^{(t)}_{L}$. 
\end{lemma}

\begin{remark}
    Lemma~\ref{lem-grad-char-tech} shows that the policy gradient is essentially
driven by the $q$-direction, with $|\nabla_r \widetilde\cJ_L|$ being a lower-order term.
Its magnitude exhibits two regimes: when $\Delta_L$ is small, the
long-horizon factor $(\Delta_L)^{L}$ exponentially suppresses the
gradient; when $\Delta_L\approx 1$, the update is in a convergence regime
and is controlled by the shrinking term $1-\Delta_L$.
Since larger $q$ increases the target attention weight $\attn_L$ and
hence the margin $\Delta_L$ via \Cref{lem:step-invariant-prob}, the
takeaway is that length-$L$ learning is negligible until $q$ is large enough to make $\Delta_L$ moderate, after which progress slows down again as $\Delta_L\to 1$.
\end{remark}

\paragraph{Two alternating phases in mixed-horizon training.} Under mixed-horizon training, learning typically proceeds from shorter to longer horizons: shorter tasks suffer from weaker attention dilution, so they contribute usable gradient earlier and drive up $q$; as $q$ grows, longer horizons gradually become gradient-effective and then start to show visible return. Consequently, learning alternates between
(i) a gradient transfer phase, where $\nabla_q \widetilde\cJ_{L_k}$ dominates while
$\nabla_q \widetilde\cJ_{L_{k+1}}$ remains negligible, and (ii) a reward-emergence
phase, where $\nabla_q \widetilde\cJ_{L_{k+1}}$ becomes substantial and $L_{k+1}$ quickly reaches mastery.

\paragraph{Regime comparison: smooth vs.\ delayed implicit curriculum.} The key difference is whether $\nabla_q \widetilde\cJ_{L_{k+1}}$ becomes non-negligible
before $\nabla_q \widetilde\cJ_{L_k}$ has decayed to a  saturated signal. With \Cref{lem-grad-char-tech,lem:step-invariant-prob} in mind, activating $L_{k+1}$ in the first place requires its target attention $\attn_{L_{k+1}}$  to reach a constant-level regime so that the gate
$(\Delta_{L_{k+1}})^{L_{k+1}}$ is not exponentially suppressed; the handoff is then controlled by when
$$
(\Delta_{L_{k+1}})^{L_{k+1}} \approx 1-\Delta_{L_k}.
$$
For large $R=\omega(1)$, reaching this regime for $L_{k+1}$ forces $q$ so large that $L_k$ is already driven to
$\attn_{L_k}=1-o(1)$, making $1-\Delta_{L_k}$ small (on the order of $d^{-(C_B-1)}$). Thus
$\nabla_q \widetilde\cJ_{L_{k+1}}$ stays negligible over a long plateau, and reward emerges via a
grokking-style jump: the curriculum handoff is delayed until the next horizon activates only after the previous one has nearly saturated. For moderate $R=O(1)$, the same catch-up happens while $L_k$ is still away from full saturation, so $\nabla_q \widetilde\cJ_{L_k}$ and $\nabla_q \widetilde\cJ_{L_{k+1}}$ overlap and jointly drive progress. This overlap is the relay mechanism, and it is what makes the implicit curriculum smooth near the edge of competence.

\subsection{Proof of Lemma~\ref{lem-grad-char-tech}: Fourier Analysis on Groups}
We begin by discussing the central technical challenge in analyzing long-horizon policy gradients, and then introduce our Fourier-based techniques for tackling the challenges.
\newcounter{myfn}
\paragraph{Key technical challenges.} The starting point is to express the policy gradient in terms of the one-step \emph{action distribution} on the group.
By simple transitivity, each transition $\hat y_{\ell-1}\to \hat y_\ell$ corresponds to a unique group element
$u_\ell\in\cG$ such that $\hat y_\ell=u_\ell(\hat y_{\ell-1})$.
Let $\mu_\ell$ denote the one-step action law of $u_\ell$ on $\cG$ (under the current policy), and write $\mu_\ell(g)\;=\;\mathbb P(u_\ell=g)$. Armed with this notation, the step-$\ell$ gradient reduces to a posterior-vs-prior gap for the target action:
\[
\textstyle
\nabla_q \cJ_L \propto
\sum_{\ell\in [L]} \Big(\mathbb P(u_\ell=g_\ell\mid \hat y_L=y_L)-\mu_\ell(g_\ell)\Big).
\]
Let $G_\ast=g_L\circ\cdots\circ g_1$.
Simple transitivity also implies that terminal success is exactly the group-product constraint
\[
\textstyle
\hat y_L=y_L
\quad\Longleftrightarrow\quad
u_L\circ\cdots\circ u_1 = G_\ast .
\]
Hence, the success probability and the numerator in the posterior can be written as point evaluations of convolution products, for instance:\footnote{see 
\Cref{sec-prem-harmonic} for formal definitions of convolution $\ast$ and Fourier transforms.} 
\[
\textstyle
\mathbb P(\hat y_L=y_L)
\;=\;
(\mu_L\ast\cdots\ast\mu_1)(G_\ast)\setcounter{myfn}{\value{footnote}}. 
\]
The challenge is that these are high-order convolutions evaluated at a specific group element.
When $L$ is large, a direct expansion in the group domain involves exponentially many mixed terms and offers no clean control:
conditioning on $u_L\circ\cdots\circ u_1=G_\ast$ couples all steps, so the posterior
$\mathbb P(u_\ell=\cdot\mid \hat y_L=y_L)$ is inherently trajectory-level and does not factorize into per-step statistics.

\paragraph{Fourier analysis to estimate the dominant signal.}
To make these convolution powers tractable, we pass to the Fourier domain on $\cG$~\citep{Jebara2008GroupTM,terras1999fourier},
where convolution becomes multiplication, turning the $L$-fold convolution into a structured product of Fourier operators:
$\prod_{\ell\in[L]} \widehat{\mu}_\ell(\lambda)$.
Here $\widehat{\mu}_\ell(\lambda)$ is the Fourier transform of $\mu_{\ell}$ at an irreducible unitary representation $\lambda$\textcolor{BrickRed}{\footnotemark[\value{myfn}]}. 
Notice that the step-invariant three-way partition of next-token outcomes in \Cref{lem:step-invariant-prob} is equivalently a three-way partition of $u_\ell$, and thus
\[
\mu_\ell(g)
=
p_{L,1}\1\{g=g_\ell\}
+
p_{L,2}\1\{g\in G^L\setminus\{g_\ell\}\}
+
p_{L,3}\1\{g\in \cG\setminus G^L\}.
\]
Exploiting this structure, we obtain
$$
\widehat{\mu}_\ell(\lambda)=\Delta_{L}\lambda(g_{\ell})+\delta_L\sum_{g\in G^{L}\setminus\{g_{\ell}\} } \lambda(g),
$$
which splits into a target-aligned contribution and a residual contribution from context distractors. Taking products across $L$ steps, the leading contribution arises from selecting the aligned component at each step (under mild separation conditions on $\Delta_L$ and $\delta_L$), which yields the characteristic $(\Delta_L)^{L}(1-\Delta_L)$ structure in the resulting posterior deviation.

\paragraph{Prior use of group representations in machine learning.}
Group representation theory has been widely explored in machine learning to model and exploit symmetry~\citep{esteves2020theoretical, marchetti2024harmonics} and to analyze structured distributions~\citep{chen2020group,Jebara2008GroupTM}.
Our use is different in both setting and goal: we bring Fourier analysis into the study of
\emph{long-horizon, compositional} RL objectives, where the terminal success event couples all steps and makes policy-gradient estimation inherently trajectory-dependent. The group structure turns this global conditioning into an $L$-fold convolution object that can be controlled spectrally.

\section{Experiments}\label{sec:experiments}

In this section, we provide empirical support for our theoretical findings. We first use a controlled synthetic setting that mirrors our theory, and then report real-model evidence on \(N\)-digit multiplication.

\subsection{Synthetic Experiments}
\paragraph{General setup.} 
The experimental setting is designed to mirror our theoretical framework and assumptions in \Cref{sec:problem-setup}. We consider a cyclic group action over $\mathbb{Z}_{96}$, and study two training paradigms:
\begin{itemize}
    \item fixed-length training with reasoning lengths $L \in \{5, 15, 45\}$;
    \item mixed-length training: $\cL_3=\{5, 15, 45\}$ and $\cL_7=\{5, 35\}$. The reasoning depth $L$ is stochastically assigned within each training batch ($N=512$), where the problem length of each individual sample is uniformly drawn from $\cL_R$ with $R$ the difficulty ratio.
\end{itemize}
We use an abelian cyclic action for experimental convenience; for the lengths considered here (\(L \ge 5\)), the commutativity-induced shortcut effects discussed in \Cref{remark-group} are mild.

\paragraph{Model and training settings.} We employ a one-layer detached attention layer paired with a fixed MLP transition head. First, the transition MLP is pretrained via supervised learning to master the one-step group operation $(y, g) \mapsto y \cdot g$, after which its parameters are frozen. The attention parameters $Q$ are subsequently trained using the \textsf{REINFORCE} algorithm to maximize the terminal reward. We utilize an exponential moving-average baseline (momentum 0.95) and an entropy penalty ($10^{-3}$) to facilitate stable policy gradients.

\paragraph{Evaluation.} We periodically compute the per-length average success rate by running greedy rollouts and calculating the fraction of episodes where the model correctly predicts the entire trajectory $\{\hat{y}_L = y_L\}$ over 30 batches of size 512. We additionally report an attention-alignment metric, peak attention-hit rate, defined as the fraction of steps where the argmax attention weight selects the unique prompt position corresponding to the current action token, i.e., the prompt index $\mathfrak{s}^{-1}(x_{a,k})$ within the sampled set $\{x_{p,\ell}\}_{\ell=1}^L$. A higher hit rate indicates that the attention layer recovers the underlying permutation \(\mathfrak{s}\) by locating the correct prompt position at each step, consistent with attention concentration (\(\attn_L \to 1\)). 

\begin{figure*}[htb]
    \centering
    \begin{subfigure}[t]{0.45\textwidth}
        \centering
\includegraphics[width=.9\linewidth]{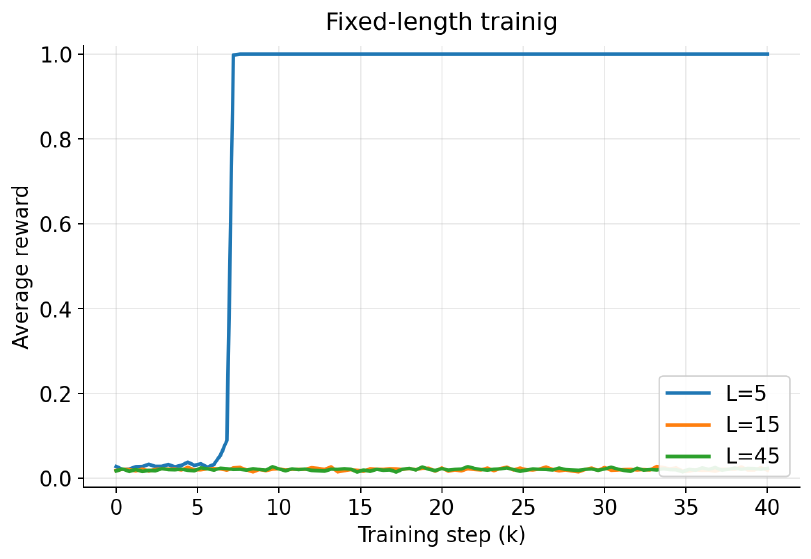}
        \caption{ \centering RL rapidly learns short-horizon compositions, whereas longer horizons exhibit a near-flat reward plateau.}
        \label{fig:fixed-a}
    \end{subfigure}
    \hfill
    \begin{subfigure}[t]{0.45\textwidth}
        \centering
\includegraphics[width=.9\linewidth]{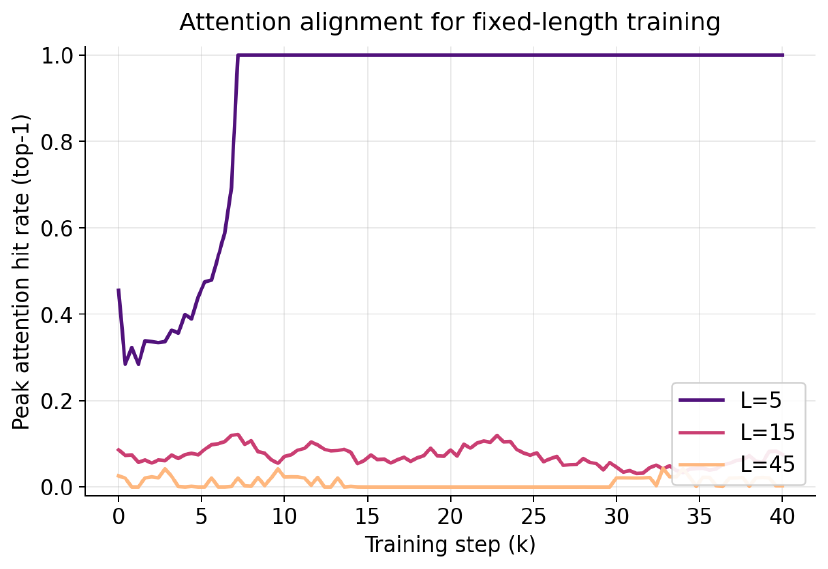}
        \caption{\centering Attention concentrates for short horizons, but saturates at a lower level for longer horizons.}
        \label{fig:fixed-b}
    \end{subfigure}
    %\vspace{-0.1cm}
\caption{ Average reward and peak attention hit rate during fixed-length RL training.}
    \label{fig:fixed-length}
\end{figure*}
\begin{figure*}[ht]
    \centering
    \begin{subfigure}[t]{0.45\textwidth}
        \centering
\includegraphics[width=.9\linewidth]{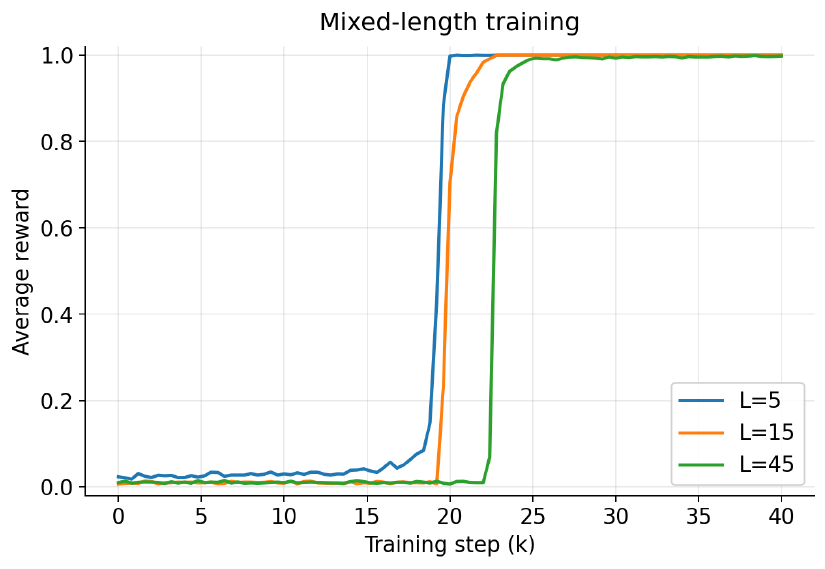}
        \caption{\centering $R=3$: the implicit curriculum stays smooth, with relay across horizons.}
        \label{fig:mixed-a}
    \end{subfigure}
    \hfill
    \begin{subfigure}[t]{0.45\textwidth}
        \centering
\includegraphics[width=.9\linewidth]{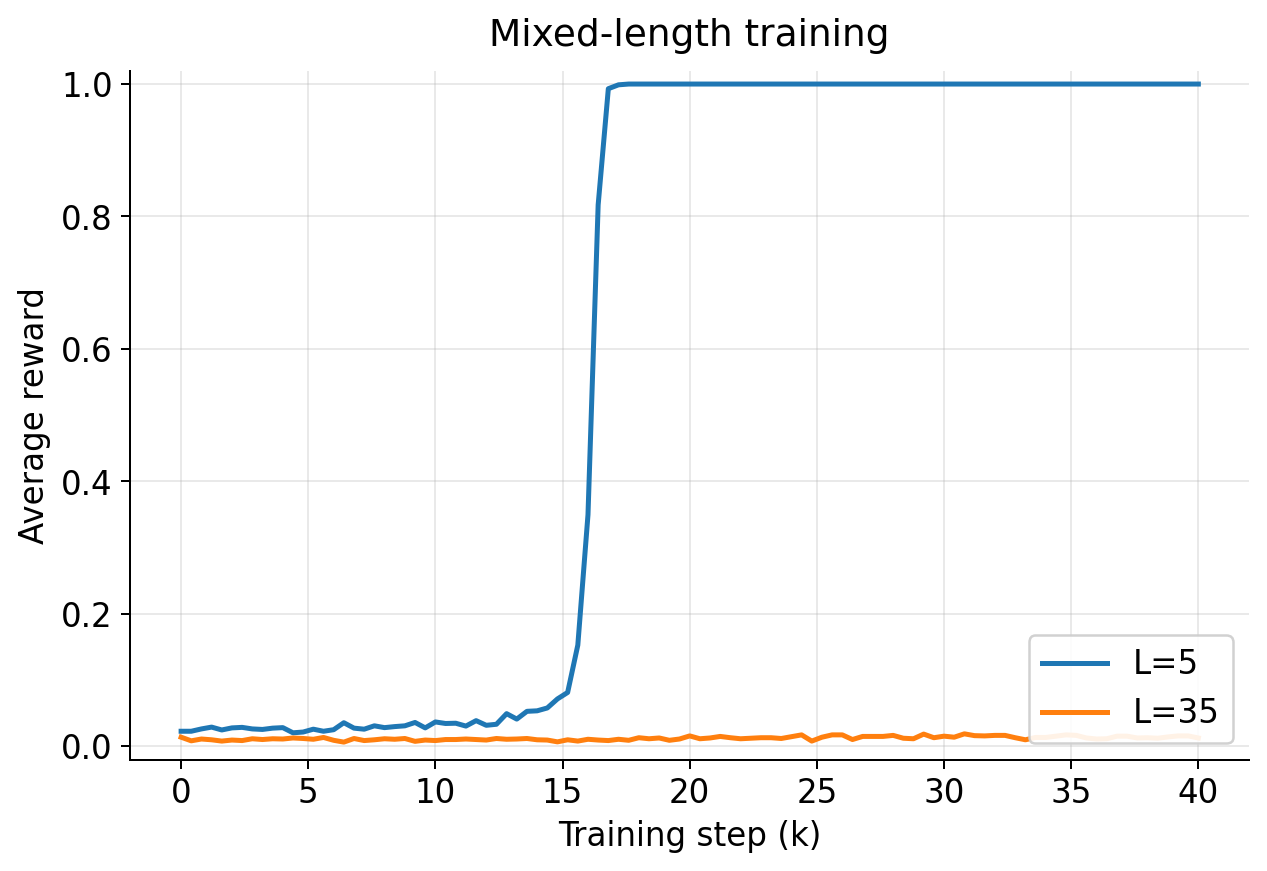}
        \caption{$R=7$: a larger difficulty ratio yields a prolonged plateau at longer horizons.}
        \label{fig:mixed-b}
    \end{subfigure}
    %\vspace{-0.1cm}
\caption{ Average reward during mixed-length RL training under different difficulty ratios.}
    \label{fig:mixed-length}
\end{figure*}
\paragraph{Results for fixed-length training.} The average reward and peak attention hit rate under fixed-length training are reported in \Cref{fig:fixed-length}. We overlay three curves corresponding to training runs at different lengths in the same plot. We observe that the short-horizon setting ($L_1=5$) achieves nearly optimal reward together with strong attention concentration, whereas longer horizons exhibit near-flat plateaus in both reward and attention. This behavior is consistent with the predictions of \Cref{thm:rl-constant} and \Cref{prop:flat-region}.

\paragraph{Results for mixed-length training.} For mixed-length training, we consider two difficulty-ratio regimes: moderate ($R=3$) and large ($R=7$). As shown in \Cref{fig:mixed-length}, both settings exhibit the paper's implicit curriculum, but in different forms. In the moderate-ratio setting, \Cref{fig:mixed-a} shows the smooth relay regime: the plateau between consecutive horizons is shortened, so progress relays efficiently from easier to harder tasks. In contrast, \Cref{fig:mixed-b} shows the delayed-handoff regime: when the difficulty ratio is too large, the longer horizon ($L=35$) remains at near-zero reward throughout training, reflecting a prolonged plateau before the next stage of progress becomes visible. Taken collectively, these observations validate the predictions from \Cref{thm:grokking-1} and \Cref{thm:relay-1} in the mixed-difficulty regime.

\subsection{Real-world LLM RL Behaviors on N-digit Multiplication}

We complement the controlled experiments with RLVR runs on large-scale models. These experiments use \(N\)-digit by \(N\)-digit integer multiplication as a scalable reasoning task with a natural difficulty parameter \(N\).
\paragraph{General setup.}
Each prompt asks for the product of two $N$-digit operands sampled
uniformly from $[10^{N-1}, 10^N)$, and the reward is binary exact match on
the final answer. We run experiments on Qwen2.5-1.5B-Instruct~\citep{yang2024qwen25} and Qwen3-4B-Base~\citep{yang2025qwen3}. We study:
\begin{itemize}
    \item \emph{Single-N RL}: training on a single fixed digit length
          $N\in \{3,4,5\}$ for Qwen2.5-1.5B-Instruct
          and $N\in\{4,5,6\}$ for Qwen3-4B-Base;
    \item \emph{Mixed-N RL}: training on prompts whose digit length is
          drawn uniformly per sample from a configured set within
          each batch, with no explicit curriculum schedule. We consider smooth-gap and large-gap mixtures across the two model scales: Qwen2.5-1.5B-Instruct uses \(N \in \{3,4,\ldots,9\}\) and \(N \in \{3,6,9\}\), while Qwen3-4B-Base uses \(N \in \{4,6,10\}\) and \(N \in \{4,10,18\}\). 
\end{itemize}

% \paragraph{Task and models.}
% Each example is a chain-of-thought prompt asking for the product of two \(N\)-digit operands sampled uniformly from \([10^{N-1}, 10^N)\). The reward is binary exact match on the final numeric answer. We run experiments on Qwen2.5-1.5B-Instruct~\citep{yang2024qwen25} and Qwen3-4B-Base~\citep{yang2025qwen3}.

\paragraph{RL training and evaluation.}
All experiments use GRPO~\citep{shao2024deepseekmath} with \(n=8\) rollouts per prompt, temperature \(1.0\), clip ratio \(0.2\), no KL penalty, token-level loss, and an over-length penalty. We train with Adam~\citep{Kingma2014AdamAM} at learning rate \(10^{-6}\) for 1000 update steps, using maximum response lengths of 4096 tokens for the 1.5B model and 8192 tokens for the 4B model. Checkpoints and per-digit evaluations are produced every 25 steps, with greedy decoding at evaluation. We report per-digit pass@1 on held-out splits with 200 problems per digit length. Dashed lines indicate digit lengths that are evaluated but not trained.

\begin{figure*}[t]
    \centering
    \includegraphics[width=0.85\textwidth]{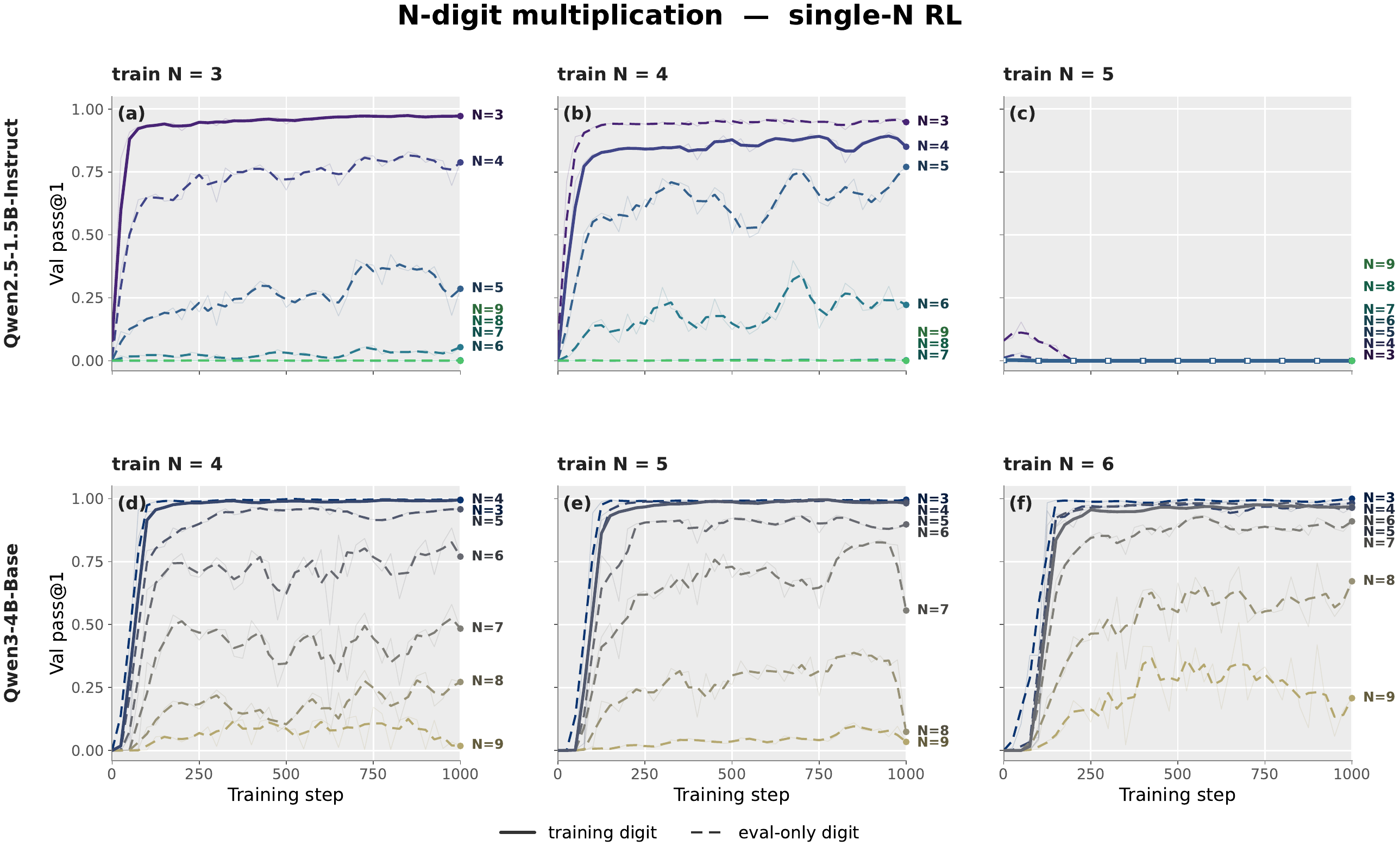}
    \caption{
        RL training on single-\(N\) dataset for Qwen2.5-1.5B-Instruct and Qwen3-4B-Base across different digit lengths. Across settings, models rapidly improve on the training length and often transfer to nearby or shorter lengths, but performance drops sharply for longer unseen lengths. Qwen3-4B-Base exhibits stronger in-distribution learning and broader near-length generalization than Qwen2.5-1.5B-Instruct, yet both models show limited extrapolation to substantially larger \(N\), indicating that single-length RL primarily induces generalization to nearby lengths.}
    \label{fig:single-n-qwen}
\end{figure*}

\paragraph{Results for single-N training.} Single-N RL reveals a \emph{capability frontier} that depends sharply on
model scale.
  In \Cref{fig:single-n-qwen}, for Qwen2.5-1.5B-Instruct
  (the top panel), training succeeds only at
  less difficult digits $N\in\{3,4\}$: in-domain accuracy saturates above 0.9 within $100$
  steps and transfers broadly to neighboring digits, whereas training on 
  $N=5$ stays at zero accuracy over the entire $1000$-step budget, supporting our theory in \Cref{thm:rl-constant} and \Cref{prop:flat-region}. Qwen3-4B-Base
  (the bottom panel) pushes the frontier
  substantially further: all three runs ($N{=}4,5,6$) succeed and transfer
  down to $N{=}3$ and up to $N{=}9$. 
\paragraph{Results for mixed-N training.}
  For mixed-N training, both smooth-gap and large-gap regimes
  exhibit the implicit curriculum, but in distinct forms
  (\Cref{fig:implicit-curriculum-qwen}). In the \emph{smooth relay} regime (left panel), successive
  training digits emerge in an easy-to-hard staircase: each harder $N$
  takes off shortly after its easier predecessor significantly grows, and the
  hardest training digit climbs to non-trivial accuracy. In the \emph{grokking} regime
  (right panel), a larger
  gap between consecutive training digits induces a \emph{prolonged
  plateau} for the next horizon and the hardest in-domain digit remains at zero
  throughout training. This result validates the predictions of \Cref{thm:grokking-1} and
  \Cref{thm:relay-1} in mixed-difficulty settings for large-scale models.

\section{Conclusions}
In this work, we have analyzed the training dynamics of RLVR on a multi-step compositional reasoning task. To the best of our knowledge, we have provided the first end-to-end learning dynamics analysis for outcome-based RL with transformer-based policies, accompanied by explicit convergence guarantees. Our main conclusion is that RLVR scales through the emergence of an \emph{implicit curriculum}: optimization naturally progresses across difficulty levels, and the difficulty spectrum determines whether that progression is smooth or not. When the training distribution contains a smooth enough spectrum, progress is relayed from easier tasks to slightly harder ones, keeping optimization near the edge of competence. When the spectrum contains large gaps, the same curriculum exhibits delayed handoffs and learning instead displays grokking-like plateaus and abrupt transitions. Technically, we have introduced a novel Fourier analysis on groups that makes long-horizon conditioning and compositional structure tractable, and we have provided synthetic and real-model experiments that corroborate the predicted mechanisms.

\paragraph{Limitations and future directions.} One major limitation of our study is that we focus on a single dimension of curriculum structure: we control for all other factors and only scale the length of compositions in the data distribution. This simplification allows us to study the optimization dynamics of RL in a clean, abstracted setting. However, this design choice also limits the generality of our conclusions. In real-world reasoning problems, hard reasoning problems may be different from easy problems in many aspects. For example, they may require more long-tailed distributed atomic skills than easy problems, and they may sit in very different semantic contexts than easy problems, both of which we cannot study theoretically yet. There are also other reasoning patterns, such as planning and search, which we cannot analyze in the current setting. We hope our work can inspire future research into these topics.

\section*{Acknowledgement}

The work of Z.~Wen is supported in part by NSF DMS-2134080, DMS-2134133, CCF-2106778, and Simons Foundation grant 888970. Y.~Wei is supported in part by the NSF CAREER
award DMS-2143215 and the NSF grants CCF-2418156, CCF-2106778 and the Wharton Dean’s Research
Fund. Y.~Chen is supported in part by the Alfred P. Sloan
Research Fellowship, the NSF grants IIS-2218773 and CIF-2221009, the ONR grants N00014-22-1-2354 and N00014-25-1-2344, the Wharton AI \& Analytics
Initiative’s AI Research Fund, and the Amazon Research Award. The work of Y.~Liang is supported in part by NSF grants DMS-2134145 and ECCS-2515482.

\bibliography{cot}
\bibliographystyle{apalike}

%%%%%%%%%%%%%%%%%%%%%%%%%%%%%%%%%%%%%%%%%%%%%%%%%%%%%%%%%%%%%%%%%%%%%%%%%%%%%%%
%%%%%%%%%%%%%%%%%%%%%%%%%%%%%%%%%%%%%%%%%%%%%%%%%%%%%%%%%%%%%%%%%%%%%%%%%%%%%%%
% APPENDIX
%%%%%%%%%%%%%%%%%%%%%%%%%%%%%%%%%%%%%%%%%%%%%%%%%%%%%%%%%%%%%%%%%%%%%%%%%%%%%%%
%%%%%%%%%%%%%%%%%%%%%%%%%%%%%%%%%%%%%%%%%%%%%%%%%%%%%%%%%%%%%%%%%%%%%%%%%%%%%%%
\newpage
\appendix
\onecolumn
%\allowdisplaybreaks

\begin{center}
 \LARGE  \bf Appendix: Complete Proofs
\end{center}

\startcontents[sections]
{
\hypersetup{linkcolor=blue}
\printcontents[sections]{l}{1}{\setcounter{tocdepth}{2}}
}

\section{Preliminaries}
In this section, we introduce some useful notation and derive several preliminary policy gradient lemmas, which will be repeatedly used in the subsequent training-dynamics analysis. Throughout the proof, we use $\dpos$ to denote $|\cX|$. 

\subsection{Gradient Computations}
\paragraph{Notation for gradient expressions.} 
Consider a problem instance of length $L$, we denote a full answer trajectory with initial prefix $Z_{a,0}$ by
\[
\hat{Z}_{a}^{L} \triangleq \bigl( Z_{a,0}, \hat{Z}_{a,1}, \ldots, \hat{Z}_{a,L} \bigr).
\]
For each $1 \le \ell \le L$, we denote the partial trajectory up to step $\ell$ (including the initial prefix) by
\[
\hat{Z}_{a}^{L,\ell} \triangleq \bigl( Z_{a,0}, \hat{Z}_{a,1}, \ldots, \hat{Z}_{a,\ell} \bigr).
\]
We introduce the following shorthand notation (for $j\in\tau(\mathcal Y)$, $r\in[m]$, and $\ell,k\in[L]$):
\begin{subequations}
   \begin{align}
     \Ecal_{j}(\hat{Z}^{L,\ell}_a, Z_{p}^{L}) &\triangleq \1_{\tau(\hat{y}_{\ell}) = j} - \pi_{\theta}(j \mid \hat{y}_{\ell-1}, G^{L}),\label{eq-def-Ecal-icl}\\
%\frac{e^{F_{j}(\hat{Z}^{L,\ell-1})}}{\sum_{j'\in \tau(\Y)}e^{F_{j'}(\hat{Z}^{L,\ell-1})}} 
\Lambda_{j,r}(\hat{Z}^{L,\ell-1}_a, Z_{p}^{L}) &\triangleq \frac{1}{2}\Big(\big\langle W_{j, r}, \hat{Z}_{a,\ell-1}\big\rangle+ \sum_{k \in [L]} \attn_{{a,\ell-1} \rightarrow p,k}\cdot\big\langle W_{j, r}, Z_{p,k}\big\rangle\Big)  . \label{eq-def-Lambda-icl}\\
     \Xi_{\ell, k}(\hat{Z}^{L,\ell}_a, Z_{p}^{L}) &\triangleq \frac{1}{2}\sum_{j \in \tau(\Y)} \Ecal_{j}(\hat{Z}^{L,\ell}_a, Z_{p}^{L})  \sum_{r\in [m]}\sigma^{\prime}\big(\Lambda_{j,r}(\hat{Z}^{L,\ell-1}_a, Z_{p}^{L}) \big)\dbrack{W_{j,r},Z_{p,k}}. \label{eq-def-xi}%\
 \end{align} 
\end{subequations}
Here, $\pi_{\theta}(j \mid \hat{y}_{\ell-1}, G^{L})= \mathsf{softmax}\left(\mathsf{TF}_\theta\left(\hat{Z}_{a, \ell-1}, Z_p^L\right)\right)_{j}$.

\begin{fact}[Gradients of Q]\label{fact-gd}
    Given a problem length $L$, we have the following expression for the policy gradient w.r.t. the attention matrix $Q$:
          \begin{align*}
        & \nabla_{Q}\widetilde\cJ_L
      =\frac{1}{L}\mathbb{E}_{Z^L,\hat{y}^L}
      \left[
      \1_{\hat y_L=y_L}\,
      \sum_{\ell=1}^{L}
      \nabla_{Q} \log \pi_{\theta}( \hat{y}_{\ell} \mid \hat{y}_{\ell-1}, G^{L})
      \right],
      \end{align*}
      where 
      \begin{align*}
      &
      \nabla_{Q} \log \pi_{\theta}( \hat{y}_{\ell} \mid \hat{y}_{\ell-1}, G^{L})
   \\
             & =\sum_{{k} \in [L]}\attn_{{a,\ell-1} \rightarrow p,k} \cdot\left(\Xi_{\ell, k}(\hat{Z}^{L,\ell}_a, Z_{p}^{L})  - \sum_{{k}^{\prime} \in [L]}\attn_{{a,\ell-1} \rightarrow p, k^{\prime}}\Xi_{\ell, k^{\prime}}(\hat{Z}^{L,\ell}_a, Z_{p}^{L}) \right)x_{a,\ell-1}x_{p,k}^{\top}. 
      \end{align*}
Moreover, the gradient of $Q$ for the supervised loss $\Loss_L$ can be written as:
      \begin{align*}
      &
     -\nabla_{Q} \Loss_{L}
   \\
             & =\frac{1}{L}\mathbb{E}_{Z^L}\Big[\sum_{\ell=1}^L\sum_{k=1}^L\attn_{{a,\ell-1} \rightarrow p,k} \cdot\left(\Xi_{\ell, k}({Z}^{L,\ell}_a, Z_{p}^{L})  - \sum_{{k}^{\prime} \in [L]}\attn_{{a,\ell-1} \rightarrow p, k^{\prime}}\Xi_{\ell, k^{\prime}}({Z}^{L,\ell}_a, Z_{p}^{L}) \right)x_{a,\ell-1}x_{p,k}^{\top}\Big]. 
      \end{align*}
\end{fact}

Then we consider the gradient of ${x}^{\top} \nabla_{Q} \log \pi_{\theta}( \hat{y}_{\ell} \mid \hat{y}_{\ell-1}, G^{L})x'$ for $x,x'\in\cX$.

\begin{lemma}\label{lemm-Q-exp-1}
Given  $x,x'\in\cX$, if $x=\mathfrak{s}(x')$, when $x_{a,\ell-1}=x$ for $\ell\in [L]$,   then we have
\begin{align*}
  &{x}^{\top}\nabla_{Q} \log \pi_{\theta}( \hat{y}_{\ell} \mid \hat{y}_{\ell-1}, G^{L}) x'\\
    &= \attn_{{a,\ell-1} \rightarrow p,\ell} \cdot\left(\Xi_{\ell, \ell}(\hat{Z}^{L,\ell}_a, Z_{p}^{L})  - \sum_{{k}^{\prime} \in [L]}\attn_{{a, \ell-1} \rightarrow p, k^{\prime}}\Xi_{\ell, k^{\prime}}(\hat{Z}^{L,\ell}_a, Z_{p}^{L})\right)\\
    &= \attn_{{a,\ell-1} \rightarrow p,\ell} \cdot \bigg(\sum_{j \in \tau(\Y)} \Ecal_{j}(\hat{Z}^{L,\ell}_a, Z_{p}^{L})\sum_{r\in [m]}\sigma^{\prime}\big(\Lambda_{j,r}(\hat{Z}^{L,\ell-1}_a, Z_{p}^{L}) \big)\cdot \Big( \dbrack{W_{j,r},Z_{p, \ell}}- \Lambda_{j,r}(\hat{Z}^{L,\ell-1}_a, Z_{p}^{L}) \Big)\bigg) .
\end{align*}

\end{lemma}

\begin{lemma}\label{lemm-Q-exp-2}
  Given  $x,x'\in\cX$, if $x\neq \mathfrak{s}(x')$, when $x_{p,\ell'}=x'$ and $x_{a,\ell-1}=x$  for $\ell, \ell'\in [L]$, noticing that $\ell'\neq \ell$,   then we have
  \begin{align*}
    &{x}^{\top}\nabla_{Q} \log \pi_{\theta}( \hat{y}_{\ell} \mid \hat{y}_{\ell-1}, G^{L}) x'\\
      &= \attn_{{a,\ell-1} \rightarrow p,\ell'} \cdot\left(\Xi_{\ell, \ell'}(\hat{Z}^{L,\ell}_a, Z_{p}^{L})  - \sum_{{k}^{\prime} \in [L]}\attn_{{a, \ell-1} \rightarrow p, k^{\prime}}\Xi_{\ell, k^{\prime}}(\hat{Z}^{L,\ell}_a, Z_{p}^{L})\right)\\
      &=  \attn_{{a,\ell-1} \rightarrow p,\ell'} \cdot \bigg(\sum_{j \in \tau(\Y)} \Ecal_{j}(\hat{Z}^{L,\ell}_a, Z_{p}^{L})\sum_{r\in [m]}\sigma^{\prime}\big(\Lambda_{j,r}(\hat{Z}^{L,\ell-1}_a, Z_{p}^{L}) \big)\cdot \Big( \dbrack{W_{j,r},Z_{p, \ell'}}- \Lambda_{j,r}(\hat{Z}^{L,\ell-1}_a, Z_{p}^{L}) \Big)\bigg) .
  \end{align*}
  
  \end{lemma}
Observe that there are $d^{L-1}$ intermediate-state sequences
$(\hat y_1,\dots,\hat y_{L-1})\in\mathcal Y^{L-1}$
that lead to $\hat y_L = y_L$.
Only these trajectories yield a nonzero terminal reward.
Hence, it suffices to restrict our attention to their contributions.
Combining this with the preceding lemmas, a direct calculation gives

\begin{lemma}  \label{lem-gd-main}
  Given  $x,x'\in\cX$, we have  
  \begin{itemize}
    \item if $x=\mathfrak{s}(x')$, then
    \begin{align*}
    &  x^{\top}\nabla_{Q}\widetilde\cJ_{L} x'  =\frac{1}{2L\dpos} \mathbb{E}_{Z^{L}}
   \Bigg[ \sum_{v_{1:L}\in\cY^{L-1}\times\{y_{L}\}} \Bigg(\prod_{\ell'=1}^{L} \pi_{\theta}(v_{\ell'} \mid v_{\ell'-1}, G^{L})  \Bigg)\cdot \Big(\sum_{\ell=1}^{L}  \attn_{{a,\ell-1} \rightarrow p,\ell}\\
   &~~~\bigg(\sum_{j \in \tau(\cY)} \Ecal_{j}(\hat{Z}^{L,\ell}_a, Z_{p}^{L}) \sum_{r\in [m]}\sigma^{\prime}\big(\Lambda_{j,r}(\hat{Z}^{L,\ell-1}_a, Z_{p}^{L}) \big)\cdot\Big( \dbrack{W_{j,r},Z_{p, \ell}}- \Lambda_{j,r}(\hat{Z}^{L,\ell-1}_a, Z_{p}^{L}) \Big)\bigg) \Big|_{\hat{y}_{1:\ell}=v_{1:\ell}}\Big) \Bigg]. 
   \end{align*}
   \item else, 
   \begin{align*}
   & x^{\top}\nabla_{Q}\widetilde\cJ_{L} x'  =\frac{1}{2L\dpos(\dpos-1)} \mathbb{E}_{Z^{L}}
 \Bigg[ \sum_{v_{1:L}\in\cY^{L-1}\times\{y_{L}\}} \Bigg(\prod_{\ell'=1}^{L} \pi_{\theta}(v_{\ell'} \mid v_{\ell'-1}, G^{L})  \Bigg)\cdot \Big(\sum_{\ell=1}^{L}\sum_{\ell'\neq \ell}  \attn_{{a,\ell-1} \rightarrow p,\ell'}\\
 &~~~\bigg(\sum_{j \in \tau(\cY)} \Ecal_{j}(\hat{Z}^{L,\ell}_a, Z_{p}^{L}) \sum_{r\in [m]}\sigma^{\prime}\big(\Lambda_{j,r}(\hat{Z}^{L,\ell-1}_a, Z_{p}^{L}) \big)\cdot\Big( \dbrack{W_{j,r},Z_{p, \ell'}}- \Lambda_{j,r}(\hat{Z}^{L,\ell-1}_a, Z_{p}^{L}) \Big)\bigg) \Big|_{\hat{y}_{1:\ell}=v_{1:\ell}}\Big) \Bigg]. 
 \end{align*}
  \end{itemize}

\end{lemma}
\begin{proof}
The two results are obtained by invoking \Cref{lemm-Q-exp-1} and \Cref{lemm-Q-exp-2}, respectively. Specifically, for a given $\ell$, the event $\{x_{a,\ell-1}=x\}$ occurs with probability $1/\dpos$, while for $k \neq \ell$, the event $\{x_{p,\ell}=\mathfrak{s}^{-1}(x), x_{p,k}=x', x'\neq \mathfrak{s}^{-1}(x)\}$ occurs with probability $\frac{1}{\dpos(\dpos-1)}$. The remaining calculation details are omitted here for brevity. 
\end{proof}
Similarly, for the gradient of SFT loss  $\Loss_L$, we have 
\begin{lemma}  \label{lem-gd-main-sft}
  Given  $x,x'\in\cX$, we have  
  \begin{itemize}
    \item if $x=\mathfrak{s}(x')$, then
    \begin{align*}
      -x^{\top}\nabla_{Q}\Loss_{L} x' & =\frac{1}{2L\dpos} \mathbb{E}_{Z^{L}}
   \Bigg[ \sum_{\ell=1}^{L}  \attn_{{a,\ell-1} \rightarrow p,\ell}\\
   &\bigg(\sum_{j \in \tau(\cY)} \Ecal_{j}({Z}^{L,\ell}_a, Z_{p}^{L}) \sum_{r\in [m]}\sigma^{\prime}\big(\Lambda_{j,r}({Z}^{L,\ell-1}_a, Z_{p}^{L}) \big)\cdot\Big( \dbrack{W_{j,r},Z_{p, \ell}}- \Lambda_{j,r}({Z}^{L,\ell-1}_a, Z_{p}^{L}) \Big)\bigg) \Bigg]. 
   \end{align*}
   \item else, 
   \begin{align*}
   - x^{\top}\nabla_{Q}\Loss_{L} x' & =\frac{1}{2L\dpos(\dpos-1)} \mathbb{E}_{Z^{L}}
 \Bigg[ \sum_{\ell=1}^{L}\sum_{\ell'\neq \ell}  \attn_{{a,\ell-1} \rightarrow p,\ell'}\\
 &\bigg(\sum_{j \in \tau(\cY)} \Ecal_{j}({Z}^{L,\ell}_a, Z_{p}^{L}) \sum_{r\in [m]}\sigma^{\prime}\big(\Lambda_{j,r}({Z}^{L,\ell-1}_a, Z_{p}^{L}) \big)\cdot\Big( \dbrack{W_{j,r},Z_{p, \ell'}}- \Lambda_{j,r}({Z}^{L,\ell-1}_a, Z_{p}^{L}) \Big)\bigg)  \Bigg]. 
 \end{align*}
  \end{itemize}

\end{lemma}

We further introduce additional notation to simplify the presentation. Given $G^{L}=\{g_1,\cdots, g_L\}$ and the initial value $y_0$ with induced $\{y_1,\cdots, y_L\}$, define 
\begin{align}
    &\fJ(\theta;y_0, G^{L})\triangleq  \sum_{\bv\in\cY^{L-1}\times\{y_{L}\}} \Bigg(\prod_{\ell'=1}^{L} \pi_{\theta}\big(v_{\ell'}
    \mid v_{\ell'-1}, G^{L}
    \big)\, \Bigg) \Big(\sum_{\ell=1}^{L}  \attn_{{a,\ell-1} \rightarrow p, \ell}\fG_{\ell}(\theta; \bv)\Big) \label{eq-gd-J}\\
    \text{ where } &\fG_{\ell}(\theta; \bv)\triangleq  \sum_{j \in \tau(\cY)} \Ecal_{j}(\hat{Z}^{L,\ell}_a, Z_{p}^{L}) \sum_{r\in [m]}\sigma^{\prime}\big({\Lambda}_{j,r}(\hat{Z}^{L,\ell-1})\big)\cdot\Big( \dbrack{W_{j,r},Z_{p, \ell}}- {\Lambda}_{j,r}(\hat{Z}^{L,\ell-1})\Big) \Big|_{ \hat{y}_{1:\ell}=v_{1:\ell} }.
\end{align}
In what follows, we suppress the dependence on $\theta$ and write
$\fJ(y_0, G^{L})$ and $\fG_{\ell}(\bv)$ for brevity.

\paragraph{Notations for scalarized attention dynamics.} Based on the gradient update, the quantity $\langle x, Qx'\rangle$ takes only two possible values,
depending on whether $x=\mathfrak{s}(x')$ (the matched position) or $x\neq \mathfrak{s}(x')$ (a mismatched position).
Accordingly, we define the (unnormalized) \emph{target} and \emph{non-target} attention scores as
\begin{subequations}\label{eq:score-two-values}
\begin{align}
  q &\triangleq \left\langle Q\,\mathfrak{s}(x),\, x\right\rangle, \qquad x \in \mathcal{X}, \label{eq:q-def}\\
  r &\triangleq \left\langle Q\,\mathfrak{s}(x),\, x^{\prime}\right\rangle, \qquad x^{\prime} \in \mathcal{X}\setminus\{x\}. \label{eq:r-def}
\end{align}
\end{subequations}
With this notation, \Cref{lem-gd-main} can be viewed as a policy-gradient update on $(q,r)$.
Hence, it suffices to track the dynamics of these two scalars in the sequel.  Thus, following \cref{eq-gd-J}, we have
\begin{align}
\nabla_{q}\widetilde\cJ_{L}
=
\frac{1}{2L\dpos}\E_{Z^{L}}\left[\fJ(y_0, G^{L})\right]. \label{eq-gd-J-q}
\end{align}
Furthermore, under this reduction, for a fixed problem length $L$,
the attention weights $\attn_{a,\ell-1\to p,\ell'}$ (for $\ell,\ell'\in[L]$)
take only two distinct values depending on whether the prompt position matches:
\begin{subequations}\label{eq:attn-two-values}
\begin{align}
\attn_{a,\ell-1\to p,\ell}
&=
\frac{e^{q}}{e^{q}+(L-1)e^{r}},
\label{eq:attn-target}\\
\attn_{a,\ell-1\to p,\ell'}
&=
\frac{e^{r}}{e^{q}+(L-1)e^{r}},
\qquad \ell'\neq \ell.
\label{eq:attn-nontarget}
\end{align}
\end{subequations}
When the context is clear, we denote the target attention weight
$\attn_{a,\ell-1\to p,\ell}$ by $\attn_L$ for brevity.

\subsection{Some Useful  Bounds}
\paragraph{Notation for activated neurons.} Fix an output index $j\in\tau(\cY)$, define the fiber $\fF_j\triangleq \{(g,y)\in\cG\times\cY: \tau(g(y))=j\}$, i.e., the set of transition--state pairs whose next state is tokenized as $j$.
For each $(g,y)\in\fF_j$, let $r_{g\cdot y}$ denote the (unique) neuron in the
pre-trained MLP that is activated for predicting $j=\tau(g(y))$ as defined in \cref{eq:choice-mlp-compact}. We further define the set of all activated neurons
\begin{align*}
\fA\triangleq\cup_{j\in\tau(\Y)}\fA_{j}, \text{ where } \fA_j\triangleq \{r\mid \exists (g,y)\in\fF_j, r=r_{g\cdot y}\}.
\end{align*}
Equivalently, $\fA$ collects the activated neurons across all fibers $\{\fF_j\}_{j\in[n_Y]}$.

Substituting the conditions from \cref{eq:choice-mlp-compact} yields the following characterizations of $\Lambda_{j,r}$.

\begin{lemma}[Characterizations of $\Lambda$]\label{lem-lambda-char}
Given $\ell\in[L]$ and $(\hat Z_a^{L,\ell-1}, Z_p^{L})$.
  Let $\{\attn_{a,\ell-1\to p,k}\}_{k=1}^L$ denote the attention weights
  from the answer token at step $\ell-1$ to the $L$ prompt tokens.
  Then we have:
  \begin{enumerate}[(a)]
    \item For any $j\in \tau(\cY)$ and any activated neuron $r\in\fA_j$,  
    \[
    \Lambda_{j,r}(\hat Z_a^{L,\ell-1}, Z_p^{L})
    =
  \frac{1}{2}\Big( V_{j,r}(\hat y_{\ell-1}) + \sum_{k=1}^L \attn_{a,\ell-1\to p,k} V_{j,r}(g_k)\Big).
    \]
    \item For any $j\in\tau(\cY)$ and any non-activated neuron $r\notin\fA_j$,
    \[
    \Lambda_{j,r}(\hat Z_a^{L,\ell-1}, Z_p^{L}) = 0.
    \]
    \end{enumerate}
\end{lemma}

\begin{lemma}[Values of $\Lambda$ at step $\ell$]\label{lem-lambda}
  Fix $\ell\in[L]$ and an input $(\hat Z_a^{L,\ell-1}, Z_p^{L})$.
  Let $\{\attn_{a,\ell-1\to p,k}\}_{k=1}^L$ denote the attention weights.
  Then the following properties hold.
  
  \begin{enumerate}[(a)]
  \item \textbf{target transition $g_\ell$.}
  Let $j \coloneqq \tau\big(g_\ell\cdot \hat y_{\ell-1}\big)$.
  Then
  \[
  \Lambda_{j,\,r_{g_\ell\cdot \hat y_{\ell-1}}}(\hat Z_a^{L,\ell-1}, Z_p^{L})
  =
  \attn_{a,\ell-1\to p,\ell}\,B+\sigma_0,
  \]
  and for any $r\in\fA_j\setminus\{r_{g_\ell\cdot \hat y_{\ell-1}}\}$,
  \[
  \Lambda_{j,r}(\hat Z_a^{L,\ell-1}, Z_p^{L}) <0.
  \]
  
  \item \textbf{ in-context distractor $g_{k}$, $k\neq \ell$.}
  Fix $k\in[L]\setminus\{\ell\}$ and let
  $j \coloneqq \tau\big(g_{k}\cdot \hat y_{\ell-1}\big)$.
  Then
  \[
  \Lambda_{j,\,r_{g_{k}\cdot \hat y_{\ell-1}}}(\hat Z_a^{L,\ell-1}, Z_p^{L})
  =
  \attn_{a,\ell-1\to p,k}\,B+\sigma_0,
  \]
  and for any $r\in\fA_j\setminus\{r_{g_{k}\cdot \hat y_{\ell-1}}\}$,
  \[
  \Lambda_{j,r}(\hat Z_a^{L,\ell-1}, Z_p^{L})<0.
  \]
  
  \item \textbf{vocabulary distractor $g$.}
  For any $g\in\cG\setminus G^L$ and
  $j \coloneqq \tau\big(g\cdot \hat y_{\ell-1}\big)$,
  \[
  \Lambda_{j,\,r_{g\cdot \hat y_{\ell-1}}}(\hat Z_a^{L,\ell-1}, Z_p^{L})
  =
  \sigma_0,
  \]
  and for any $r\in\fA_j\setminus\{r_{g\cdot \hat y_{\ell-1}}\}$,
  \[
  \Lambda_{j,r}(\hat Z_a^{L,\ell-1}, Z_p^{L}) <0.
  \]
  \end{enumerate}
  \end{lemma}

\begin{proof}
By \Cref{lem-lambda-char}, for any $j\in \tau(\cY)$ and any neuron $r\in \fA_j$, we have
\[
\Lambda_{j,r}(\hat Z_a^{L,\ell-1}, Z_p^{L})
=
\frac{1}{2}\Big(V_{j,r}(\hat y_{\ell-1})+\sum_{k=1}^L \attn_{a,\ell-1\to p,k} V_{j,r}(g_k)\Big).
\]
By the simple transitivity assumption, there exists a unique $g^\star\in\cG$ such that
$\tau\!\big(g^\star(\hat y_{\ell-1})\big)=j$.
Invoking \cref{eq:choice-mlp-compact}, we have
\[
V_{j,r_{g^\star\cdot \hat y_{\ell-1}}}(\hat y_{\ell-1}) = B+2\sigma_0.
\]
If $g^\star=g_k\in G^L$ for some $k\in[L]$, then
$V_{j,r_{g^\star\cdot \hat y_{\ell-1}}}(g_k)=B$
and
$V_{j,r_{g^\star\cdot \hat y_{\ell-1}}}(g_{k'})=-B$ for all $k'\neq k$.
Therefore,
\begin{align*}
2\Lambda_{j,r_{g^\star\cdot \hat y_{\ell-1}}}(\hat Z_a^{L,\ell-1}, Z_p^{L})
&=
\attn_{a,\ell-1\to p,k}B
+
\sum_{k'\neq k}\attn_{a,\ell-1\to p,k'}(-B)
+
(B+2\sigma_0)\\
&=
\attn_{a,\ell-1\to p,k}B
-(1-\attn_{a,\ell-1\to p,k})B
+B+2\sigma_0\\
&=
2\attn_{a,\ell-1\to p,k}B+2\sigma_0.
\end{align*}
Otherwise, if $g^\star\notin G^L$, then $V_{j,r_{g^\star\cdot \hat y_{\ell-1}}}(g_k)=-B$ for all $k\in[L]$,
and hence
\[
2\Lambda_{j,r_{g^\star\cdot \hat y_{\ell-1}}}(\hat Z_a^{L,\ell-1}, Z_p^{L})
=
\sum_{k=1}^L \attn_{a,\ell-1\to p,k}(-B)
+(B+2\sigma_0)
=
2\sigma_0.
\]
Finally, consider any other pair $(g,y)$ such that $\tau(g(y))=j$ but $y\neq \hat y_{\ell-1}$.
For the corresponding neuron $r_{g\cdot y}\in\fA_j$, we have $V_{j,r_{g\cdot y}}(\hat y_{\ell-1})=-B$.
Moreover, among $\{g_k\}_{k=1}^L$, at most one index can contribute $+B$ and the remaining contribute $-B$,
so
\begin{align*}
2\Lambda_{j,r_{g\cdot y}}(\hat Z_a^{L,\ell-1}, Z_p^{L})
&\le
\Bigl(2\max_{k\in[L]}\attn_{a,\ell-1\to p,k}-1\Bigr)B
-B\\
&<0.
\end{align*}
This concludes the proof.
\end{proof}

Throughout the following analysis, we suppress the dependence on the underlying instance.
When the context is clear, we abbreviate $\Lambda_{j,r}(\hat{Z}^{L,\ell-1}_a, Z_{p}^{L})$ as $\Lambda_{j,r}$ and $\Ecal_{j}(\hat{Z}^{L,\ell}_a, Z_{p}^{L})$ as $\cE_j$.

Hence, combining the reduced attention pattern in \cref{eq:attn-two-values}, namely, one target receiving weight \(\attn_L^{(t)}\) and \(L-1\) symmetric non-targets, the above characterization of the activations \(\Lambda_{j,r}\) implies a step-invariant, context-level structure for the next-state
distribution \(\pi_\theta(\cdot \mid \widehat{y}_{\ell-1}, G^L)\). In particular, the candidates decompose into three groups: (i) the target \(g_{\ell}\cdot \widehat{y}_{\ell-1}\); (ii) the \(L-1\) symmetric non-targets \(\{g_{\ell'}\cdot \widehat{y}_{\ell-1}:\ \ell'\neq \ell\}\); and (iii) the remaining \(d-L\) states outside the context induced set (i.e., vocabulary distractors). We formalize this decomposition in the following lemma.

\begin{lemma}\label{lem:pi-form}
At step $\ell$, conditioning on $\hat{y}_{\ell-1}$ and $G^{L}$, the policy
  $\pi^{(t)}_{\theta}(\cdot \mid \hat{y}_{\ell-1}, G^{L})$ satisfies:
  \begin{enumerate}[(i)]
  \item For $j=\tau\big(g_{\ell}(\hat y_{\ell-1})\big)$,
  \begin{align}
  \pi^{(t)}_{\theta}(j\mid \hat{y}_{\ell-1}, G^{L})
  =\frac{d^{\attn_{L}^{(t)}C_B}}
  {d^{\attn_{L}^{(t)}C_B}+(L-1)d^{\frac{1-\attn_{L}^{(t)}}{L-1}C_B}+(d-L)}
  \triangleq p^{(t)}_{L,1}.
  \end{align}
  
  \item For $j\in \tau\Big(\big\{g\cdot \hat{y}_{\ell-1}: g\in G^{L}, g\neq g_{\ell}\big\}\Big)$,
  \begin{align}
  \pi^{(t)}_{\theta}(j\mid \hat{y}_{\ell-1}, G^{L})
  =\frac{d^{\frac{1-\attn_{L}^{(t)}}{L-1}C_B}}
  {d^{\attn_{L}^{(t)}C_B}+(L-1)d^{\frac{1-\attn_{L}^{(t)}}{L-1}C_B}+(d-L)}
  \triangleq p^{(t)}_{L,2}.
  \end{align}
  
  \item For any other $j\in \tau(\Y)$,
  \begin{align}
  \pi^{(t)}_{\theta}(j\mid \hat{y}_{\ell-1}, G^{L})
  =\frac{1}
  {d^{\attn_{L}^{(t)}C_B}+(L-1)d^{\frac{1-\attn_{L}^{(t)}}{L-1}C_B}+(d-L)}
  \triangleq p^{(t)}_{L,3}.
  \end{align}
  \end{enumerate}
  
  Moreover, $\pi^{(t)}_{\theta}(j\mid \hat{y}_{\ell-1}, G^{L})$ does not depend on $\ell$.
  Hence, we suppress the index $\ell$ and write $p^{(t)}_{L,1}$, $p^{(t)}_{L,2}$,
  and $p^{(t)}_{L,3}$ for brevity.
  \end{lemma}

\paragraph{Probabilistic Event.}  
We conclude this subsection by introducing a probabilistic event that characterizes the potential for \textit{path collisions}, where an incorrect sequence of operations inadvertently leads to the same outcome as the intended compositional path. Such an event serves as a key tool for bounding the interference from distracting trajectories in our subsequent analysis:
\begin{align*}
  \fE_{L}
  \triangleq
  \Big\{
    \exists \mathbf{i}=(i_1,\dots,i_L)\in[L]^L,
    \ \mathbf{i}\neq(1,\dots,L),
    \ \text{s.t. } g_{i_L}\circ  \cdots  \circ g_{i_1}(y_0)=y_{L}
  \Big\}.
\end{align*}

\begin{lemma}[Probability of Trajectory Collision]\label{lem-prob-dis}
    Under \Cref{assum-group,assump-collision}, for every fixed $L=O(1)$, the probability that any noncanonical sequence of $L$ actions formed by the same set of available operators hits the target state $y_L$ satisfies
    \begin{align}
         \mathbb{P}(\fE_{L}) \le \rho_L(d)=o(1).
    \end{align}
\end{lemma}
\begin{proof}
    By simple transitivity, for any fixed $y_0$, the equality
    $g_{i_L}\circ\cdots\circ g_{i_1}(y_0)=y_L$ is equivalent to
    $g_{i_L}\circ\cdots\circ g_{i_1}=g_L\circ\cdots\circ g_1$.
    Applying the union bound over all noncanonical index tuples $\mathbf{i}\neq(1,\dots,L)$ gives
    $\mathbb{P}(\fE_L)\le \rho_L(d)$, which is $o(1)$ for fixed $L$ by \Cref{assump-collision}.
\end{proof}

\begin{remark}[Two-step collisions]\label{rem:two-step-collisions}
For $L=2$, the only nontrivial noncanonical collision is commutation of the two sampled operators. Under sampling without replacement, this probability is
\begin{align}
    \mathbb{P}\big(g_1\circ g_2=g_2\circ g_1\big)
    =
    \frac{k(\cG)-1}{d-1},
\end{align}
where $k(\cG)$ is the number of conjugacy classes of $\cG$. Thus the two-step part of \Cref{assump-collision} asks only for sublinear conjugacy-class growth, rather than a uniformly bounded number of conjugacy classes.
Indeed, the number of ordered commuting pairs in $\cG^2$ is
$\sum_{g\in\cG}|C_{\cG}(g)|=d\,k(\cG)$; removing the $d$ diagonal pairs and normalizing by the $d(d-1)$ ordered distinct pairs gives the displayed probability.
For example, for $\cG=\mathrm{PSL}_2(q)$ one has $k(\cG)=O(q)$ while $|\cG|=\Theta(q^3)$, so this two-step collision probability is $O(q^{-2})$; more generally, conjugacy-class bounds for finite groups of Lie type are given in \citet{FulmanGuralnick2009Conjugacy}.
\end{remark}

\section{Gradient Characterization for General Length}\label{sec:gradient-char}
In this section, we use spectral analysis to derive explicit gradient formulas for $q^{(t)}$ and $r^{(t)}$ under the step-invariant constraints in \Cref{lem:pi-form}. For each fixed problem length $L$, these characterizations provide the analytical foundation for our subsequent study of training dynamics across different training settings in later sections.

\paragraph{Action distribution induced by $\pi_\theta$.}
We introduce an action variable $u_\ell\in\cG$ as the unique group element applied at step $\ell$.
By the simply-transitive action of $\cG$ on $\cY$, each transition $\hat y_{\ell-1}\to \hat y_\ell$
corresponds to a unique $u_\ell$ such that $\hat y_\ell = u_\ell(\hat y_{\ell-1})$.
With this notation, \Cref{lem:pi-form} can be equivalently stated as a step-invariant action distribution.

\begin{lemma}\label{lem:action-dist}
    Fix a step $\ell$ and condition on $(\hat y_{\ell-1},G^L)$.
    Then the policy $\pi_\theta(\cdot\mid \hat y_{\ell-1},G^L)$ can be partitioned into the following three classes:
    \begin{subequations}
        \begin{itemize}
            \item {target action} ($p_{L,1}$): applying the correct in-context rule $g_{\ell}$,
            \begin{align}
            p_{L,1}
            &:= \pi_{\theta}\left(j\mid \hat{y}_{\ell-1}, G^{L}\right),
            \qquad
            j=\tau\big(g_{\ell}(\hat y_{\ell-1})\big).
            \end{align}
            
            \item {in-context distractor actions} ($p_{L,2}$): applying an incorrect rule from the context,
            \begin{align}
            p_{L,2}
            &:= \pi_{\theta}\left(j\mid \hat{y}_{\ell-1}, G^{L}\right),
            \qquad
            j\in \tau\Big(\big\{g(\hat{y}_{\ell-1}) : g\in G^{L}, g\neq g_{\ell}\big\}\Big).
            \end{align}
            
            \item {vocabulary distractor actions} ($p_{L,3}$): any other token not corresponding to an
            in-context transition,
            \begin{align}
            p_{L,3}
            &:= \pi_{\theta}\left(j\mid \hat{y}_{\ell-1}, G^{L}\right),
            \qquad
            j\in \tau\Big(\big\{g(\hat{y}_{\ell-1}) : g\notin G^{L}\big\}\Big).
            \end{align}
            \end{itemize}
    \end{subequations}
    
\end{lemma}

\subsection{Trajectory measure induced by $\mathsf{TF}_{\theta}$}

In this part, we formally define the trajectory measure induced by the model $\mathsf{TF}_{\theta}$
and derive an explicit gradient representation under this probabilistic framework.
This formulation allows us to express the optimization objective in terms of conditional (posterior)
probabilities over trajectories.

Given a problem instance $(y_0,G^{L})$, for any trajectory $\bv=(v_1,\ldots,v_L)$, we define the
trajectory measure $\tilde{\mathbb P}_{\theta,(y_0,G^{L})}$ induced by $\mathsf{TF}_{\theta}$ as
\begin{align}
\tilde{\mathbb P}_{\theta,(y_0,G^{L})}(\bv)
:= \prod_{\ell'=1}^{L}
\pi_{\theta}\left(
v_{\ell'} \mid v_{\ell'-1}, G^{L}
\right).
\label{eq:trajectory-measure}
\end{align}
Under the induced measure \cref{eq:trajectory-measure}, we can rewrite the gradient expression in
\cref{eq-gd-J-q} as
\begin{align}
\fJ(y_0,G^{L})
&=
\attn_{L}\cdot
\tilde{\mathbb P}_{\theta,(y_0,G^{L})}\left(v_{L}=y_L\right)\cdot
\tilde{\mathbb E}_{\theta,(y_0,G^{L})}\left[
\sum_{\ell=1}^{L}\fG_{\ell}(\bv)\middle| v_{L}=y_L
\right],
\label{eq:J-trajectory-form}
\end{align}
where
\begin{align}
\tilde{\mathbb P}_{\theta,(y_0,G^{L})}\left(v_{L}=y_L\right)
&=
\sum_{\bv\in \cY^{L-1}\times\{y_{L}\}}
\tilde{\mathbb P}_{\theta,(y_0,G^{L})}(\bv),
\label{eq:success-prob}
\\
\tilde{\mathbb E}_{\theta,(y_0,G^{L})}\left[
\sum_{\ell=1}^{L}\fG_{\ell}(\bv)\middle| v_{L}=y_L
\right]
&=
\frac{
\sum_{\bv\in \cY^{L-1}\times\{y_{L}\}}
\tilde{\mathbb P}_{\theta,(y_0,G^{L})}(\bv)
\sum_{\ell=1}^{L}\fG_{\ell}(\bv)
}{
\tilde{\mathbb P}_{\theta,(y_0,G^{L})}\left(v_{L}=y_L\right)
}.
\label{eq:cond-exp}
\end{align}

We now turn to the term $\fG_{\ell}$.
By \Cref{lem-lambda}, we can rewrite $\fG_{\ell}(\bv)$ as
\begin{align*}
\fG_{\ell}(\bv)
&=\sum_{j \in \tau(\cY)} \Ecal_j \sum_{r\in [m]}
\sigma^{\prime}\big({\Lambda}_{j,r}\big)
\Big( V_{j,r}(g_{\ell})- {\Lambda}_{j,r}\Big) \\
&=\sum_{g \in  \cG}
\Ecal_{\tau(g(\hat y_{\ell-1}))}
\Big( V_{\tau(g(\hat y_{\ell-1})),r_{g(\hat y_{\ell-1})}}(g_{\ell})
- {\Lambda}_{\tau(g(\hat y_{\ell-1})),r_{g(\hat y_{\ell-1})}}\Big)\\
&=\1\{g(v_{\ell-1})=v_{\ell}\}
\Big( V_{\tau(g(\hat y_{\ell-1})),r_{g(\hat y_{\ell-1})}}(g_{\ell})
- {\Lambda}_{\tau(g(\hat y_{\ell-1})),r_{g(\hat y_{\ell-1})}}\Big)\\
&\quad-\sum_{g \in  \cG} \pi_{\theta}\left(\tau(g(v_{\ell-1}))\mid v_{\ell-1}, G^{L}\right)
\Big( V_{\tau(g(\hat y_{\ell-1})),r_{g(\hat y_{\ell-1})}}(g_{\ell})
- {\Lambda}_{\tau(g(\hat y_{\ell-1})),r_{g(\hat y_{\ell-1})}}\Big).
\end{align*}

Using the step-invariant three-way partition in \Cref{lem:pi-form}, the last term further simplifies to
\begin{align*}
\fG_{\ell}(\bv)
&=\1\{g(v_{\ell-1})=v_{\ell}\}
\Big( V_{\tau(g(\hat y_{\ell-1})),r_{g(\hat y_{\ell-1})}}(g_{\ell})
- {\Lambda}_{\tau(g(\hat y_{\ell-1})),r_{g(\hat y_{\ell-1})}}\Big)\\
&\quad- p_{L,1}\Big( (1-\attn_{L})B-\sigma_0\Big)\\
&\quad+ (L-1)p_{L,2}\Big(\big(1+\tfrac{1-\attn_{L}}{L-1}\big)B+\sigma_0 \Big)\\
&\quad+ (d-L)p_{L,3}\Big( B+\sigma_0\Big),
\end{align*}
where we used that the in-context distractor set has size $L-1$ and the remaining vocabulary set has size $d-L$.

Next, define the posterior probabilities (under the trajectory measure conditioned on success)
\begin{align}
\rho_{\ell,1}
&\triangleq \tilde{\mathbb P}_{\theta,(y_0, G^{L})}\left(u_\ell=g_{\ell}\middle| v_{L}=y_{L}\right), \label{eq:rho1}\\
\rho_{\ell,2}
&\triangleq \tilde{\mathbb P}_{\theta,(y_0, G^{L})}\left(u_\ell\in G^{L}\setminus\{g_{\ell}\}\middle| v_{L}=y_{L}\right). \label{eq:rho2}
\end{align}

Taking the conditional expectation of $\fG_{\ell}(\bv)$ given $v_L=y_L$ yields
\begin{align*}
\tilde{\mathbb E}_{\theta,(y_0, G^{L})}\left[\fG_{\ell}(\bv)\middle| v_{L}=y_L\right]
&= \rho_{\ell,1}\Big( (1-\attn_{L})B-\sigma_0\Big)
-\rho_{\ell,2}\Big(\big(1+\tfrac{1-\attn_{L}}{L-1}\big)B+\sigma_0 \Big)\\
&\quad-(1-\rho_{\ell,1}-\rho_{\ell,2})\Big( B+\sigma_0\Big)\\
&\quad- p_{L,1}\Big( (1-\attn_{L})B-\sigma_0\Big)
+ (L-1)p_{L,2}\Big(\big(1+\tfrac{1-\attn_{L}}{L-1}\big)B+\sigma_0 \Big)\\
&\quad+\big(1-p_{L,1}-(L-1)p_{L,2}\big)\Big( B+\sigma_0\Big)\\
&=\Bigg[(\rho_{\ell,1}-p_{L,1})(2-\attn_{L})
+\Big(p_{L,2}-\frac{\rho_{\ell,2}}{L-1}\Big){(1-\attn_{L})}\Bigg]B.
\end{align*}
Putting it back to  \cref{eq:J-trajectory-form}, we have
\begin{align}
  &\nabla_q \widetilde\cJ_{L}    =  \frac{1}{2L\dpos}\E_{Z^{L}}   \left[ \fJ(y_0,G^{L}) \right]\notag \\
    &=
   \frac{1}{2L\dpos} \E_{Z^{L}} \Big[  \attn_{L}\cdot
    \tilde{\mathbb P}_{\theta,(y_0,G^{L})}\left(v_{L}=y_L\right) B\cdot
   \sum_{\ell} \Big((\rho_{\ell,1}-p_{L,1})(2-\attn_{L})
+\Big(p_{L,2}-\frac{\rho_{\ell,2}}{L-1}\Big){(1-\attn_{L})}\Big) \Big].
    \label{eq:J-trajectory-form-conditional}
    \end{align}

Intuitively, the gradient is driven by the gap between the posterior action probabilities
conditioned on success and their unconditional counterparts. Therefore, controlling the gradient reduces to estimating the posterior probabilities
$\rho_{\ell,1}$ and $\rho_{\ell,2}$ as functions of the probability tuple
$(p_{L,1},p_{L,2},p_{L,3})$.
In what follows, we suppress the dependence on $(\theta,(y_0, G^{L}))$ whenever the context is clear.

\subsection{Preliminaries: Harmonic Analysis on $\cG$}\label{sec-prem-harmonic}
The trajectory measure introduced in the previous part involves cumulative products of group actions,
which correspond to repeated convolutions of measures on the underlying group.
A direct combinatorial analysis of these convolutions is often intractable.
To address this, we work in the Fourier domain via the irreducible representations of the group.

In this section, we briefly review the basic facts of harmonic analysis on finite
groups and collect the spectral tools we will use to decouple these convolution operations.
For background and more detailed treatments, see \citet{Serre1977LinearRO,Jebara2008GroupTM}.

\begin{definition}[Irreducible Representations]
  Let $\cG$ be a finite group of order $|\cG|=N$.  Let $\Lambda$ denote the set of equivalence classes of irreducible unitary representations of $\cG$. For each $\lambda \in \Lambda$, $\lambda: \cG \to U(d_\lambda)$ is a homomorphism, where $U(d_\lambda)$ is the group of $d_\lambda \times d_\lambda$ unitary matrices.
   \begin{itemize}
       \item The trivial representation is denoted by $\mathbf{1}$, with $d_\mathbf{1}=1$ and $\mathbf{1}(g)=1, \forall g$.
\item Orthogonality Relations: For $\lambda, \eta \in \Lambda$:
$$\langle \lambda_{ij}, \eta_{kl} \rangle = \sum_{g \in \cG} \lambda_{ij}(g) \overline{\eta_{kl}(g)} = \frac{N}{d_\lambda} \delta_{\lambda \eta} \delta_{ik} \delta_{jl}$$

Specifically, for $\lambda \neq \mathbf{1}$, $\sum_{g \in \cG} \lambda(g) = \mathbf{0}_{d_\lambda\times d_\lambda}$.
   \end{itemize}
\end{definition}

\begin{definition}[Fourier Transform] Let $\Lambda$ be the set of irreducible unitary representations of $\cG$. For any $\lambda \in \Lambda$, let $d_\lambda$ be its dimension. For a function $f: \cG \to \mathbb{C}$, the Fourier transform is:
$$\widehat{f}(\lambda) \triangleq \sum_{h \in \cG} f(h) \lambda(h).$$

\end{definition}

\begin{definition}[Convolution]
    The convolution of two functions $f, \nu: \cG \to \mathbb{C}$ is defined as:
$$(f * \nu)(g) := \sum_{h \in \cG} f(gh^{-1}) \nu(h).$$
The Fourier transform maps convolution to matrix multiplication:
$$\widehat{f * \nu}(\lambda) = \widehat{f}(\lambda) \widehat{\nu}(\lambda).$$

\end{definition}
\begin{lemma}[Fourier Inversion Formula \& Plancherel Identity] \label{lem-inverse}
Function $f: \cG \to \mathbb{C}$ can be reconstructed from its Fourier coefficients:
$$f(g) = \frac{1}{N} \sum_{\lambda \in \Lambda} d_\lambda \text{Tr}\left( \widehat{f}(\lambda) \lambda(g)^{-1} \right).$$
Using the inversion formula at $g=e$ (identity), we have the identity:
$$\sum_{\lambda \in \Lambda} d_\lambda^2 = N.$$
We distinguish the trivial representation $\mathbf{1}$ (where $\lambda(h)=1$) from non-trivial representations $\lambda \neq \mathbf{1}$. Note that $\sum_{\lambda \neq \mathbf{1}} d_\lambda^2 = N-1$.
    
\end{lemma}

\begin{definition}[Character Value \& Spectral Decay Factor]
    Let $\cG$ be a finite group and let $\Lambda$ denote the set of its irreducible unitary representations. The \textbf{character} of a representation $\lambda \in \Lambda$, denoted by $\chi_\lambda: \cG \to \mathbb{C}$, is defined as the trace of the linear operator $\lambda(g)$ for each $g \in \cG$:
    $$ \chi_\lambda(g) := \mathrm{Tr}(\lambda(g)). $$
    The scalar $\chi_\lambda(g)$ is referred to as the character value of the element $g$ corresponding to $\lambda$. Furthermore, we define the \textbf{spectral decay factor}, denoted by $\gamma(\cG)$, as the maximum normalized character value over all non-trivial representations and non-identity elements:
    $$ \gamma(\cG) := \max_{\substack{\lambda \in \Lambda, \lambda \neq \mathbf{1} \\ g \in \cG, g \neq e}} \frac{|\chi_\lambda(g)|}{d_\lambda}, $$
    where $d_\lambda$ denotes the dimension of the representation $\lambda$.
\end{definition}

\begin{remark}[Magnitude of $\gamma(\cG)$] For any finite group $\cG$, it holds that $0 \le \gamma(\cG) \le 1$. Specifically, if $\cG$ is abelian or has a non-trivial center $Z(\cG) \neq \{e\}$, then $\gamma(\cG) = 1$. If $\cG$ is a non-abelian simple group, then $\gamma(\cG) < 1$. Furthermore, for many sequences of simple groups (e.g., $PSL_2(q)$), $\gamma(\cG) \to 0$ as $|\cG| \to \infty$, indicating rapid spectral decay.
\end{remark}

\subsection{Spectral Decomposition of the Step Measure}
With the harmonic analysis framework in place, we translate conditional expectations under the
trajectory measure induced by $\mathsf{TF}_{\theta}$ into convolution equations on the group.

\paragraph{Reduction to Group Actions.}
Since $\cG$ acts simply transitively on $\cY$, for any trajectory $\bv$ with fixed $v_0=y_0$ there
exists a unique sequence of group actions $(u_1,\ldots,u_L)\in\cG^{L}$ such that
$v_\ell = u_\ell(v_{\ell-1})$ for $\ell\in[L]$.
Consequently,
\[
v_L = (u_L\cdots u_1)(y_0).
\]
Let the target path be given by $(g_1,\ldots,g_L)$ so that $y_L=(g_L\cdots g_1)(y_0)$, and define the
target composition $G_{\ast}\triangleq g_L\cdots g_1$.
Then the endpoint constraint is equivalent to the group equation
\[
v_L=y_L \quad\Longleftrightarrow\quad u_L\cdots u_1 = G_{\ast}.
\]
Under this representation, the posteriors of interest can be written as
\[
\rho_{\ell,1}
=\tilde{\mathbb P}\left(u_\ell=g_{\ell}\middle|u_L\cdots u_1=G_{\ast}\right),
\qquad
\rho_{\ell,2}
=\tilde{\mathbb P}\left(u_\ell\in G^{L}\setminus\{g_{\ell}\}\middle|u_L\cdots u_1=G_{\ast}\right).
\]

\begin{definition}[One-step measure]\label{def:one-step-measure}
For each step $\ell\in[L]$, define a probability measure $\mu_\ell$ on $\cG$ by
\[
\mu_\ell(h)\triangleq \tilde{\mathbb P}(u_\ell=h)
=\begin{cases}
p_{L,1}, & h=g_{\ell},\\
p_{L,2}, & h\in G^{L}\setminus\{g_{\ell}\},\\
p_{L,3}, & h\in \cG\setminus G^{L}.
\end{cases}
\]
Equivalently,
\[
\mu_\ell
= p_{L,1}\delta_{g_\ell}
+ p_{L,2}\delta_{G^{L}\setminus\{g_\ell\}}
+ p_{L,3}\delta_{\cG\setminus G^{L}},
\]
where $\delta_S$ denotes the (unnormalized) uniform measure on a set $S\subseteq\cG$.
Moreover, the marginal probability of the endpoint is given by
\[
\tilde{\mathbb P}(v_L=y_L) = (\mu_L*\cdots*\mu_1)(G_{\ast}).
\]
\end{definition}

We now compute the Fourier transform $\widehat{\mu}_\ell(\lambda)$ for a nontrivial irreducible
representation $\lambda\neq\mathbf{1}$.
\begin{definition}[Spectral objects and effective parameters]\label{def:spectral-objects}
    Given $\ell\in[L]$, for an irreducible representation $\lambda$ of $\cG$, define the sample operator
    \[
    W_\ell(\lambda)\triangleq \sum_{h\in G^{L}\setminus\{g_\ell\}}\lambda(h).
    \]
    We also define the effective parameters
    \[
    \Delta_{L}\triangleq p_{L,1}-p_{L,3},
    \qquad
    \delta_{L}\triangleq p_{L,2}-p_{L,3},
    \qquad
    \sigma_{G^{L}}\triangleq \max_{\lambda\neq\mathbf{1}}\max_{\ell\in [L]}\|W_\ell(\lambda)\|_{\mathrm{op}},
    \]
    where $\|\cdot\|_{\text{op}}$ denotes the operator norm induced by $\|\cdot\|_2$, i.e., $\|A\|_{\text{op}}\coloneqq \sup_{\|x\|_2=1}\|Ax\|_2$.
    \end{definition}
    
    \begin{lemma}[Fourier transform of the one-step measure]\label{lem:mu-hat-decomp}
    Let $\mu_\ell$ be the one-step measure in \Cref{def:one-step-measure}.
    For any nontrivial irreducible representation $\lambda\neq\mathbf{1}$,
    \[
    \widehat{\mu}_\ell(\lambda)
    = \Delta_{L}\lambda(g_\ell) + \delta_{L}W_\ell(\lambda).
    \]
    \end{lemma}
    
    \begin{proof}
    By definition of $\mu_\ell$,
    \[
    \widehat{\mu}_\ell(\lambda)
    = p_{L,1}\lambda(g_\ell)
    + p_{L,2}\sum_{h\in G^{L}\setminus\{g_\ell\}}\lambda(h)
    + p_{L,3}\sum_{h\in \cG\setminus G^{L}}\lambda(h).
    \]
    For $\lambda\neq\mathbf{1}$, by \Cref{lem-inverse},  we have $\sum_{h\in\cG}\lambda(h)=0$, hence
    $\sum_{h\in \cG\setminus G^{L}}\lambda(h)=-(\lambda(g_\ell)+W_\ell(\lambda))$.
    Substituting and collecting terms yields the claim.
    \end{proof}

\subsection{Posterior Estimation}
Building on the operator decomposition, we compute the posterior probabilities by evaluating traces
of the resulting spectral products.
For notational convenience, we define the events
\[
E \triangleq \{v_L = y_L\},\qquad
A_\ell \triangleq \{u_\ell = g_\ell\},\qquad
B_\ell \triangleq \{u_\ell \in G^L\setminus\{g_\ell\}\}.
\]
Thus, the main task is to control $\tilde{\mathbb P}(E)$ as well as the joint probabilities
$\tilde{\mathbb P}(A_\ell\cap E)$ and $\tilde{\mathbb P}(B_\ell\cap E)$.

\begin{lemma}\label{lem:posterior-expansion}
    With $E,A_\ell,B_\ell$ defined above, we have the expansions
    \begin{subequations}\label{eq-exps}
            \begin{align}
    \tilde{\mathbb P}(E)
    &= \frac{1}{d} + \Big(1-\frac{1}{d}\Big)\Delta_{L}^{L} + \cR_E, \label{eq:PE-exp}\\
    \tilde{\mathbb P}(A_\ell\cap E)
    &= \frac{p_{L,1}}{d} + \Big(1-\frac{1}{d}\Big)p_{L,1}\Delta_{L}^{L-1} + \cR_A, \label{eq:PAE-exp}\\
    \tilde{\mathbb P}(B_\ell\cap E)
    &= \frac{(L-1)p_{L,2}}{d} + \cR_B. \label{eq:PBE-exp}
    \end{align}
    \end{subequations}
    Moreover, the remainders satisfy
    \begin{subequations}\label{eq-reminder}
            \begin{align}
    |\cR_E|
    &\le \Big(1-\frac{1}{d}\Big)\Big[
    (\Delta_{L}+\sigma_{G^L}\delta_{L})^{L}
    -\Delta_{L}^{L}
    -L\sigma_{G^L}\delta_{L}\Delta_{L}^{L-1}
    +(L-1)L\gamma(\cG)\delta_{L}\Delta_{L}^{L-1}
    \Big],\\
    |\cR_A|
    &\le p_{L,1}\Big(1-\frac{1}{d}\Big)\Big[
    (\Delta_{L}+\sigma_{G^L}\delta_{L})^{L-1}
    -\Delta_{L}^{L-1}
    -(L-1)\sigma_{G^L}\delta_{L}\Delta_{L}^{L-2}
    +(L-1)^2\gamma(\cG)\delta_{L}\Delta_{L}^{L-2}
    \Big],\\
    |\cR_B|
    &\le p_{L,2}\Big(1-\frac{1}{d}\Big)\Big[
    \sigma_{G^L}\big((\Delta_{L}+\sigma_{G^L}\delta_{L})^{L-1}-\Delta_{L}^{L-1}\big)
    +(L-1)\gamma(\cG)\Delta_{L}^{L-1}
    \Big].
    \end{align}
    \end{subequations}
    \end{lemma}
    
\begin{proof}

By \Cref{lem-inverse} and the convolution theorem, each quantity
        $\tilde{\mathbb P}(\cdot)$ can be written as a sum of traces of products of Fourier operators.
        We isolate the trivial-representation contribution and bound the remaining terms using the
        decomposition $\widehat{\mu}_k(\lambda)=\Delta_{L}\lambda(g_k)+\delta_{L}W_k(\lambda)$.
        We spell out the details for $\tilde{\mathbb P}(E)$; the bounds for
        $\tilde{\mathbb P}(A_\ell\cap E)$ and $\tilde{\mathbb P}(B_\ell\cap E)$ follow analogously.

\paragraph{Estimation of $\tilde{\mathbb P}(E)$.}   Let us start with $\tilde{\mathbb P}(E)$. By \Cref{lem-inverse}, we have 
\begin{align*} 
\tilde{\mathbb P}(E) &= (\mu_{L}\ast\cdots\ast\mu_{1})(G_{\ast})\\
&= \frac{1}{d} \sum_{\lambda \in \Lambda} d_\lambda \text{Tr}\left( \widehat{{\mu_{L}\ast\cdots\ast\mu_{1}}}(\lambda) \lambda(G_{\ast})^{-1} \right)\\
&=\frac{1}{d} \sum_{\lambda \in \Lambda} d_\lambda \text{Tr}\Bigg( \underbrace{\left[ \prod_{k=L}^1 \widehat{\mu}_k(\lambda) \right] }_{ =:\Pi (\lambda)}\lambda(G_{\ast})^{-1} \Bigg). 
\end{align*}
\begin{itemize}
    \item For $\lambda = \mathbf{1}$: $\hat{\mu}_k(\mathbf{1}) = 1$. Hence $\Pi(\lambda)=1$, and we can obtain 
    \begin{align*}
        \frac{1}{d}  d_\lambda \text{Tr}\left( \Pi(\lambda) \lambda(G_{\ast})^{-1} \right)=\frac{1}{d}  d_\mathbf{1}\cdot d_\mathbf{1}=\frac{1}{d}. 
    \end{align*}
    \item  For $\lambda \neq \mathbf{1}$:
by the decomposition $\hat{\mu}_k(\lambda) = \Delta_{L}\lambda(g_k) + \delta_{L}W_{k}(\lambda)$ from \Cref{lem:mu-hat-decomp}, we have 
\begin{align*}
    \Pi(\lambda)&=\prod_{k=L}^1 \big(\Delta_{L}\lambda(g_k) + \delta_{L}W_{k}(\lambda)\big) =\Delta_{L}^L\prod_{k=L}^1\lambda(g_k)+ T_{\mathrm{res}}(\lambda).
\end{align*}
Then the trace contribution by the first term is:
\begin{align*}
\text{Tr}\bigg(\Delta_{L}^L\big(\prod_{k=L}^1\lambda(g_k)\big) \lambda(G_{\ast})^{-1}\bigg) 
& = (\Delta_{L})^L \text{Tr}(\lambda(\prod_{k=L}^1g_k)\lambda(G_{\ast})^{-1}) = d_\lambda (\Delta_{L})^L.
\end{align*}

Summing this over all $\lambda \neq \mathbf{1}$, we obtain 
    \begin{align*}
       \frac{1}{d} \sum_{\lambda \neq \mathbf{1}} d_\lambda \text{Tr}\bigg(\Delta_{L}^L\big(\prod_{k=L}^1\lambda(g_k)\big) \lambda(G_{\ast})^{-1}\bigg)
      &=\frac{1}{d}  \sum_{\lambda \neq \mathbf{1}} d_\lambda \cdot d_\lambda (\Delta_{L})^L \\
      & =\Big(1-\frac{1}{d}\Big)(\Delta_{L})^L,
    \end{align*}
    where the last equality holds since $\sum_{\lambda \neq \mathbf{1}} d_\lambda^2 = d-1$ by \Cref{lem-inverse}.
Therefore, it suffices to control the operator norm of the residual term $T_{\mathrm{res}}(\lambda)$.
Notice that $T_{\mathrm{res}}(\lambda)$ can be expanded into $2^{L}-1$ terms, each of the form $M_L \cdots M_1$,
where for each $k$,
\[
M_k \in \left\{\Delta_{L}\lambda(g_k), \delta_{L}W_k(\lambda)\right\},
\]
and at least one factor $M_k$ equals $\delta_{L}W_k(\lambda)$.
We further decompose $T_{\mathrm{res}}(\lambda)$ into two parts,
$T_{\mathrm{res}}(\lambda) = T_{\mathrm{res},1}(\lambda) + T_{\mathrm{res},2}(\lambda)$:
    \begin{itemize}
        \item $T_{\mathrm{res},1}(\lambda)$ consists of the terms for which there exists a unique $k^\ast\in[L]$
    such that $M_{k^\ast}=\delta_{L}W_{k^\ast}(\lambda)$. In this case,
        \begin{align}
  & \Big|\text{Tr}\bigg(\delta_{L}\Delta_{L}^{L-1}\Big(\prod_{k=L}^{k^\ast+1}\lambda(g_k)\Big) W_{k^\ast}(\lambda) \Big(\prod_{k=k^\ast-1}^1\lambda(g_k)\Big)\lambda(G_{\ast})^{-1}\bigg) \Big| \notag\\
   &= \Big|\delta_{L}\Delta_{L}^{L-1}\text{Tr}(W_{k^\ast}(\lambda) \lambda(g_{k^\ast})^{-1})\Big| \notag\\
   & = \Big|\delta_{L}\Delta_{L}^{L-1} \sum_{g\in{G^{L}}\setminus\{g_{k^\ast}\}}\text{Tr}(\lambda(gg_{k^\ast}^{-1}))\Big| \notag\\
   &=\delta_{L}\Delta_{L}^{L-1} \sum_{g\in{G^{L}}\setminus\{g_{k^\ast}\}}\Big|\chi_\lambda(gg_{k^\ast}^{-1})\Big| \notag\\
   &\leq \delta_{L}\Delta_{L}^{L-1} (L-1)\cdot d_{\lambda}\gamma(\cG), \label{eq-non-unit}
\end{align}
where the last inequality uses $gg_{k^\ast}^{-1}\neq e$ (here $e$ denotes the identity element of $\cG$) and the definition of $\gamma(\cG)$.
    Since there are $L$ such terms in $T_{\mathrm{res},1}(\lambda)$, we obtain
\begin{align*}
       \Big|\frac{1}{d} \sum_{\lambda \neq \mathbf{1}} d_\lambda \text{Tr}\bigg(T_{\mathrm{res},1}(\lambda) \lambda(G_{\ast})^{-1}\bigg)\Big|\leq \Big(1-\frac{1}{d}\Big)\delta_{L}\Delta_{L}^{L-1} (L-1)L\gamma(\cG).
\end{align*}
\item $T_{\mathrm{res},2}(\lambda)$ collects the remaining $2^{L}-1-L$ terms, i.e., those for which at least
    two factors $M_k$ equal $\delta_{L}W_k(\lambda)$. Then
    \begin{align*}
        \|T_{\mathrm{res},2}(\lambda) \|_{\text{op}}&=\Bigg\|\sum_{M_{k}\in\{\Delta_{L}\lambda(g_k) , \delta_{L}W_{k}(\lambda)\}, \sum_{k=1}^{L} \1_{ M_{k}=\delta_{L}W_{k}(\lambda)}\ge 2} M_{L}\cdot ,\cdots,\cdot M_{1}\Bigg\|_{\text{op}}\\
        &\leq \sum_{M_{k}\in\{\Delta_{L}\lambda(g_k) , \delta_{L}W_{k}(\lambda)\}, \sum_{k=1}^{L} \1_{ M_{k}=\delta_{L}W_{k}(\lambda)}\ge 2} \prod_{k=L}^1 \|M_{k}\|_{\text{op}}\\
        &\overset{(a)}{\le}  \sum_{i=2}^L \binom{L}{i} \Delta_{L}^{L-i} (\sigma_{{G^{L}}}\delta_{L})^i= (\Delta_{L} + \sigma_{{G^{L}}}\delta_{L})^L - \Delta_{L}^L-L\sigma_{{G^{L}}}\delta_{L}\Delta_{L}^{L-1},
    \end{align*}
    Here \((a)\) holds since $\|\Delta_{L}\lambda(g_k)\|_{\text{op}}\le \Delta_{L}$,
    $\|\delta_{L}W_k(\lambda)\|_{\text{op}}\le \sigma_{G^{L}}\delta_{L}$, and there are exactly
    $\binom{L}{i}$ choices of indices for which $i$ different $M_k$'s equal $\delta_{L}W_k(\lambda)$.
    Consequently,
        \begin{align*}
        \Big|\frac{1}{d} \sum_{\lambda \neq \mathbf{1}} d_\lambda \text{Tr}\bigg(T_{\mathrm{res},2}(\lambda) \lambda(G_{\ast})^{-1}\bigg)\Big|\le\frac{1}{d}  \sum_{\lambda \neq \mathbf{1}} d_\lambda \cdot d_\lambda  \|T_{\mathrm{res},2}(\lambda) \|_{\text{op}}  \|\lambda(G_{\ast})^{-1}\|_{\text{op}} 
        \\=\Big(1-\frac{1}{d}\Big)\bigg((\Delta_{L} + \sigma_{{G^{L}}}\delta_{L})^L - \Delta_{L}^L-L\sigma_{{G^{L}}}\delta_{L}\Delta_{L}^{L-1}\bigg). 
    \end{align*}
    \end{itemize}

\end{itemize}
Putting everything together, we obtain
\[
\tilde{\mathbb{P}}(E)
= \frac{1}{d} + \Big(1-\frac{1}{d}\Big)\Delta_{L}^L + \cR_E,
\]
where the remainder term $\cR_E$ satisfies
\[
|\cR_E|
\le \Big(1-\frac{1}{d}\Big)\bigg(
(\Delta_{L} + \sigma_{G^{L}}\delta_{L})^L
- \Delta_{L}^L
- L\sigma_{G^{L}}\delta_{L}\Delta_{L}^{L-1}
+ \delta_{L}\Delta_{L}^{L-1}(L-1)L\gamma(\cG)
\bigg).
\]

\paragraph{Estimation of $\tilde{\mathbb P}(A_{\ell}\cap E)$.} The analysis is similar to that of $\tilde{\mathbb{P}}(E)$.
The key difference is that we replace the measure at step $\ell$ by the Dirac measure
$p_{L,1}\delta_{g_{\ell}}$, since the $\ell$-th step takes the action $g_{\ell}$.
Correspondingly, its Fourier transform becomes $p_{L,1}\lambda(g_{\ell})$.
Hence,
\begin{align*} 
\tilde{\mathbb P}(A_{\ell}\cap E) &= (\mu_{L}\ast\cdots p_{L,1}\delta_{g_{\ell}}\cdots\ast\mu_{1})(G_{\ast})\\
&=\frac{1}{d} \sum_{\lambda \in \Lambda} d_\lambda \text{Tr}\Big( \underbrace{\left[  \widehat{\mu}_L(\lambda)\cdots p_{L,1}\lambda(g_{\ell})\cdots \widehat{\mu}_1(\lambda) \right] }_{ =: \Pi_{A_{\ell}} (\lambda)}\lambda(G_{\ast})^{-1} \Big).
\end{align*}
\begin{itemize}
    \item For $\lambda = \mathbf{1}$: since $\widehat{\mu}_k(\mathbf{1})=1$, we have $\Pi_{A_{\ell}}(\mathbf{1})=p_{L,1}$ and thus
    \begin{align*}
        \frac{1}{d}  d_\lambda \text{Tr}\left( \Pi_{A_{\ell}}(\lambda) \lambda(G_{\ast})^{-1} \right)=\frac{p_{L,1}}{d}. 
    \end{align*}
    \item  For $\lambda \neq \mathbf{1}$: we can write
    \begin{align*}
   \Pi_{A_{\ell}} (\lambda)&=\Big(\prod_{k=L}^{\ell+1} \big(\Delta_{L}\lambda(g_k) + \delta_{L}W_{k}(\lambda)\big)\Big) \Big( p_{L,1}\lambda(g_{\ell})\Big) \Big(\prod_{k=\ell-1}^{1} \big(\Delta_{L}\lambda(g_k) + \delta_{L}W_{k}(\lambda)\big)\Big)\\
    &=p_{L,1}\Delta_{L}^{L-1}\Big(\prod_{k=L}^1\lambda(g_k)\Big) \lambda(g_{\ell}) \Big(\prod_{k=L}^1\lambda(g_k)\Big)+ T_{\mathrm{res},A_{\ell}}(\lambda)
\end{align*}
The trace contribution of the leading term is
\begin{align*}
  \text{Tr}\bigg(p_{L,1}\Delta_{L}^{L-1}\big(\prod_{k=L}^1\lambda(g_k)\big) \lambda(G_{\ast})^{-1}\bigg)  &  = p_{L,1}(\Delta_{L})^{L-1} \text{Tr}(\lambda(\prod_{k=L}^1g_k)\lambda(G_{\ast})^{-1}) \\
  & = d_\lambda p_{L,1} (\Delta_{L})^{L-1}.
\end{align*}
Summing over all $\lambda\neq \mathbf{1}$ yields
    \begin{align*}
        \frac{1}{d} \sum_{\lambda \neq \mathbf{1}} d_\lambda \text{Tr}\bigg(p_{L,1} (\Delta_{L})^{L-1}\big(\prod_{k=L}^1\lambda(g_k)\big) \lambda(G_{\ast})^{-1}\bigg) & =\frac{1}{d}  \sum_{\lambda \neq \mathbf{1}} d_\lambda \cdot d_\lambda p_{L,1} (\Delta_{L})^{L-1} \\
        & =\Big(1-\frac{1}{d}\Big)p_{L,1} (\Delta_{L})^{L-1}. 
    \end{align*}
    The residual term $T_{\mathrm{res},A_{\ell}}(\lambda)$ can be controlled exactly as in the analysis of
    $T_{\mathrm{res}}(\lambda)$, which gives
        \begin{align*}
        &\Big|\frac{1}{d} \sum_{\lambda \neq \mathbf{1}} d_\lambda \text{Tr}\bigg(T_{\mathrm{res},A_{\ell}}(\lambda) \lambda(G_{\ast})^{-1}\bigg)\Big|
        \\
        &\le  p_{L,1}\Big(1-\frac{1}{d}\Big)\bigg((\Delta_{L} + \sigma_{{G^{L}}}\delta_{L})^{L-1} - \Delta_{L}^{L-1}-(L-1)\sigma_{{G^{L}}}\delta_{L}\Delta_{L}^{L-2}+\delta_{L}\Delta_{L}^{L-2} (L-1)^2\gamma(\cG)
\bigg). 
    \end{align*}
\end{itemize}
Putting the above bounds together, we conclude that
\[
\tilde{\mathbb{P}}(A_{\ell}\cap E)
= \frac{p_{L,1}}{d}
+ \Big(1-\frac{1}{d}\Big)p_{L,1}\Delta_{L}^{L-1}
+ \cR_A,
\]
where
\[
|\cR_A|
\le p_{L,1}\Big(1-\frac{1}{d}\Big)\bigg(
(\Delta_{L}+\sigma_{G^{L}}\delta_{L})^{L-1}
-\Delta_{L}^{L-1}
-(L-1)\sigma_{G^{L}}\delta_{L}\Delta_{L}^{L-2}
+\delta_{L}\Delta_{L}^{L-2}(L-1)^2\gamma(\cG)
\bigg).
\]

\paragraph{Estimation of $\tilde{\mathbb P}(B_{\ell}\cap E)$.} For $B_{\ell}\cap E$, at step $\ell$ we use the measure $p_{L,2}\delta_{G^{L}\setminus\{g_{\ell}\}}$,
whose Fourier operator is
\[
\widehat{\mu}_{B_\ell}(\lambda)
= p_{L,2}\sum_{g\in G^{L}\setminus\{g_{\ell}\}}\lambda(g)
= p_{L,2}W_{\ell}(\lambda).
\]
Hence,
\begin{align*} 
\tilde{\mathbb P}(B_{\ell}\cap E) &= (\mu_{L}\ast\cdots p_{L,2}\delta_{ G^{L}\setminus \{g_{\ell}\}}\cdots\ast\mu_{1})(G_{\ast})\\
&=\frac{1}{d} \sum_{\lambda \in \Lambda} d_\lambda \text{Tr}\Big( \underbrace{\left[  \widehat{\mu}_L(\lambda)\cdots \hat{\mu}_{B_\ell}(\lambda)\cdots \widehat{\mu}_1(\lambda) \right] }_{ =: \Pi_{B_{\ell}} (\lambda)}\lambda(G_{\ast})^{-1} \Big).
\end{align*}
\begin{itemize}
    \item For $\lambda = \mathbf{1}$: since $\widehat{\mu}_k(\mathbf{1})=1$, we have $\Pi_{B_{\ell}}(\mathbf{1})=p_{L,2}(L-1)$, and thus
    \begin{align*}
        \frac{1}{d}  d_\lambda \text{Tr}\left( \Pi_{B_{\ell}}(\lambda) \lambda(G_{\ast})^{-1} \right)=\frac{(L-1)\cdot p_{L,2}}{d}. 
    \end{align*}
    \item  For $\lambda \neq \mathbf{1}$:  analogous to the decomposition of $T_{\mathrm{res}}(\lambda)$, the operator $\Pi_{B_{\ell}}(\lambda)$ can be
    expanded into $2^{L-1}$ terms of the form $M_L\cdots \big(p_{L,2}W_{\ell}(\lambda)\big)\cdots M_1$, where, for each $k\neq \ell$,
\[
M_k \in \{\Delta_{L}\lambda(g_k), \delta_{L}W_k(\lambda)\}.
\]
We further split $\Pi_{B_{\ell}}(\lambda)$ into two parts,
\[
\Pi_{B_{\ell}}(\lambda)=T_{B_{\ell},1}(\lambda)+T_{B_{\ell},2}(\lambda).
\]
  \begin{itemize}
        \item $T_{B_{\ell},1}$ consists of the unique term for which $M_k=\Delta_{L}\lambda(g_k)$ for all $k\neq \ell$.
In this case,
        \begin{align*}
  & \Big|\text{Tr}\bigg(\Delta_{L}^{L-1}\Big(\prod_{k=L}^{k^\ast+1}\lambda(g_k)\Big) p_{L,2}W_{\ell}(\lambda) \Big(\prod_{k=k^\ast-1}^1\lambda(g_k)\Big)\lambda(G_{\ast})^{-1}\bigg) \Big|\\
   &\leq p_{L,2}\Delta_{L}^{L-1}(L-1)\cdot d_{\lambda}\gamma(\cG),
\end{align*}
where the inequality follows by an argument analogous to \cref{eq-non-unit}.
 Therefore, we have:
\begin{align*}
       \Big|\frac{1}{d} \sum_{\lambda \neq \mathbf{1}} d_\lambda \text{Tr}\bigg(T_{B_{\ell},1}(\lambda) \lambda(G_{\ast})^{-1}\bigg)\Big|\leq \Big(1-\frac{1}{d}\Big)p_{L,2}\Delta_{L}^{L-1} (L-1)\gamma(\cG).
\end{align*}
\item $T_{B_{\ell},2}$ collects the remaining terms, i.e., those for which at least one index $k\neq \ell$
satisfies $M_k=\delta_{L}W_k(\lambda)$. Then
    \begin{align*}
        \|T_{\mathrm{res},2}(\lambda) \|_{\text{op}} &=\Bigg\|\sum_{M_{k}\in\{\Delta_{L}\lambda(g_k) , \delta_{L}W_{k}(\lambda)\} \text{ for }k\neq \ell, \sum_{k\neq \ell} \1_{ M_{k}=\delta_{L}W_{k}(\lambda)}\ge 1} M_{L}\cdots  p_{L,2}W_{\ell}(\lambda)\cdots M_{1}\Bigg\|_{\text{op}}\\
        &~~~~\le  p_{L,2}\sigma_{{G^{L}}}\sum_{i=1}^L \binom{L}{i} \Delta_{L}^{L-i} (\sigma_{{G^{L}}}\delta_{L})^i\\
        &~~~~=p_{L,2}\sigma_{{G^{L}}}\Big( (\Delta_{L} + \sigma_{{G^{L}}}\delta_{L})^{L-1} - \Delta_{L}^{L-1}\Big),
    \end{align*}
    which can be shown by the same argument as in the bound for $T_{\mathrm{res},2}$. 
    Consequently,
        \begin{align*}
        \Big|\frac{1}{d} \sum_{\lambda \neq \mathbf{1}} d_\lambda \text{Tr}\bigg(T_{B_{\ell},2}(\lambda) \lambda(G_{\ast})^{-1}\bigg)\Big| & \le\frac{1}{d}  \sum_{\lambda \neq \mathbf{1}} d_\lambda \cdot d_\lambda  \|T_{B_\ell,2} \|_{\text{op}}  \|\lambda(G_{\ast})^{-1}\|_{\text{op}} 
        \\ 
        &=\Big(1-\frac{1}{d}\Big)p_{L,2}\sigma_{{G^{L}}}\Big( (\Delta_{L} + \sigma_{{G^{L}}}\delta_{L})^{L-1} - \Delta_{L}^{L-1}\Big). 
    \end{align*}
    \end{itemize}

\end{itemize}
Putting everything together, we obtain
\[
\tilde{\mathbb{P}}(B_\ell \cap E)
= \frac{(L-1)p_{L,2}}{d} + \cR_{B},
\]
where the remainder term $\cR_{B}$ satisfies
\[
|\cR_{B}|
\le \Big(1-\frac{1}{d}\Big)p_{L,2}\bigg(
\sigma_{G^{L}}\Big( (\Delta_{L} + \sigma_{G^{L}}\delta_{L})^{L-1} - \Delta_{L}^{L-1}\Big)
+\Delta_{L}^{L-1}(L-1)\gamma(\cG)
\bigg).
\]

\end{proof}

The expansions in \Cref{lem:posterior-expansion} immediately imply the following deviations of the
posterior probabilities $\rho_{\ell,1}$ and $\rho_{\ell,2}$ from their corresponding priors.

\begin{proposition}[Posterior deviation and dominant term]\label{prop:posterior-dev}
    The posterior deviations admit the exact identities
    \begin{subequations}
    \label{eq-posterior-gap}
        \begin{align}
    \rho_{\ell,1}-p_{L,1}
    &=
    \frac{
    p_{L,1}\Delta_{L}^{L-1}(1-\Delta_{L})\Big(1-\frac{1}{d}\Big)
    +\cR_A-p_{L,1}\cR_E
    }{
    \tilde{\mathbb P}(E)
    },
    \label{eq:rho1-dev}
    \\
    p_{L,2}- \frac{\rho_{\ell,2}}{L-1}
    &=
    \frac{
    p_{L,2}\Delta_{L}^{L}\Big(1-\frac{1}{d}\Big)
    +p_{L,2}\cR_E-\frac{\cR_B}{L-1}
    }{
    \tilde{\mathbb P}(E)
    }.
    \label{eq:rho2-dev}
    \end{align}
    \end{subequations}

    Moreover, if
    \begin{equation}\label{eq:small-noise-cond}
    \frac{\sigma_{G^{L}}\delta_{L}}{\Delta_{L}} \ll \frac{1}{L},
    \end{equation}
    then the remainder terms satisfy
    \begin{subequations}  \label{eq-residual-bound}
          \begin{align}
    |\cR_A|
    &\le
    p_{L,1}\Delta_{L}^{L-1}\Big(1-\frac{1}{d}\Big)
    \left(
    O\left(\frac{\sigma_{G^{L}}^{2}\delta_{L}^{2}}{\Delta_{L}^{2}}\right)
    +\frac{(L-1)^2\gamma(\cG)\delta_{L}}{\Delta_{L}}
    \right),
    \label{eq:RA-bound}
    \\
    |\cR_E|
    &\le
    \Delta_{L}^{L}\Big(1-\frac{1}{d}\Big)
    \left(
    O\left(\frac{\sigma_{G^{L}}^{2}\delta_{L}^{2}}{\Delta_{L}^{2}}\right)
    +\frac{(L-1)L\gamma(\cG)\delta_{L}}{\Delta_{L}}
    \right),
    \label{eq:RE-bound}
    \\
    \frac{|\cR_B|}{L-1}
    &\le
    p_{L,2}\Delta_{L}^{L-1}\Big(1-\frac{1}{d}\Big)
    \left(
    O\left(\frac{\sigma_{G^{L}}^{2}\delta_{L}}{(L-1)\Delta_{L}}\right)
    +\gamma(\cG)
    \right).
    \label{eq:RB-bound}
    \end{align}
    \end{subequations}
  
    \end{proposition}

\begin{proof}
    For $\rho_{\ell,1}-p_{L,1}$, by \Cref{lem:posterior-expansion} we have
    \begin{align*}
    \rho_{\ell,1}-p_{L,1}
    &=\frac{\tilde{\mathbb P}(A_{\ell}\cap E)-p_{L,1}\tilde{\mathbb P}(E)}{\tilde{\mathbb P}(E)}\\
    &=\frac{\frac{p_{L,1}}{d} + p_{L,1}\Delta_{L}^{L-1} \Big(1 - \frac{1}{d}\Big) + \cR_A
    - p_{L,1}\Big(\frac{1}{d} +\Delta_{L}^{L} \Big(1 - \frac{1}{d}\Big) + \cR_E\Big)}{\tilde{\mathbb P}(E)}\\
    &=\frac{p_{L,1}\Delta_{L}^{L-1} (1-\Delta_{L})\Big(1 - \frac{1}{d}\Big) +\cR_A-p_{L,1} \cR_E}{\tilde{\mathbb P}(E)},
    \end{align*}
    which gives \cref{eq:rho1-dev}.
    For $p_{L,2}-\frac{\rho_{\ell,2}}{L-1}$, we similarly obtain
    \begin{align*}
    p_{L,2}-\frac{\rho_{\ell,2}}{L-1}
    &=\frac{p_{L,2}(L-1)\tilde{\mathbb P}(E)-\tilde{\mathbb P}(B_{\ell}\cap E)}{(L-1)\tilde{\mathbb P}(E)}\\
    &=\frac{p_{L,2}(L-1)\Big(\frac{1}{d} + \Delta_{L}^L \Big(1 - \frac{1}{d}\Big) + \cR_E\Big)
    -\frac{(L-1)p_{L,2}}{d} - \cR_{B}}{(L-1)\tilde{\mathbb P}(E)}\\
    &=\frac{p_{L,2}\Delta_{L}^L \Big(1 - \frac{1}{d}\Big) + p_{L,2}\cR_E-\cR_{B}/(L-1)}{\tilde{\mathbb P}(E)},
    \end{align*}
    which gives \cref{eq:rho2-dev}.
    
    It remains to bound $\cR_E$, $\cR_A$, and $\cR_B$ under \cref{eq:small-noise-cond}.
    For notational simplicity, let $x\triangleq \sigma_{G^L}\delta_{L}/\Delta_{L}$. Then $$(\Delta_{L}+\sigma_{G^L}\delta_{L})^{k}=\Delta_{L}^{k}(1+x)^{k}.$$
    
\paragraph{Bounds for $\cR_E$ and $\cR_A$.}
    The expressions in the brackets for $\cR_E$ and $\cR_A$ contain
    \[
    (1+x)^k-1-kx,\qquad k\in\{L,L-1\}.
    \]
    Under $x\ll 1/L$, the second-order Taylor remainder gives $(1+x)^k-1-kx = O(k^2x^2)$, which implies
    \[
    (\Delta_{L}+\sigma_{G^L}\delta_{L})^k-\Delta_{L}^k-k\sigma_{G^L}\delta_{L}\Delta_{L}^{k-1}
    =\Delta_{L}^k\cdot O(k^2x^2)
    =O\big(k^2\sigma_{G^L}^2\delta_{L}^2\Delta_{L}^{k-2}\big).
    \]
    Substituting this estimate into the displayed bounds for $\cR_E$ and $\cR_A$ yields the claimed
    controls for $|\cR_E|$ and $|\cR_A|$.
    
\paragraph{Bound for $\cR_B$.}
    Here the bracket contains $(\Delta_{L}+\sigma_{G^L}\delta_{L})^{L-1}-\Delta_{L}^{L-1}
    =\Delta_{L}^{L-1}\big((1+x)^{L-1}-1\big)$.
    Under $x\ll 1/L$, the first-order estimate gives $(1+x)^{L-1}-1=O((L-1)x)$, hence
    \[
    \sigma_{G^L}\big((\Delta_{L}+\sigma_{G^L}\delta_{L})^{L-1}-\Delta_{L}^{L-1}\big)
    = \sigma_{G^L}\Delta_{L}^{L-1}\cdot O((L-1)\sigma_{G^L}\delta_{L}/\Delta_{L})
    = O\big((L-1)\sigma_{G^L}^2\delta_{L}\Delta_{L}^{L-2}\big).
    \]
    Plugging this into the displayed bound for $\cR_B$ yields the stated control on $|\cR_B|$.
    
    \end{proof}

\subsection{Gradient Characterization: Proof of Lemma~\ref{lem-grad-char-tech}} 
Combining the posterior deviations in \Cref{prop:posterior-dev} with
\cref{eq:J-trajectory-form-conditional}, we obtain the following characterization of the gradient, which is a formal version of \Cref{lem-grad-char-tech}.

\begin{proposition}[Gradient characterization]\label{prop:grad-char}
    Given problem length $L$,  suppose that
    \[
    \frac{L^2 \delta_{L}}{\Delta_{L}} = o(1)\cdot(1-\Delta_{L})
    \qquad\text{and}\qquad
    \frac{p_{L,2}}{p_{L,1}} = o(1)\cdot(1-\Delta_{L}).
    \]
    Then
    \begin{align*}
    \nabla_q \widetilde\cJ_{L}
    &= \Theta(\log d /\dpos)\cdot p_{L,1}\Delta_{L}^{L-1}\Big(1-\frac{1}{d}\Big)(1-\Delta_{L}),
    \\
    |\nabla_r \widetilde\cJ_{L}|
    &= O(1/\dpos)\cdot \nabla_q \widetilde\cJ_{L}.
    \end{align*}
    \end{proposition}
    
\begin{proof}
    Recall that
    \begin{align}
    \nabla_q \widetilde\cJ_{L}
    =
   \frac{1}{2L\dpos} \E_{Z^{L}}\Big[
    \attn_{L}\cdot B\cdot \underbrace{\tilde{\mathbb P}(E)
    \sum_{\ell=1}^{L}\Big(
    (\rho_{\ell,1}-p_{L,1})(2-\attn_{L})
    +
    \Big(p_{L,2}-\frac{\rho_{\ell,2}}{L-1}\Big)(1-\attn_{L})
    \Big)}_{ =: J_{\text{gap}}}
    \Big]. \label{eq:grad-q-form-re}
    \end{align}
    Note that $2-\attn_{L}\ge 1$ and $1-\attn_{L}\in(0,1)$.  Under the stated assumptions, since $\sigma_{G^{L}}\le L-1$, the condition
    $\sigma_{G^{L}}\delta_{L}/\Delta_{L}\ll 1/L$ holds. Hence \Cref{prop:posterior-dev} applies.
    We now verify that for $J_{\text{gap}}$, all remainder contributions are negligible compared to the leading term
    $p_{L,1}\Delta_{L}^{L-1}\big(1-\frac{1}{d}\big)(1-\Delta_{L})$.
    
    \begin{itemize}
    \item \textbf{Bounding $\cR_A$.}
    By \cref{eq:RA-bound} in \Cref{prop:posterior-dev},
    \begin{align*}
    |\cR_A|
    &\le
    p_{L,1}\Delta_{L}^{L-1}\Big(1-\frac{1}{d}\Big)
    \left(
    O\left(\frac{\sigma_{G^{L}}^{2}\delta_{L}^{2}}{\Delta_{L}^{2}}\right)
    +\frac{(L-1)^2\gamma(\cG)\delta_{L}}{\Delta_{L}}
    \right).
    \end{align*}
    Using $\sigma_{G^{L}}\le L-1$ and $\frac{L^{2}\delta_{L}}{\Delta_{L}}=o(1)\cdot(1-\Delta_{L})$,
    we have
    \[
    \frac{\sigma_{G^{L}}^{2}\delta_{L}^{2}}{\Delta_{L}^{2}}
    \le
    \frac{L^{2}\delta_{L}}{\Delta_{L}}\cdot \frac{\delta_{L}}{\Delta_{L}}
    =
    o(1)\cdot(1-\Delta_{L})\cdot \frac{\delta_{L}}{\Delta_{L}}
    =
    o(1)\cdot(1-\Delta_{L}),
    \]
    and similarly
    \[
    \frac{(L-1)^2\gamma(\cG)\delta_{L}}{\Delta_{L}}
    \le
    \gamma(\cG)\cdot o(1)\cdot(1-\Delta_{L}).
    \]
    Therefore,
    \begin{align*}
    |\cR_A|
    &\le
    p_{L,1}\Delta_{L}^{L-1}\Big(1-\frac{1}{d}\Big)(1-\Delta_{L})
    \big(o(1)+o(1)\gamma(\cG)\big)
    \ll
    p_{L,1}\Delta_{L}^{L-1}\Big(1-\frac{1}{d}\Big)(1-\Delta_{L}).
    \end{align*}
    
    \item \textbf{Bounding $\cR_E$.}
    By \cref{eq:RE-bound} in \Cref{prop:posterior-dev},
    \begin{align*}
    |\cR_E|
    &\le
    \Delta_{L}^{L}\Big(1-\frac{1}{d}\Big)
    \left(
    O\left(\frac{\sigma_{G^{L}}^{2}\delta_{L}^{2}}{\Delta_{L}^{2}}\right)
    +\frac{(L-1)L\gamma(\cG)\delta_{L}}{\Delta_{L}}
    \right).
    \end{align*}
    Using the same estimates as above and $\Delta_{L}\le 1$, we obtain
    \[
    |\cR_E|
    \ll
    \Delta_{L}^{L}\Big(1-\frac{1}{d}\Big)(1-\Delta_{L}).
    \]
     Consequently, $p_{L,1} |\cR_E|$ is dominated by $p_{L,1}\Delta_{L}^{L}\Big(1-\frac{1}{d}\Big)(1-\Delta_{L})$.
    
    \item \textbf{Bounding $\cR_B$.}
    By \cref{eq:RB-bound} in \Cref{prop:posterior-dev},
    \begin{align*}
    \frac{|\cR_B|}{L-1}
    &\le
    p_{L,2}\Delta_{L}^{L-1}\Big(1-\frac{1}{d}\Big)
    \left(
    O\left(\frac{\sigma_{G^{L}}^{2}\delta_{L}}{(L-1)\Delta_{L}}\right)
    +\gamma(\cG)
    \right).
    \end{align*}
    Using $\sigma_{G^{L}}\le L-1$ and $\frac{L^{2}\delta_{L}}{\Delta_{L}}=o(1)\cdot(1-\Delta_{L})$,
    we have
    \[
    \frac{\sigma_{G^{L}}^{2}\delta_{L}}{(L-1)\Delta_{L}}
    \le
    \frac{L^{2}\delta_{L}}{\Delta_{L}}
    =
    o(1)\cdot(1-\Delta_{L}).
    \]
    Moreover, $\frac{p_{L,2}}{p_{L,1}}=o(1)\cdot(1-\Delta_{L})$ implies
    \[
    p_{L,2}\Delta_{L}^{L-1}
    \le
    p_{L,1}\Delta_{L}^{L-1}\cdot o(1)\cdot(1-\Delta_{L}).
    \]
    Thus,
    \begin{align*}
    \frac{|\cR_B|}{L-1}
    &\le
    p_{L,1}\Delta_{L}^{L-1}\Big(1-\frac{1}{d}\Big)(1-\Delta_{L})
    \Big(o(1)+\gamma(\cG)\cdot o(1)\Big)
    \ll
    p_{L,1}\Delta_{L}^{L-1}\Big(1-\frac{1}{d}\Big)(1-\Delta_{L}).
    \end{align*}
    Finally, the same assumption $\frac{p_{L,2}}{p_{L,1}}=o(1)\cdot(1-\Delta_{L})$ also yields
    \[
    p_{L,2}\Delta_{L}^{L}\Big(1-\frac{1}{d}\Big)
    \ll
    p_{L,1}\Delta_{L}^{L-1}\Big(1-\frac{1}{d}\Big)(1-\Delta_{L}),
    \]
    so the contribution of the second posterior deviation term is dominated by the first term.
    \end{itemize}
    Plugging the above bounds into the expression for $\nabla_q \cJ_L$, and using that
    $\attn_L=\Theta(1)$, which is implied by $\frac{p_{L,2}}{p_{L,1}}=o(1)\cdot(1-\Delta_{L})$, we conclude that
    \[
    \nabla_q \widetilde\cJ_{L}
    = \Theta(\log d /\dpos)\cdot p_{L,1}\Delta_{L}^{L-1}\Big(1-\frac{1}{d}\Big)(1-\Delta_{L}).
    \]
    
    The analysis for $|\nabla_r \widetilde\cJ_{L}|$ is similar. Alternatively, we may invoke the direct comparison
    bound in \Cref{lem-gd-main} to obtain
    $|\nabla_r \widetilde\cJ_L| = O(1/\dpos)\cdot \nabla_q \widetilde\cJ_L$.
    \end{proof}

\subsection{Exponentially Flat Region for Long-Horizon Tasks: Proof of Proposition~\ref{prop:flat-region}}
Following the same decomposition underlying \Cref{lem:posterior-expansion},
we show that when the step-invariant probability tuple $(p_{L,1},p_{L,2},p_{L,3})$
has small effective margins
$\Delta_L:=p_{L,1}-p_{L,3}$ and $\delta_L:=p_{L,2}-p_{L,3}$, the resulting policy gradient is upper bounded by a quantity that decays
exponentially in the horizon length $L$.
We then specialize this general exponential barrier to our concrete setting,
which immediately yields \Cref{prop:flat-region}.
\begin{proposition}\label{prop-flat-gen}
Under Assumptions \ref{assum-group}--\ref{ass:pretrained-mlp},
for any $2\le L\le L_{\max}$, suppose the step-invariant probability tuple
$(p_{L,1},p_{L,2},p_{L,3})$ satisfies, with
$\Delta_L:=p_{L,1}-p_{L,3}$ and $\delta_L:=p_{L,2}-p_{L,3}$,
\begin{align}
\Delta_L + L\delta_L \le \tilde{O}\left(d^{-\Omega(1)}\right),
\qquad
p_{L,i}\le d^{-\Omega(1)}\ \ \text{for } i\in[3]. \label{eq:flat-region-cond}
\end{align}
Then,
\begin{equation}
\left| \nabla_{q} \widetilde\cJ_{L} \right|
\le \tilde{O}\left( \frac{1}{\dpos} \right)\cdot d^{-\Omega(L)},
\qquad 
\left| \nabla_{r} \widetilde\cJ_{L} \right|
\le \tilde{O}\left( \frac{1}{\dpos^2} \right)\cdot d^{-\Omega(L)}.
\end{equation}
\end{proposition}
\begin{proof}
A key takeaway from \Cref{lem:posterior-expansion} is that,
when bounding the remainder contributions (e.g., $\cR_E$),
we decompose the remainder term $T_{\mathrm{res}}(\lambda)$ into several parts.
Independent of this finer decomposition, its operator norm admits the crude bound
\[
\|T_{\mathrm{res}}(\lambda)\|_{\mathrm{op}}
\le (\Delta_L+\sigma_{G_L}\delta_L)^{L}-\Delta_L^L.
\]
Using this bound directly gives
\[
\cR_E \le \Bigl(1-\frac{1}{d}\Bigr)(\Delta_L+\sigma_{G_L}\delta_L)^{L}.
\]
The same argument applies to $\cR_A$ and $\cR_B$, yielding
\[
\cR_A \le \Bigl(1-\frac{1}{d}\Bigr)p_{L,1}(\Delta_L+\sigma_{G_L}\delta_L)^{L-1},
\qquad
\cR_B \le \Bigl(1-\frac{1}{d}\Bigr)p_{L,2}\sigma_{G_L}(\Delta_L+\sigma_{G_L}\delta_L)^{L-1}.
\]
Invoking \cref{eq-posterior-gap} from \Cref{prop:posterior-dev} and substituting the above bounds into \cref{eq:grad-q-form-re}, we obtain
\begin{align*}
\left| \nabla_{q} \widetilde\cJ_{L} \right|
&\le \tilde{O}\left( \frac{B}{\dpos} \right)
\Bigl(
p_{L,1}\Delta_L^{L-1}
+
p_{L,1}(\Delta_L+\sigma_{G_L}\delta_L)^{L-1}
+
p_{L,2}\frac{\sigma_{G_L}}{L-1}(\Delta_L+\sigma_{G_L}\delta_L)^{L-1}
\Bigr)\\
&\le \tilde{O}\left( \frac{1}{\dpos} \right)\cdot d^{-\Omega(L)},
\end{align*}
where in the last step we use $\sigma_{G_L}\le L-1$ together with the assumptions
$\Delta_L+L\delta_L\le \tilde{O}\left(d^{-\Omega(1)}\right)$ and
$p_{L,i}\le d^{-\Omega(1)}$ for $i\in[3]$.
The bound for $\left|\nabla_{r}\widetilde\cJ_L\right|$ follows by the same reasoning and is omitted.
\end{proof}

\begin{proposition}[\Cref{prop:flat-region} restated]\label{prop-trap-appendix}
Under Assumptions~\ref{assum-group}--\ref{ass:pretrained-mlp},
suppose $\TF_{\theta^{(0)}}$ is initialized according to \Cref{assump-init}.
Then for any horizon $L>2C_B$, whenever the feature magnitudes satisfy
$\max\{|r^{(t)}|,|q^{(t)}|\}\le 0.01$, we have $\cJ_{L}^{(t)}=\frac{1}{d}(1\pm o(1))$, and
\[
\big|\nabla_{q} \widetilde\cJ_L^{(t)}\big|
\le \tilde{O}\left(\frac{1}{\dpos}\right)\cdot d^{-\Omega(L)},
\qquad
\big|\nabla_{r} \widetilde\cJ_L^{(t)}\big|
\le \tilde{O}\left(\frac{1}{\dpos^2}\right)\cdot d^{-\Omega(L)}.
\]
\end{proposition}

\begin{proof}
Since $\max\{|r^{(t)}|,|q^{(t)}|\}\le 0.01$, the attention weights satisfy
\[
\attn^{(t)}_{L} C_B
\le \frac{C_B e^{0.02}}{e^{0.02}+L-1} < 1,
\qquad
\frac{1-\attn^{(t)}_{L}}{L-1} C_B
\le \frac{C_B}{e^{-0.02}+L-1} < 1.
\]
In particular, this implies $p_{L,i}^{(t)}\le d^{-\Omega(1)}$ for all $i\in[3]$.
It remains to bound $\Delta_L^{(t)}+L\delta_L^{(t)}$, which we do by considering two regimes.

\paragraph{Case 1: $L<d^{0.01}$.}
By \Cref{lem:pi-form},
\[
p_{L,1}^{(t)}
\le O\left(\frac{1}{L+d^{1-\frac{C_B e^{0.02}}{e^{0.02}+L-1}}}\right)
= d^{-\Omega(1)},
\qquad
p_{L,2}^{(t)}
\le O\left(\frac{1}{L+d^{1-\frac{C_B}{e^{-0.02}+L-1}}}\right)
\le d^{-0.5}.
\]
Therefore,
\[
\Delta_L^{(t)}+L\delta_L^{(t)}
\le p_{L,1}^{(t)}+L p_{L,2}^{(t)}
\le d^{-\Omega(1)}.
\]

\paragraph{Case 2: $L\ge d^{0.01}$.}
In this regime, we bound $\Delta_L^{(t)}$ and $\delta_L^{(t)}$ directly.
In particular,
\[
\Delta_L^{(t)}
\le O\left(\frac{e^{\frac{C_B e^{0.02}}{e^{0.02}+L-1}\log d}-1}{d}\right)
\le \tilde{O}\left(\frac{1}{Ld}\right),
\]
and
\[
L\delta_L^{(t)}
\le O\left(
L\cdot\frac{e^{\frac{C_B}{e^{-0.02}+L-1}\log d}-1}{d}
\right)
\le \tilde{O}\left(\frac{1}{d}\right).
\]
Thus,
\[
\Delta_L^{(t)}+L\delta_L^{(t)}
\le \tilde{O}\left(\frac{1}{d}\right).
\]
In both regimes, the conditions of \Cref{prop-flat-gen} are satisfied.
Therefore, applying \Cref{prop-flat-gen} yields the desired gradient bound.
Moreover, \(\cJ_{L}^{(t)}=\frac{1}{d}(1\pm o(1))\) follows directly from \Cref{cor:reward-crude}.

\end{proof}

\subsection{Reward Characterization}

Note that $\tilde{\mathbb{P}}(E)$ is exactly the expected reward for a fixed instance $(y_0, G_L)$.
Consequently,
\[
\cJ_L=\E_{Z^L}\big[\tilde{\mathbb{P}}(E)\big].
\]
Therefore, the gradient characterization in \Cref{prop:grad-char} immediately yields a corresponding characterization of the reward.

\begin{lemma}\label{cor:reward-characterization}
Given a problem of length $L$, suppose that
\[
\frac{L^2 \delta_{L}}{\Delta_{L}} = o(1)\cdot(1-\Delta_{L})
\qquad\text{and}\qquad
\frac{p_{L,2}}{p_{L,1}} = o(1)\cdot(1-\Delta_{L}).
\]
Then,
\[
\cJ_L
=
\frac{1}{d}
+\Bigl(1-\frac{1}{d}\Bigr)(1\pm o(1))\cdot \Delta_{L}^{L}.
\]
\end{lemma}

Moreover, by adapting the argument in \Cref{prop-flat-gen} to control the residual term $\cR_E$, we obtain the following coarse upper bound.

\begin{lemma}\label{cor:reward-crude}
Given a problem of length $L$, we have
\[
\Big|\cJ_L-\frac{1}{d}\Big|
\le \Bigl(1-\frac{1}{d}\Bigr)\bigl(\Delta_L+\sigma_{G_L}\delta_L\bigr)^{L}.
\]
\end{lemma}

\section{Learning Dynamics of Short-horizon RL}\label{sec:constant-len}

In this section, we focus on the regime $L \le C_B$.  Our analysis tracks the training dynamics of the two scalar quantities $q$ and $r$
defined in \cref{eq:q-def} and \cref{eq:r-def}.
We proceed in three steps.
First, we state an induction hypothesis that is maintained throughout training.
Second, under this hypothesis, we derive one-step update bounds for $q$ and $r$.
Finally, we close the induction by showing that the hypothesis holds for all iterations.

We will focus on the RL training dynamics; the same proof structure and bookkeeping apply to SFT training.
Accordingly, at the end of this section, we briefly list the key lemmas and the corresponding induction for SFT, and omit the details.

\begin{induction}\label{induction-con}
    Given $\Omega(\frac{1}{\polylog d})<\epsilon<\frac{1}{4}$, and let $T_{1}$ be the first iteration such that
    $\attn_{L}^{(t)} \ge 1-\epsilon$.
    Then for every iteration $t<T_{1}$, the following statements hold:
    \begin{enumerate}[(a)]
        \item $O\big(\log \frac{L}{\epsilon}\big)\ge q^{(t)}\ge 0$, and $q^{(t)}$ is monotonically nondecreasing in $t$ (starting from $0$);
        \item $|r^{(t)}|\le O(1/\dpos) q^{(t)}$.
    \end{enumerate}
    \end{induction}
    
    \subsection{Attention and Logit Preliminaries}\label{sec-attn-con}
  We first introduce several properties of the attention scores and logits if \Cref{induction-con} holds.  
\begin{lemma}\label{lem-attn-con}
    If \Cref{induction-con} holds for all iterations $<t$,  then we have 
                \begin{enumerate}[(a)]
        \item $\attn_{L}^{(t)}=\frac{e^{q^{(t)}-r^{(t)}}}{e^{q^{(t)}-r^{(t)}}+(L-1)}\ge \frac{1}{L}$; 
         \item  $\attn^{(t)}_{a,\ell-1\to p, k}= \frac{1}{(L-1)+e^{q^{(t)}-r^{(t)}} }=\frac{1}{L-1}\big(1-\attn_{L}^{(t)}\big)$ for $k\neq \ell$.
        \end{enumerate} 
\end{lemma}

Therefore, direct calculations by combining \Cref{lem-attn-con} and \Cref{lem:pi-form} yield the following lemma.

\begin{lemma}\label{lem-logit-con}
    Assume that \Cref{induction-con} holds for all iterations $<t$. We have 
    $$p^{(t)}_{L,1}\geq \Omega(1),\quad 1-p^{(t)}_{L,1}\geq\Omega\Big(\frac{1}{d^{(1-\epsilon)C_B-1}}\Big)$$ and the following bounds
    on the transition probabilities $p^{(t)}_{L,2}$ and $p^{(t)}_{L,3}$.
    
    \begin{enumerate}[(1)]
    \item \textbf{Regime I:} if $\attn^{(t)}_{L}< 1-\frac{L-1}{C_{B}}$, then
    \begin{enumerate}[(a)]
    \item \emph{in-context distractor transition}
    \[
    p_{L,2}^{(t)}
    = \Theta\left(d^{-\left(\attn_{L}^{(t)}-\frac{1-\attn_{L}^{(t)}}{L-1}\right)C_B}\right)
    = O\left(\frac{1}{L}\right)\Big(1-p_{L,1}^{(t)}\Big).
    \]
    
    \item \emph{vocabulary distractor transition}
    \[
    p_{L,3}^{(t)}
    = O\left(d^{-\attn_{L}^{(t)}C_B}\right)
    = O\left(\frac{1}{d^{\frac{C_B}{L-1}(1-\attn^{(t)}_{L})}}\right)\Big(1-p_{L,1}^{(t)}\Big).
    \]
    \end{enumerate}
    
    \item \textbf{Regime II:} if $\attn^{(t)}_{L}\ge 1-\frac{L-1}{C_{B}}$, then
    \begin{enumerate}[(a)]
    \item \emph{in-context distractor transition}
    \[
    p_{L,2}^{(t)}
    = O\left(d^{-\attn_{L}^{(t)}C_B}\right)
    = O\left(\frac{1}{d}\right)\Big(1-p_{L,1}^{(t)}\Big).
    \]
    
    \item \emph{vocabulary distractor transition}
    \[
    p_{L,3}^{(t)}
    = O\left(d^{-\attn_{L}^{(t)}C_B}\right)
    = O\left(\frac{1}{d}\right)\Big(1-p_{L,1}^{(t)}\Big).
    \]
    \end{enumerate}
    \end{enumerate}
    \end{lemma}
    
\subsection{Gradient Lemma}
Since the initialization is uniform, the initial step-wise probabilities satisfy $p^{(0)}_{L,1}=p^{(0)}_{L,2}$, so the gradient characterization in \Cref{prop:grad-char} is not directly applicable. We therefore need finer control of the gradients at the very beginning of training.
\begin{lemma}\label{lem-gd-con-1}
    Assume that \Cref{induction-con} holds for all iterations $<t$, when $1-\attn_{L}^{(t)}\ge \Omega(1)$, we have
    \begin{align*}
    \nabla_{q}\widetilde\cJ^{(t)}_{L}  \geq  \Omega\Big(\frac{\log d}{\dpos d^{\attn_{L}^{(t)}C_B-1}}\Big).
    \end{align*}

\end{lemma}
\begin{proof}
    By \cref{eq-gd-J-q}, we have
    \begin{align}
        \nabla_q \widetilde\cJ_{L}^{(t)}
        &=
       \frac{1}{2L\dpos} \E_{Z^{L}}\left[\fJ^{(t)}(y_0, G^{L})\right] \notag \\
&=\frac{1}{2L\dpos} \E_{Z^{L}}\Bigg[\sum_{\bv\in\cY^{L-1}\times\{y_{L}\}} \Bigg(\prod_{\ell'=1}^{L} \pi^{(t)}_{\theta}\big(v_{\ell'}
\mid v_{\ell'-1}, G^{L}
\big) \Bigg) \Big(\sum_{\ell=1}^{L}\cdot  \attn_{{a,\ell-1} \rightarrow p, \ell}\fG_{\ell}(\bv)\Big)\Bigg]. 
        \end{align}
   Therefore, in the following, we will divide the trajectory $\bv\in\cV^{L-1}\times \{y_L\}$ into three different categories:
    \begin{itemize}
    \item $\bv= \bv^\ast\triangleq(y_1,\cdots,y_{L})$, 
        by \Cref{lem-logit-con},  we can obtain    
        \begin{align*}
            \prod_{\ell'=1}^{L} \pi^{(t)}_{\theta}\big(v_{\ell'}
            \mid v_{\ell'-1}, G^{L}
            \big) =p^{(t)}_{L,1} \ge \Omega(1).
    \end{align*}
Along the trajectory $\bv^{\ast}$, by \Cref{lem-lambda},  we have
    \begin{align}
        \fG_{\ell}(\bv)& = \sum_{j \in \tau(\cY)} \Ecal_{j}^{(t)}\sum_{r\in [m]}\sigma^{\prime}\big(\Lambda^{(t)}_{j,r}\big)\cdot \Big( \dbrack{W_{j,r},Z_{p, \ell}}- \Lambda^{(t)}_{j,r}\Big) \notag \\
                        &=\big(1-p_{L,1}^{(t)}\big)\Big( B- \attn_{L}^{(t)}(B+\sigma_0)\Big)\notag- (L-1)p_{L,2}^{(t)}\Big(-B-\attn_{L}^{(t)}(B+\sigma_0)\Big)\notag\\
                         &~~~~~~-(d-L)p_{L,3}^{(t)}\Big(-B-\sigma_0\Big)\notag\\
                          &\geq \Omega(B)\cdot (1-\attn^{(t)}_{L})\cdot \Big(1-p_{L,1}^{(t)}\Big). \label{eq-fec-1}
    \end{align}
        Therefore, we have
        \begin{align}
        \E_{Z^{L}}\Bigg[\Bigg(\prod_{\ell'=1}^{L} \pi^{(t)}_{\theta}\big(y_{\ell'}
            \mid y_{\ell'-1}, G^{L}
            \big) \Bigg) \Big(\sum_{\ell=1}^{L}  \attn_{{a,\ell-1} \rightarrow p, \ell}\fG_{\ell}(\bv^{\star})\Big)\Bigg] \geq \Omega\Big(\frac{\log d}{\dpos d^{\attn_{L}^{(t)}C_B-1}}\Big).
            \label{eq-gd-main-1}
    \end{align}
      \item On the event $\fE_{L}$, there  exists other trajectories (may more than one) $\bv'\neq \bv^{\ast}$, s.t., 
      $v'_{\ell}=g'_{\ell}(v'_{\ell-1})$ with some $g'_{\ell}\in G^L$ for all $\ell\in [L]$ (letting $v_0=y_0$).  We denote the collection of such $\bv'$ as $\boldsymbol{\cV}'$. In the following, we denote  $g_{i_{\ell}}=g'_{\ell}$.    For $\bv'\in\boldsymbol{\cV}'$, there will be at least one $\tilde \ell$, s.t., $i_{\tilde\ell}\neq {\tilde \ell}$. Thus,   by \Cref{lem-logit-con}, we have 
      \begin{align*}
         \pi_{\theta}^{(t)}\big(v'_{\tilde \ell}
    \mid
   v'_{\tilde\ell-1}, G^{L}
    \big)\leq O(1/L)\cdot \Big(1-p_{L,1}^{(t)}\Big).
      \end{align*}
      Hence, 
             \begin{align}
    & \prod_{\ell'=1}^{L}    \pi_{\theta}^{(t)}\big(v'_{\ell'}
    \mid
   v'_{\ell'-1}, G^{L}
    \big) \leq O(1)\cdot \Big(1-p_{L,1}^{(t)}\Big)\cdot \prod_{\ell'=1}^{L} \pi^{(t)}_{\theta}\big(y_{\ell'}
    \mid y_{\ell'-1}, G^{L}
    \big).
    \end{align}
        Moreover,  along the trajectory $\bv'$,  the analysis is similar as $\bv^{\ast}$, we have 
                 \begin{align}
        &\Bigg|\sum_{j \in \tau(\cY)} \Ecal_{j}^{(t)}\sum_{r\in [m]}\sigma^{\prime}\big(\Lambda^{(t)}_{j,r}\big)\cdot \Big( \dbrack{W_{j,r},Z_{p,\ell}}- \Lambda^{(t)}_{j,r}\Big) \Bigg|\notag \\
                        &\leq \big(1-  \pi_{\theta}^{(t)}\big(g_{i_{ \ell}}( \hat y_{ \ell-1})
                        \mid
                       v'_{\ell-1}, G^{L}
                        \big) \big)\Big( B+ \attn_{L}^{(t)}B+ \sigma_0\Big)\notag\\
                        &~~~~~~+\sum_{\ell'\neq i_{\ell}} \pi_{\theta}^{(t)}\big(g_{ \ell'}( \hat y_{ \ell-1})
                        \mid
                       v'_{\ell-1}, G^{L}
                        \big) \Big(B+\attn_{L}^{(t)}B+ \sigma_0\Big)\notag\\
                         &~~~~~~+\sum_{g\notin G^L}\pi_{\theta}^{(t)}\big(g( \hat y_{ \ell-1})
                         \mid
                        v'_{\ell-1}, G^{L}
                         \big) \Big(B+\sigma_0\Big) \\
                         & \leq O(B). \label{eq-fe-1}
    \end{align}
        Hence, combining \cref{eq-fec-1} \cref{eq-fe-1}  and the probability of $\fE_{L}$ from \Cref{lem-prob-dis} together, it holds that
        \begin{align}
          &  \Bigg| \E_{Z^{L}}\Bigg[\sum_{\bv\in \boldsymbol{\cV}'} \Bigg(\prod_{\ell'=1}^{L} \pi^{(t)}_{\theta}\big(v_{\ell'}
            \mid v_{\ell'-1}, G^{L}
            \big) \Bigg) \Big(\sum_{\ell=1}^{L}\cdot  \attn_{{a,\ell-1} \rightarrow p, \ell}\fG_{\ell}(\bv)\Big)\Bigg]\Bigg| \notag \\
& \leq o(1) \E_{Z^{L}}\Bigg[|\boldsymbol{\cV}'|\Bigg(\prod_{\ell'=1}^{L} \pi^{(t)}_{\theta}\big(y_{\ell'}
\mid y_{\ell'-1}, G^{L}
\big) \Bigg) \Big(\sum_{\ell=1}^{L}\cdot  \attn_{{a,\ell-1} \rightarrow p, \ell}\fG_{\ell}(\bv^{\star})\Big)\Bigg] \notag \\
& \leq o(1)\E_{Z^{L}}\Bigg[\Bigg(\prod_{\ell'=1}^{L} \pi^{(t)}_{\theta}\big(y_{\ell'}
\mid y_{\ell'-1}, G^{L}
\big) \Bigg) \Big(\sum_{\ell=1}^{L}\cdot  \attn_{{a,\ell-1} \rightarrow p, \ell}\fG_{\ell}(\bv^{\star})\Big)\Bigg] , \label{eq-gd-main-2}
    \end{align}
    where the last inequality follows the fact that $|\boldsymbol{\cV}'|\leq L^{L}=O(1)$.
\item for other $\bv\in \Big(\cY^{L-1}\times\{y_{L}\}\Big)\setminus \big(\boldsymbol{\cV}'\cup\{\bv^{\ast}\}\big)$,
there will be at least one $\hat{\ell}\in [L]$, s.t., $v_{\hat \ell}=g(v_{\hat \ell-1})$ with $g\neq \hat \cG$.  
\begin{itemize}
    \item if there exists only one such $\hat{\ell}$, then since the group is simply transitive, there exists another $\hat{\ell}'\neq \hat{\ell}$ s.t., $v_{\hat \ell'}=g(v_{\hat \ell'-1})$ with $g\neq g_{\hat \ell'}$.  By \Cref{lem-logit-con}, for such $\hat \ell$ and $\hat \ell'$ , we have 
    \begin{align}
   &    \pi_{\theta}^{(t)}\big(v_{\hat \ell}
    \mid
  v_{\hat\ell-1}, G^{L}
    \big)\cdot \pi_{\theta}^{(t)}\big(v_{\hat \ell'}
    \mid
  v_{\hat\ell'-1}, G^{L}
    \big) \notag\\
    &\leq O\Bigg(d^{-\attn^{(t)}_{L}C_B}\Bigg)\cdot\Big( 1-p_{L,1}^{(t)}\Big)=O\Bigg(\frac{1}{d^{1+\Omega(1)}}\Bigg)(1-p_{L,1}^{(t)}). \label{eq-other-prob-1}
      \end{align}
      % There will  be at most $O(N)$ such  $\bv$
      \item if there exists exact $k>1$ such $\hat{\ell}$, denoted as $\hat{\ell}_1,\cdots, \hat{\ell}_k$.    By \Cref{lem-logit-con},  we have 
    \begin{align}
   &  \prod_{i=1}^k   \pi_{\theta}^{(t)}\big(v_{\hat \ell_i}
   \mid
 v_{\hat\ell_i-1}, G^{L}
   \big) \notag\\
    &\leq O\Bigg(d^{- (k-1)\attn^{(t)}_{L}C_B}\Bigg) \cdot O\bigg(\frac{1}{d^{\frac{C_B}{L-1} (1-\attn^{(t)}_{L})}}\bigg)\Big( 1-p_{L,1}^{(t)}\Big)\notag\\
    &\leq O\Bigg(\frac{1}{d^{k+\Omega(1)}}\Bigg)(1-p_1^{(t)}). \label{eq-other-prob-2}
      \end{align}
      Here, the last inequality holds since $\big((k-1)x+\frac{1-x}{L-1}\big)C_B$ is monotonically increase for $x\ge \frac{1}{L}$ and the minimum value is $\frac{kC_B}{L}=k+\Omega(1)$.
\end{itemize}
Moreover, we have
         \begin{align}
         \Bigg|\sum_{j \in \tau(\cY)} \Ecal_{j}^{(t)}\sum_{r\in [m]}\sigma^{\prime}\big(\Lambda^{(t)}_{j,r}\big)\cdot \Big( \dbrack{W_{j,r},Z_{p, \ell}}- \Lambda^{(t)}_{j,r}\Big)\Bigg|  
                          & \leq O(B), \label{eq-other-2}
                        %\Big( V_{\tau(y_{\ell}), r_{g_{\ell}\cdot y_{\ell-1}}}(g_{\ell})- {\Lambda}_{j,r}(\tilde{\Zb}^{L,\ell-1})\Big)
    \end{align}
which in turn leads to
                \begin{align}
  & \Bigg|  \E_{Z^{L}}\Bigg[\sum_{\bv\in \Big(\cY^{L-1}\times\{y_{L}\}\Big)\setminus \big(\boldsymbol{\cV}'\cup\{\bv^{\ast}\}\big)} \Bigg(\prod_{\ell'=1}^{L} \pi^{(t)}_{\theta}\big(v_{\ell'}
  \mid v_{\ell'-1}, G^{L}
  \big) \Bigg) \Big(\sum_{\ell=1}^{L}\cdot  \attn_{{a,\ell-1} \rightarrow p, \ell}\fG_{\ell}(\bv)\Big)\Bigg]  \Bigg|\notag\\
   &\leq \sum_{k=1}^{L}\binom{L}{k} O(N^k) O\Big(\frac{1}{d^{k+\Omega(1)}}  \Big)\E_{Z^{L}}\Bigg[\Bigg(\prod_{\ell'=1}^{L} \pi^{(t)}_{\theta}\big(y_{\ell'}
   \mid y_{\ell'-1}, G^{L}
   \big) \Bigg) \Big(\sum_{\ell=1}^{L}\cdot  \attn_{{a,\ell-1} \rightarrow p, \ell}\fG_{\ell}(\bv^{\star})\Big)\Bigg]\notag\\
   &\leq O\Big(\frac{1}{d^{\Omega(1)}}  \Big)\E_{Z^{L}}\Bigg[\Bigg(\prod_{\ell'=1}^{L} \pi^{(t)}_{\theta}\big(y_{\ell'}
   \mid y_{\ell'-1}, G^{L}
   \big) \Bigg) \Big(\sum_{\ell=1}^{L}\cdot  \attn_{{a,\ell-1} \rightarrow p, \ell}\fG_{\ell}(\bv^{\star})\Big)\Bigg]. \label{eq-gd-main-3}
    \end{align}
    \end{itemize}
    Therefore, we put \cref{eq-gd-main-1}, \cref{eq-gd-main-2}, \cref{eq-gd-main-3} together, and thus conclude that 
    \begin{align*}
        \nabla_q \widetilde\cJ_{L}^{(t)}
        &=
       \frac{1}{L\dpos} \E_{Z^{L}}\left[\fJ^{(t)}(y_0, G^{L})\right] \geq  \Omega\Big(\frac{\log d}{\dpos d^{\attn_{L}^{(t)}C_B-1}}\Big).
    \end{align*}
\end{proof}

\begin{lemma}\label{lem-gd-con-3}
    Assume that \Cref{induction-con} holds for all iterations $<t$, when $\attn_{L}^{(t)}\ge 1-\frac{L-1}{C_B}$, we have 
    \begin{align*}
        \nabla_q \widetilde\cJ_{L}^{(t)}
        &=
        \Theta\Big(\frac{\log d}{\dpos d^{\attn_{L}^{(t)}C_B-1}}\Big).
    \end{align*}
\end{lemma}
\begin{proof}
    By \Cref{lem-logit-con}, when $\attn_{L}^{(t)}\ge 1-\frac{L-1}{C_B}$,   we have $\Delta_L^{(t)}, p_{L,1}^{(t)}=\Omega(1)$, and $\delta_L^{(t)}\leq p_{L,2}^{(t)}\leq O(1/d)(1-\Delta_L^{(t)})$. Hence, the condition in \Cref{prop:grad-char} holds, and we can directly apply it to complete the proof.
\end{proof}

\begin{lemma}\label{lem-gd-con-2}
    Assume that \Cref{induction-con} holds for all iterations $<t$.
    Then,
    \[
    \bigl|\nabla_r \widetilde\cJ^{(t)}_{L}\bigr|
    \le O\left(\frac{1}{\dpos}\right)\bigl|\nabla_q \widetilde\cJ_{L}^{(t)}\bigr|.
    \]
    \end{lemma}
    
    \begin{proof}
    The claim follows from \Cref{lem-gd-main}, we thus omit the details.
    \end{proof}
    
\subsection{Proof of Theorem~\ref{thm:rl-constant}}
By combining the gradient bounds in \Cref{lem-gd-con-1,lem-gd-con-3,lem-gd-con-2}, we show that the induction
hypothesis \Cref{induction-con} is maintained throughout this stage until $q^{(t)}$ reaches the target
scale $\Omega\bigl(\log(L/\epsilon)\bigr)$. At that point, we obtain $\attn^{(t)}_{L}\ge 1-\epsilon$,
which leads to the following lemma.

\begin{lemma}[End of Constant-Length Training]\label{lem-warm-end}
For any $\Omega(1/\polylog d)<\epsilon<\frac{L-1}{C_B}$, the induction hypothesis \Cref{induction-con} holds for all iterations
\[
t<T_{1}
= O\left(\dpos d^{(1-\epsilon)C_B-1}\cdot
\frac{\log(L/\epsilon)}{\eta\log d}\right).
\]
Moreover, at $t=T_{1}$ we have: $q^{(T_{1})}\ge \Omega(\log(L/\epsilon))$; $\bigl|r^{(T_{1})}\bigr|\le O\left(\frac{1}{d}\right)q^{(T_{1})}$.
\end{lemma}
\begin{proof}
    Assume \Cref{induction-con} holds up to iteration $t$. Then \Cref{lem-gd-con-1,lem-gd-con-3} imply that the
    policy gradient for $q$ is strictly positive and satisfies a lower bound of the form
    \[
    \nabla_q \widetilde\cJ_L^{(t)}
    \ge
    \Omega\left(\frac{\log d}{\dpos}\right)\cdot
  d^{-(1-\epsilon)C_B+1},
    \]
    where we used that along this stage $\attn_L^{(t)}\le 1-\epsilon$. Under policy gradient update  with step size $\eta$, we therefore have the per-iteration increase
    \[
    q^{(t+1)}-q^{(t)}=\eta\nabla_q \widetilde\cJ_L^{(t)}
    \ge
    \eta\cdot
    \Omega\left(\frac{\log d}{\dpos}\right)\cdot d^{-(1-\epsilon)C_B+1}.
    \]
    Summing over iterations until $q^{(t)}$ reaches $\Omega(\log(L/\epsilon))$ yields
    \[
    T_1
    =
    O\left(
    \frac{\dpos}{\eta\log d}\cdot
   {\log(L/\epsilon)}\cdot
    d^{(1-\epsilon)C_B-1}
    \right),
    \]
    as claimed.
    
    Finally, \Cref{lem-gd-con-2} gives
    $|\nabla_r \widetilde\cJ_L^{(t)}|\le O(\frac{1}{d})\nabla_q \widetilde\cJ_L^{(t)}$ throughout this stage, and hence
    $r^{(t)}$ remains slaved to $q^{(t)}$, i.e.,
    $|r^{(t)}|\le O(\frac{1}{d})q^{(t)}$ for all $t\le T_1$. 
    \end{proof}
\Cref{thm:rl-constant} follows immediately from \Cref{lem-warm-end}.

\subsection{Proof of Theorem~\ref{thm:sft-constant}}
The proof follows the same template as the RL case: we (i) set up an induction hypothesis controlling the key
parameters, (ii) derive lower/upper bounds on the relevant gradients under this hypothesis, and (iii) combine these bounds to upper bound the time needed for $q^{(t)}$ to reach the target scale, at which point the
attention satisfies $\attn^{(t)}_L \ge 1-\epsilon$. In other words, \Cref{thm:sft-constant} is obtained by
assembling the lemmas below.

By comparing \Cref{lem-gd-main} and \Cref{lem-gd-main-sft}, we observe that while the gradient forms are
structurally similar, the supervised analysis is more direct. This simplification arises because we only need
to track the ground-truth trajectory defined by $\bv = (y_1, \dots, y_L)$. Consequently, we can establish an
induction hypothesis analogous to \Cref{induction-con}.
\begin{induction}\label{induction-con-sft}
  For any  length $2\le L\le\polylog d$,  fix any $\Omega(\frac{1}{\polylog d})<\epsilon< \min\{\frac{1}{2}(1-\frac{1}{C_B}), \frac{L-1}{2L}\}$, and let $T_{1}$ be the first iteration such that
    $\attn_{L}^{(t)} \ge 1-\epsilon$.
    Then for every iteration $t<T_{1}$, the following statements hold:
    \begin{enumerate}[(a)]
        \item $O(\log \frac{L}{\epsilon})\ge q^{(t)}\ge 0$, and $q^{(t)}$ is monotonically nondecreasing in $t$ (starting from $0$);
        \item $|r^{(t)}|\le O(1/\dpos) q^{(t)}$.
    \end{enumerate}
    \end{induction}

Building on this induction, we characterize the gradient dynamics through the following lemmas.
\begin{lemma}\label{lem-gd-con-sft}
    Assume that \Cref{induction-con-sft} holds for all iterations $<t$, then we have
    \begin{itemize}
        \item if $\attn_{L}^{(t)}\leq \frac{1}{C_B}$,  then
            \begin{align*}
    -\nabla_{q}\Loss^{(t)}_{L}  \geq   \Omega\Big(\frac{\log d}{\dpos L}\Big).
    \end{align*}
    \item else, we have 
        \begin{align*}
    -\nabla_{q}\Loss^{(t)}_{L}  \geq  \Omega\Big(\frac{\epsilon\log d}{\dpos d^{(1-\epsilon)C_B-1}}\Big).
    \end{align*}
    \end{itemize}
\begin{proof}
    The claim follows from a similar analysis as \cref{eq-fec-1} in the proof of \Cref{lem-gd-con-1}, we thus omit the details.
\end{proof}
\begin{lemma}\label{lem-gd-con-sft-2}
    Assume that \Cref{induction-con-sft} holds for all iterations $<t$.
    Then,
    \[
    \bigl|\nabla_r\Loss^{(t)}_{L}\bigr|
    \le O\left(\frac{1}{\dpos}\right)\bigl|\nabla_q \Loss^{(t)}_{L}\bigr|.
    \]
    \end{lemma}
    
\end{lemma}

By combining the results above, we obtain the total training time required to reach the target attention level:

\begin{lemma}[End of  Training]\label{lem-warm-end-sft}
For any constant length $2\le L\le\polylog d$,  fix any $\Omega(\frac{1}{\polylog d})<\epsilon< \min\{\frac{1}{2}(1-\frac{1}{C_B}), \frac{L-1}{2L}\}$, the induction hypothesis \Cref{induction-con-sft} holds for all iterations
\[
t<T_{1}
= O\left(\dpos d^{(1-\epsilon)C_B-1}\cdot
\frac{\log(L/\epsilon)}{\eta\epsilon\log d}+\frac{L \dpos}{\eta \log d}\right).
\]
Moreover, at $t=T_{1}$ we have: $q^{(T_{1})}\ge \Omega(\log(L/\epsilon))$; $\bigl|r^{(T_{1})}\bigr|\le O\left(\frac{1}{d}\right)q^{(T_{1})}$.
\end{lemma}

\Cref{thm:sft-constant} follows immediately from \Cref{lem-warm-end-sft} by noting that
$q^{(T_1)}=\Omega(\log(L/\epsilon))$ implies $\attn^{(T_1)}_L\ge 1-\epsilon$ (by the definition of $T_1$) and the stated bound on $T_1$ matches the claimed iteration complexity.

\section{Learning Dynamics of Mixed-difficulty RL}\label{sec:mixed}

In this section, we study the mixed-difficulty setting, where tasks of different lengths are interleaved.
By combining the constant-length analysis in \Cref{sec:constant-len} with the gradient characterizations from \Cref{sec:gradient-char}, we analyze two regimes of the difficulty ratio $R$:
(i) the large difficulty ratio regime $R=\omega(1)$, which gives rise to grokking-style dynamics, and
(ii) the moderate difficulty ratio regime $R=O(1)$, which leads to smoother relay dynamics.

We begin by reviewing the mixed-difficulty setup and introducing some timestamps that will be useful for characterizing the overall learning dynamics.

\paragraph{Mixed-difficulty setup.}
Let $R>1$ denote the \emph{difficulty ratio}, and set the starting (effectively short) horizon to be $L_1:=C_B$.
Define the horizon set $\cL_R=\{L_1,L_2,\ldots,L_K\}$ recursively by
\[
L_k=\min\{\lceil R L_{k-1}\rceil,L_{\max}\},\qquad 2\le k\le K,
\]
where $K=\left\lceil \log_R(L_{\max}/L_1)\right\rceil$, so that $L_K=L_{\max}$.
For simplicity, we focus on the case $R\ge 2$ throughout.

\paragraph{Mastery and visible return states.}
For any \(L_i\in\cL_R\), we say the horizon \(L_i\) has \emph{visible return} at time \(t\) if
\begin{equation}
\cJ_{L_i}^{(t)} \ge 0.01.
\end{equation}
Denote the first iteration such that \(L_i\) has visible return as \(T_{\mathsf{vis},i}\).
We say the horizon \(L_i\) is \emph{mastered} at time \(t\) if
\begin{equation}
\cJ_{L_i}^{(t)} \ge 0.99.
\end{equation}
Denote the first iteration such that \(L_i\) is mastered as \(T_{\mathsf{mas},i}\).
\paragraph{Plateau between consecutive horizons.}
For \(k\in\{1,\ldots,K-1\}\), define
\begin{equation}\label{eq:Tk-def}
\cT_{k}
\triangleq  T_{\mathsf{vis},k+1} - T_{\mathsf{mas},k}=
\left|\left\{t \middle| \cJ_{L_{k}}^{(t)} \geq 0.99,\ \cJ_{L_{k+1}}^{(t)} < 0.01\right\}\right|.
\end{equation}
In words, \(\cT_k\) counts the number of iterations during which \(L_k\) is already mastered
while \(L_{k+1}\) has not yet achieved a visible return.

\subsection{Analysis of Large Difficulty Gap Regime}
In this subsection, we analyze the large difficulty ratio regime, where $R=\omega(1)$. Following a similar proof strategy as in \Cref{sec:constant-len},  we start with the induction hypothesis that is expected to hold through the training process.
\begin{induction}\label{induction-mixed-large}
Given $\Omega\big(\frac{1}{\poly \log d}\big)<\epsilon<\frac{1}{4}$, and let $T^{\star}$ be the first iteration such that $\attn_{L_{\max}}^{(t)}\ge 1-\epsilon$. Then, for all iterations $t< T^{\star}$, the following hold:
\begin{enumerate}[(a)]
\item $0\le q^{(t)}\le O\Big(\log\frac{L_{\max}}{\epsilon}\Big)$, and $q^{(t)}$ monotonically increases.
\item $|r^{(t)}|\le O(1/\dpos)q^{(t)}$.
\end{enumerate}
\end{induction}

\subsubsection{Properties of the Attention Scores and Critical Thresholds}
We record some properties of the attention scores and critical thresholds. 

\begin{lemma}\label{lem-attn-large}
    If \Cref{induction-mixed-large} holds for all iterations $<t$,  then we have 
                \begin{enumerate}[(a)]
        \item $\attn_{L}^{(t)}=\frac{e^{q^{(t)}-r^{(t)}}}{e^{q^{(t)}-r^{(t)}}+(L-1)}\ge \frac{1}{L}$; 
         \item  $\attn^{(t)}_{a,\ell-1\to p, k}= \frac{1}{(L-1)+e^{q^{(t)}-r^{(t)}} }=\frac{1}{L-1}\big(1-\attn_{L}^{(t)}\big)$ for $k\neq \ell$.
        \end{enumerate} 
\end{lemma}

\begin{lemma}[Critical threshold of $q$]\label{lemma-mixed-q-critical} 
    If \Cref{induction-mixed-large} holds, then given $L\in\cL_R$,
    the critical threshold of $q$  required to satisfy $\cJ_{L}\ge 1-\xi$ for some constant  $0<\xi\leq 1$ is given by:
\begin{align*}
    q &\ge  \log \frac{L-1}{C_B-1} +f\Big(\frac{\log L - \log \log \frac{1}{1-\xi}}{\log d}\Big)\\
    &\ge \log \frac{L-1}{C_B-1} + \frac{C_B}{C_B-1} \cdot \frac{\log L - \log \log \frac{1}{1-\xi}}{\log d} + \mathcal{O}\left( \frac{\log^2 L}{\log^2 d} \right).
\end{align*}
where $f(x)=\log\Big(\frac{1+x}{1-x/(C_B-1)}\Big)$.
Similarly, the critical threshold of $q$ required to satisfy $\attn_{L}\ge 1-\xi$ for any  $0<\xi\leq 1$ is given by:
\begin{align*}
    q \ge  \log \frac{(1-\xi)(L-1)}{\xi}.
\end{align*}
    \end{lemma}
\begin{proof}
   Given $\cJ_{L}\ge 1-\xi$,  by \Cref{cor:reward-characterization},  we have 
   \begin{align*}
    \Delta_{L}\ge (1-\xi)^{1/L}=1-\frac{-\log (1-\xi)}{L}.
   \end{align*}
   Then, by \Cref{lem:pi-form}, we can derive that 
   $$\attn_L \ge \frac{1}{C_B} + \frac{\log L - \log(-\log(1-\xi))}{C_B \log d}.$$
   Hence, applying \Cref{lem-attn-large}, we have 
   \begin{align*}
    q \ge  \log (L-1) + \log \Big(\frac{\attn_L}{1-\attn_L}\Big)\ge  \log \frac{L-1}{C_B-1} +f\left(\frac{\log L - \log \log \frac{1}{1-\xi}}{\log d}\right).
   \end{align*}
   Here, we then use the first-order Taylor expansion for $f(x)$ to get the second inequality. 
\end{proof}
Notice that in the large difficulty ratio regime, the changes in $\log L$ between two consecutive horizons are $\Omega(\log R)\gg 1$, which is much larger than the $\frac{\log L}{\log d}\le O(1)$ term. Therefore, the above lemma implies that the change in $q$ between two consecutive horizons is dominated by $\Omega(\log R)$.
    \subsubsection{Warm-up Stage for $L_1$}\label{sec:warmup-large}
We define the warm-up stage as the period during which the starting horizon $L_1$
reaches the mastery state, namely $0\le t< T_{\mathsf{mas},1}$.
At initialization, the attention scores are essentially uniform across horizons.
We will show that, during this stage, the only non-negligible gradient contribution
comes from the effectively short horizon $L_1$.

We first record several basic properties of $q$, $r$, and the attention scores
throughout the warm-up stage.

\begin{lemma}\label{lemma-mixed-warmup-q-r}
If \Cref{induction-mixed-large} holds, then for all iterations $0\le t< T_{\mathsf{mas},1}$:
\begin{enumerate}[(a)]
\item $0\le q^{(t)}\le O\left(\frac{L_1}{\log d}\right)$, and $q^{(t)}$ is monotonically increasing in $t$.
\item $|r^{(t)}|\le O\left(\frac{1}{\dpos}\right) q^{(t)}$.
\end{enumerate}
\end{lemma}

\begin{proof}
The range of $q^{(t)}$ is a direct consequence of \Cref{lemma-mixed-q-critical}. The monotonicity of $q^{(t)}$ and the bound on $r^{(t)}$ follow directly from
\Cref{induction-mixed-large}.
\end{proof}

\begin{lemma}\label{lem-attn-large-warmup}
If \Cref{induction-mixed-large} holds, then for all iterations $0\le t< T_{\mathsf{mas},1}$
and for any horizon $L_i$ with $i\ge 2$, we have
\[
\attn^{(t)}_{\ans,\ell-1\to k}\le O\left(\frac{1}{L_i}\right)=o(1),
\qquad
\forall \ell\in [L_i], k\in[\ell].
\]
\end{lemma}

\begin{proof}
This follows directly from \Cref{lem-attn-large}.
Moreover, since $R=\omega(1)$ in this regime, we have $L_i\ge \omega(1)$ for all $i\ge 2$,
so the bound is indeed $o(1)$ as $d$ grows.
\end{proof}

Combining the above with the same reasoning as in \Cref{prop-trap-appendix},
we can verify that the condition \cref{eq:flat-region-cond} in \Cref{prop-flat-gen}
holds for all longer horizons during warm-up, which yields the following.

\begin{lemma}\label{lem-large-warmup-flat}
If \Cref{induction-mixed-large} holds, then for all iterations $0\le t< T_{\mathsf{mas},1}$
and for any horizon $L_i$ with $i\ge 2$, we have
\[
\big|\nabla_{q} \widetilde\cJ_{L_{i}}^{(t)}\big|
\le \tilde{O}\left(\frac{1}{\dpos}\right)\cdot d^{-\Omega(L_{i})},
\qquad
\big|\nabla_{r} \widetilde\cJ_{L_{i}}^{(t)}\big|
\le \tilde{O}\left(\frac{1}{\dpos^2}\right)\cdot d^{-\Omega(L_{i})}.
\]
\end{lemma}

Compared with \Cref{lem-gd-con-1,lem-gd-con-3} in \Cref{sec:constant-len},
\Cref{lem-large-warmup-flat} shows that during warm-up, the gradients contributed by
longer horizons $L_i$ (for $i\ge 2$) are negligible relative to the shortest horizon $L_1=C_B$.
Therefore, we can apply the constant-length analysis from \Cref{sec:constant-len}
to the warm-up stage for $L_1$, which yields the following characterization at the end of warm-up.

\begin{lemma}\label{lem-large-warmup-end}
\Cref{induction-mixed-large} holds through $0\le t< T_{\mathsf{mas},1}$ with
\[
T_{\mathsf{mas},1}= O\left(\frac{K L_{\max} L_1}{\eta \log^2 d}\right),
\]
and at time $T_{\mathsf{mas},1}$ we have
$q^{(T_{\mathsf{mas},1})}\ge \Omega\left(\frac{\log L_1}{\log d}\right)$.
\end{lemma}

\subsubsection{Transition Between Mastery States}\label{sec-mixed-large-transit}
Since we have established that the initial horizon can reach the mastery state,
we next analyze how mastery propagates across \emph{consecutive} horizons.
Specifically, we study the transition from horizon \(i\) to horizon \(i+1\)
over the time interval \([T_{\mathsf{mas},i},\, T_{\mathsf{mas},i+1})\). 

Recall the definition $K=\left\lceil \log_R\big(L_{\max}/L_1\big)\right\rceil$. By construction, the horizons grow by a factor \(R\) up to index \(K-1\),
while the last step may be truncated so that \(L_K=L_{\max}\); consequently, \(L_K/L_{K-1}\) is not necessarily equal to \(R\). For notational convenience, we therefore restrict attention to \(i\in\{1,\ldots, K-2\}\),
and fix an arbitrary \(i^\star\in\{1,\ldots, K-2\}\)
for the remainder of the analysis. Moreover, we absorb the gradient term \(\nabla_q  \mathcal{J}_{L_K}\) into \(\nabla_q \mathcal{J}_{L_{K-1}}\),
since for all times prior to \(T_{\mathsf{mas},K-1}\), \(\nabla_q \mathcal{J}_{L_K}\) can be upper bounded by \(\nabla_q \mathcal{J}_{L_{K-1}}\).

By the critical threshold of $q$ in \Cref{lemma-mixed-q-critical}, we have the following characterization of the attention scores:
\begin{lemma}\label{lemma-mixed-transition-attn}
    If \Cref{induction-mixed-large} holds, then for all iterations $T_{\mathsf{mas},i^\star}\le t< T_{\mathsf{mas},i^\star+1}$:
    \begin{enumerate}[(a)]
    \item if $i^\star>1$, then for any $i<i^\star$, we have
    \begin{align*}
        \attn_{L_{i}}^{(t)}\ge 1-O\Big(\frac{1}{R^{i^\star-i}}\Big)\cdot (1-\attn_{L_{i^\star}}^{(t)})=1-o(1).
    \end{align*}
    \item for  $i=i^\star$, we have
    \begin{align*}
     \frac{1}{C_B}+\Omega\Big(\frac{\log L_i}{\log d}\Big)<   \attn_{L_{i}}^{(t)}\le 1-\Omega\Big(\frac{1}{R}\Big).
    \end{align*}
    \item if $i^\star<K-2$, then for any $i>i^\star+1$, we have
    \begin{align*}
        \attn_{L_{i}}^{(t)}\le O\left(\frac{1}{L_i}\right).
    \end{align*}
    \end{enumerate}
    \end{lemma}
This immediately implies the following characterization of the logits:
\begin{lemma}\label{lemma-mixed-transition-logit}
    If \Cref{induction-mixed-large} holds, then for all iterations $T_{\mathsf{mas},i^\star}\le t< T_{\mathsf{mas},i^\star+1}$:
    \begin{enumerate}[(a)]
    \item if $i^\star>1$, then for any $i<i^\star$, we have $ (p_{L_{i},1}^{(t)})^{L_i}\ge \Omega(\cJ^{(t)}_{L_{i^\star}})=\Omega(1)$, and also
    \begin{align*}
        \Omega\Big(\frac{1}{d^{C_B-1}}\Big) \le  1-p_{L_{i},1}^{(t)}\le O\Big(\frac{1}{d^{(1-{e^{-q^{(t)}} R^{-(i^\star-i)}L_{i^{\star}}})C_B-1}}\Big).
    \end{align*}
    \item for  $i=i^\star$, we have
  \begin{align*}
     1-p_{L_{i},1}^{(t)}\ge \Omega\Big(\frac{1}{d^{(1-\Theta({e^{-q^{(t)}}L_{i}}))C_B-1}}\Big).
    \end{align*}
    \item if $i^\star<K-2$, then for any $i>i^\star+1$, we have
    \begin{align*}
        p_{L_{i},1}^{(t)}\le O\left(\frac{1}{d}\right).
    \end{align*}
    \end{enumerate}
    \end{lemma}
    % \begin{proof}
    %     The proof is straightforward by \Cref{lemma-mixed-q-critical}.
    % \end{proof}
The logit conditions imply that for any $i<i^\star$, we can invoke the gradient characterization in \Cref{prop:grad-char}, and for any $i>i^\star+1$, we can invoke the gradient characterization in \Cref{prop-flat-gen}. Therefore, we have the following characterization of the gradient:

\begin{lemma}\label{lemma-mixed-large-dominance}
    If \Cref{induction-mixed-large} holds, then for all iterations $T_{\mathsf{mas},i^\star}\le t< T_{\mathsf{mas},i^\star+1}$,
\begin{enumerate}[(a)]
    \item if $i^\star>1$, then for any $i<i^\star$, we have
    \begin{align*}
      \Omega\Big(\frac{1}{d^{C_B-1}}\Big)\cdot \frac{\log d}{\dpos} \le   \nabla_q \widetilde\cJ_{L_{i}}^{(t)}\le O\Big(\frac{1}{d^{(1-\frac{L_{i^{\star}}}{e^{q^{(t)}} R^{i^\star-i}})C_B-1}}\Big)\cdot \frac{\log d}{\dpos}
    \end{align*}
    \item for  $i=i^\star$, we have
    \begin{align*}
         \nabla_q \widetilde\cJ_{L_{i}}^{(t)}= \Omega\Big(\frac{1}{d^{(1-\Theta({e^{-q^{(t)}}L_{i}}))C_B-1}}\Big)\cdot \frac{\log d}{\dpos}
      \end{align*}
    \item if $i^\star<K-2$, then for any $i>i^\star+1$, we have
    \begin{align*}
       | \nabla_q \widetilde\cJ_{L_{i}}^{(t)}|\le \tilde{O}\Big(\frac{1}{\dpos}\Big)\cdot d^{-\Omega(L_{i})}.
    \end{align*}
\end{enumerate}
\end{lemma}

\Cref{lemma-mixed-large-dominance} immediately implies a gradient lower bound for $\cJ_{\mathrm{mix}, R}$ during $[T_{\mathsf{mas},i^\star}, T_{\mathsf{mas},i^\star+1})$:
\begin{lemma}\label{lemma-mixed-large-lower-bound}
    If \Cref{induction-mixed-large} holds, then for all iterations $T_{\mathsf{mas},i^\star}\le t< T_{\mathsf{mas},i^\star+1}$, we have
    \begin{align*}
        \nabla_q \widetilde{\cJ}_{\mathrm{mix}, R}^{(t)}\ge  \frac{\log d}{K \dpos}  \Omega\Big(\frac{i^\star}{d^{C_B-1}}\Big).
     \end{align*}
\end{lemma}
\begin{proof}
    By \Cref{lemma-mixed-large-dominance},  when $R\le o(\log d)$, we have 
    \begin{align*}
        \nabla_q \widetilde\cJ_{L_{i^\star}}^{(t)}/  \nabla_q \widetilde\cJ_{L_{i}}^{(t)}\le O\big(d^{1/R}\big)=O(e^{\log d/R})\gg 1.
    \end{align*}
    Thus $\nabla_q \widetilde\cJ_{L_{i^\star}}^{(t)}$ dominates the gradient of short horizons, which leads to the following lower bound:
        \begin{align*}
        \nabla_q \widetilde{\cJ}_{\mathrm{mix}, R}^{(t)}\ge  \frac{\log d}{K \dpos} \Omega\Big(\frac{1}{d^{(1-\frac{1}{R})C_B-1}}\Big).
     \end{align*}
On the other hand, we have 
    \begin{align*}
        \nabla_q \widetilde{\cJ}_{\mathrm{mix}, R}^{(t)}\ge  \frac{1}{K} \sum_{i=1}^{i^\star} \nabla_q \widetilde\cJ_{L_{i}}^{(t)}\ge \Omega\Big(\frac{i^\star \log d}{K \dpos d^{C_B-1}}\Big).
    \end{align*}
Further noting that $i^\star\le K-2\leq O(\log d)$, thus when $R\le o(\log d)$, we have $d^{\frac{C_B}{R}}\ge i^\star$, which implies that in both cases, we have
\begin{align*}
    \nabla_q \widetilde{\cJ}_{\mathrm{mix}, R}^{(t)}\ge  \frac{\log d}{K \dpos} \Omega\Big(\frac{i^\star}{d^{C_B-1}}\Big).
\end{align*}
\end{proof}

So far, we have already controlled the gradient for the horizons before or after the current consecutive mastery state. In the following, we are going to examine $\nabla_q \widetilde\cJ_{L_{i^\star}}^{(t)}+\nabla_r \widetilde\cJ_{L_{i^\star+1}}^{(t)}$.

\begin{lemma}\label{lemma-mixed-large-grok-gd}
    If \Cref{induction-mixed-large} holds, then for all iterations $T_{\mathsf{vis},i^\star+1}\le t< T_{\mathsf{mas},i^\star+1}$, we have 
    \begin{align*}
        \nabla_q \widetilde\cJ_{L_{i^\star}}^{(t)}+\nabla_q \widetilde\cJ_{L_{i^\star+1}}^{(t)}\geq \Omega\Big(\frac{\log d}{L_{i^\star+1}\dpos}\Big)
    \end{align*}
\end{lemma}

\begin{proof}
    By the critical threshold of $q$ in \Cref{lemma-mixed-q-critical}, when $t\ge T_{\mathsf{vis},i^\star+1}$, we have $p_{L_{i^\star+1},1}^{(t)}\ge 1-O(\frac{1}{L_{i^\star+1}})$. Hence, the conditions of \Cref{prop:grad-char} are satisfied, and invoking it, we then obtain 
    \begin{align*}
        \nabla_q \widetilde\cJ_{L_{i^\star+1}}^{(t)}\geq \Omega\Big(\frac{\log d}{\dpos}\Big)(1-p_{L_{i^\star+1},1}^{(t)}).
    \end{align*}
    On the other hand, since $t\le T_{\mathsf{mas},i^\star+1}$, again by  \Cref{lemma-mixed-q-critical}, we have $p_{L_{i^\star},1}^{(t)}\le 1-\Omega(\frac{1}{L_{i^\star}})$. Thus, we have 
    \begin{align*}
        \nabla_q \widetilde\cJ_{L_{i^\star+1}}^{(t)}\geq \Omega\Big(\frac{\log d}{L_{i^\star+1}\dpos}\Big),
    \end{align*}
    which completes the proof.
\end{proof}
In the following, we are going to show that during the time period $[T_{\mathsf{mas},i^\star}, T_{\mathsf{vis},i^\star+1})$, there exists a major period during which the gradient is dominated by the current mastery state $L_{i^\star}$.
\begin{lemma}\label{lemma-mixed-large-plateau-gd}
    If \Cref{induction-mixed-large} holds, then during $[T_{\mathsf{mas},i^\star}, T_{\mathsf{vis},i^\star+1})$, when \begin{align}
        q^{(t)}\in [\Omega(\log R^{0.01}L_{i^\star}), \, O(\log R^{0.99}L_{i^\star})] \label{eq-mixed-q-range}
    \end{align}
      we have 
    \begin{align*}
        \nabla_q \widetilde\cJ_{L_{i^\star}}^{(t)}+\nabla_q \widetilde\cJ_{L_{i^\star+1}}^{(t)}= (1+o(1))\nabla_q \widetilde\cJ_{L_{i^\star}}^{(t)}. 
    \end{align*}

Moreover, 
    \begin{align*}
         \nabla_q \widetilde\cJ_{L_{i^\star}}^{(t)}= \Theta\Big(\frac{1}{d^{(1-e^{-q^{(t)}}L_{i^\star})C_B-1}}\Big)\cdot \frac{\log d}{\dpos}.
      \end{align*}
\end{lemma}
\begin{proof}
    By the critical threshold of $q$ in \Cref{lemma-mixed-q-critical}, 
    $$\Omega(\log L_{i^\star})\le q^{(t)}\le  O(\log RL_{i^\star})=O(\log L_{i^\star+1}).$$
    So the condition \cref{eq-mixed-q-range} is well-defined. Furthermore, by \Cref{lemma-mixed-q-critical}, when \cref{eq-mixed-q-range} holds, we have $p_{L_{i^\star+1},1}^{(t)}\le O(\frac{1}{d})$. Hence applying \Cref{prop-flat-gen}, and  we have 
    \begin{align*}
       \big| \nabla_q \widetilde\cJ_{L_{i^\star+1}}^{(t)}\big| \le \tilde{O}\Big(\frac{1}{\dpos}\Big)\cdot d^{-\Omega(L_{i^\star+1})}.
\end{align*}
Furthermore,  \cref{eq-mixed-q-range} combined with \Cref{lemma-mixed-q-critical} implies that 
\begin{align*}
1-\attn_{L_{i^\star}}^{(t)}=\Theta(e^{-q^{(t)}}L_{i^\star}).
\end{align*}
Hence, 
\begin{align*}
     1-p_{L_{i^\star},1}^{(t)}= \Theta\Big(\frac{1}{d^{(1-e^{-q^{(t)}}L_{i^\star})C_B-1}}\Big). 
  \end{align*}
  Therefore, invoking \Cref{prop:grad-char}, we complete the proof.
\end{proof}

Putting everything together, we can then characterize the grokking-style behaviour happening during the transition period $[T_{\mathsf{mas},i^\star}, T_{\mathsf{mas},i^\star+1})$.

\begin{lemma}\label{lem-large-transit-end}
  \Cref{induction-mixed-large} holds through $[T_{\mathsf{mas},i^\star}, T_{\mathsf{mas},i^\star+1})$, where  $T_{\mathsf{mas},i^\star+1}=T_{\mathsf{mas},i^\star}+{O}\Big(\frac{d^{C_B-1}K\dpos\log R }{\eta i^\star \log d}\Big)$
  \begin{enumerate}[(a)]
    \item the reward of $J_{L_{i^\star+1}}$ saturates below $0.01$ for a time period of 
    \begin{align*}
        \cT_{k}\ge \Omega\Big(\frac{d^{C_B-1}K\dpos}{i^{\star}\eta\log d }\Big) \cdot \frac{\log R}{1+C_B R^{-0.01}\log d }. 
    \end{align*}
    \item $T_{\mathsf{mas},i^\star+1}-T_{\mathsf{vis},i^\star+1}\leq O(\frac{L_{i^\star+1}\dpos K}{\eta\log d})$.  
  \end{enumerate}
\end{lemma}

\begin{proof}
    The existence of $T_{\mathsf{mas},i^\star+1}=T_{\mathsf{mas},i^\star}+{O}\Big(\frac{d^{C_B-1}K\dpos\log R }{\eta i^\star \log d}\Big)$ is guaranteed by the gradient lower bound in \Cref{lemma-mixed-large-lower-bound}. Moreover, the second item is guaranteed by the gradient lower bound in \Cref{lemma-mixed-large-grok-gd}. Then we focus on the first statement. %By \Cref{lemma-mixed-large-plateau-gd}, we have 
    We approximate the total number of iterations $\cT_{i^\star}$ by the integral
    \begin{align*}
        \cT_{i^\star}\gtrsim \int_{\Omega(\log R^{0.01}L_{i^\star})}^{O(\log R^{0.99}L_{i^\star})} \frac{d q }{\eta \nabla_q \widetilde\cJ_{\mathrm{mix}, R}}. 
    \end{align*}
By \Cref{lemma-mixed-large-plateau-gd}, we can have a naive upper bound on the gradient:
\begin{align*}
    \nabla_q \widetilde\cJ_{\mathrm{mix}, R}\le \frac{i^\star}{K}\cdot\nabla_q \widetilde\cJ_{L_{i^\star}}^{(t)}\leq O\Big(\frac{i^\star \log d}{K \dpos }\Big)\cdot \frac{1}{d^{(1-e^{-q^{(t)}}L_{i^\star})C_B-1}}.
\end{align*}
Plugging this into the integral,  we have 
\begin{align*}
    \cT_{i^\star}&\ge \Omega\Big(\frac{d^{C_B-1}K\dpos}{i^{\star}\eta\log d }\Big)\int_{\log R^{0.01}}^{\log R^{0.99}}d^{-C_Be^{-q}} dq =\Omega\Big(\frac{d^{C_B-1}K\dpos}{i^{\star}\eta\log d }\Big) \int_{R^{-0.99}}^{R^{-0.01}}\frac{e^{-(C_B\log d)u}}{u} du\\
    &\ge \Omega\Big(\frac{d^{C_B-1}K\dpos}{i^{\star}\eta\log d }\Big) \cdot e^{-(C_BR^{-0.01}\log d)}\int_{R^{-0.99}}^{R^{-0.01}}\frac{1}{u} du\\
    &\ge \Omega\Big(\frac{d^{C_B-1}K\dpos}{i^{\star}\eta\log d }\Big) \cdot e^{-(C_BR^{-0.01}\log d)}\cdot \log R\\
    &\ge \Omega\Big(\frac{d^{C_B-1}K\dpos}{i^{\star}\eta\log d }\Big) \cdot \frac{\log R}{1+C_B R^{-0.01}\log d },
\end{align*}
where we use the Taylor expansion for the last inequality.
\end{proof}
% \Cref{thm:grokking-1} thus follows immediately from \Cref{lem-large-transit-end}.  

\subsubsection{Proof of Theorem~\ref{thm:grokking-1} and Corollary~\ref{cor-grok}}
\begin{proof}
For any \(k\in\{1,\ldots,K-2\}\), \Cref{lem-large-transit-end} implies
\begin{align*}
\cT_k
&\ge
\Omega \Big(\frac{d^{C_B-1}\,K\,\dpos}{k\,\eta\log d}\Big)\cdot
\frac{\log R}{1+C_B R^{-0.01}\log d} \\
&\ge
\widetilde{\Omega} \Big(\frac{d^{C_B-1}\,\dpos}{\eta\log d}\Big),
\end{align*}
where the last inequality uses
\(\frac{\log R}{1+C_B R^{-0.01}\log d}=\Omega(1/\log d)\)
and \(K/k=\widetilde{\Omega}(1)\).
%This yields \Cref{thm:grokking-1}.

Moreover,
\[
T_{\mathsf{mas},k+1}-T_{\mathsf{vis},k+1}
\le
O \Big(\frac{L_{k+1}\,\dpos\,K}{\eta\log d}\Big)
\le
O \Big(\frac{\dpos^2\,K}{\eta\log d}\Big)
\ll
\cT_k,
\]
since \(\dpos=d^{c_x}\) and \(c_x<C_B-1\).
Therefore,
\begin{align*}
T_{\mathsf{mas},k+1}-T_{\mathsf{mas},k}
\ge
\widetilde{\Omega} \Big(\frac{d^{C_B-1}\,\dpos}{\eta\log d}\Big).
\end{align*}
On the other hand, \Cref{lem-large-transit-end} also gives the matching upper bound
\begin{align*}
T_{\mathsf{mas},k+1}-T_{\mathsf{mas},k}
\le
\widetilde{O} \Big(\frac{d^{C_B-1}\,\dpos}{\eta\log d}\Big).
\end{align*}
Summing over \(k\in\{1,\ldots,K-2\}\), we obtain
\begin{align*}
T_{\mathsf{mas},K-1}-T_{\mathsf{mas},1}
=
\widetilde{\Theta} \Big(\frac{d^{C_B-1}\,\dpos}{\eta\log d}\Big).
\end{align*}

Finally, by \Cref{lem-large-warmup-end}, the time spent in the warm-up stage
is negligible compared to \(T_{\mathsf{mas},K-1}-T_{\mathsf{mas},1}\).
Moreover, the final step can be bounded as
\[
T_{\mathsf{mas},K}-T_{\mathsf{mas},K-1}
\le
O \big(T_{\mathsf{mas},K-1}-T_{\mathsf{mas},K-2}\big).
\]
This completes the proof.
\end{proof}

\subsection{Analysis of Moderate Difficulty Ratio Regime}

In this subsection, we analyze the moderate difficulty ratio regime, where $2\le R=O(1)$. Our overall proof strategy follows that of the large difficulty ratio regime, with several adjustments to account for the smaller gap between two consecutive scales.

We begin by stating the induction hypothesis, which we expect to remain valid throughout training.
\begin{induction}\label{induction-mixed-moderate}
Given $\Omega\big(\frac{1}{\poly \log d}\big)<\epsilon<\frac{1}{4}$, and let $T^{\star}$ be the first iteration such that $\attn_{L_{\max}}^{(t)}\ge 1-\epsilon$. Then, for all iterations $t< T^{\star}$, the following hold:
\begin{enumerate}[(a)]
\item $0\le q^{(t)}\le O\Big(\log\frac{L_{\max}}{\epsilon}\Big)$, and $q^{(t)}$ monotonically increases.
\item $|r^{(t)}|\le O(1/\dpos)q^{(t)}$.
\end{enumerate}
\end{induction}

\subsubsection{Properties of the Attention Scores and Critical Thresholds}
We record several basic properties of the attention scores and the critical thresholds.

\begin{lemma}\label{lem-attn-moderate}
If \Cref{induction-mixed-moderate} holds for all iterations \(<t\), then:
\begin{enumerate}[(a)]
\item
\[
\attn_{L}^{(t)}
=\frac{e^{q^{(t)}-r^{(t)}}}{e^{q^{(t)}-r^{(t)}}+(L-1)}
\ \ge\ \frac{1}{L};
\]
\item for any \(k\neq \ell\),
\[
\attn^{(t)}_{a,\ell-1\to p,k}
=\frac{1}{(L-1)+e^{q^{(t)}-r^{(t)}}}
=\frac{1}{L-1}\big(1-\attn_{L}^{(t)}\big).
\]
\end{enumerate}
\end{lemma}

\begin{lemma}[Critical threshold of \(q\)]\label{lemma-mixed-q-critical-moderate}
If \Cref{induction-mixed-large} holds, then for any \(L\in\cL_R\),
a sufficient threshold on \(q\) for \(\cJ_{L}\ge 1-\xi\),
where \(0<\xi\le 1\) is a constant, is
\begin{align*}
q \ \ge\  \log \frac{L-1}{C_B-1}
+ f\Big(\frac{\log L - \log \log \frac{1}{1-\xi}}{\log d}\Big),
\end{align*}
where \(f(x)=\log\Big(\frac{1+x}{1-x/(C_B-1)}\Big)\).
Similarly, a sufficient threshold on \(q\) for \(\attn_{L}\ge 1-\xi\) is
\[
q \ \ge\ \log \frac{(1-\xi)(L-1)}{\xi}.
\]
\end{lemma}

The above lemmas mirror their counterparts in the large-difficulty regime.
However, to track the variation of the critical threshold when \(L\) increases
only by a constant factor \(R\), we need a more careful comparison than the
coarse Taylor-expansion argument used for widely separated scales.

\begin{lemma}\label{lemma-mixed-q-critical-moderate-careful}
If \Cref{induction-mixed-moderate} holds and \(0<\xi\le 1\) is a constant,
let \(q_{\xi}(L)\) denote the critical threshold of \(q\) required to ensure
\(\cJ_{L,1}\ge 1-\xi\).
Then for any \(L_k, L_{k+1}\in\cL_R\),
\[
q_{\xi}(L_{k+1})-q_{\xi}(L_k)
=
\log R\cdot \Big(1+O(1/\log d)\Big).
\]
\end{lemma}

\begin{proof}
By \Cref{lemma-mixed-q-critical-moderate}, we have
\[
\frac{d q_{\xi}(L)}{dL}
=
\frac{1}{L}
\left[
1
+
\frac{1}{\log d}\cdot
\frac{C_B}{(1+X(L))(C_B-1-X(L))}
\right],
\]
where \(X(L)=\frac{\log L - \log \log \frac{1}{1-\xi}}{\log d}\).
Since \(0\le X(L)\le 1+O(1/\log d)\), it follows that
\(\frac{d q_{\xi}(L)}{dL}=\frac{1}{L}\big(1+O(1/\log d)\big)\).
Therefore,
\begin{align*}
q_{\xi}(L_{k+1})-q_{\xi}(L_k)
&=
\int_{L_k}^{L_{k+1}} \frac{1}{L}\big(1+O(1/\log d)\big)\, dL \\
&=
\log R\cdot \Big(1+O(1/\log d)\Big).
\end{align*}
\end{proof}

\subsubsection{Warm-up Stage for $L_1$}\label{sec:warmup-moderate}
We define the warm-up stage as the period during which the starting horizon $L_1$
reaches the mastery state, namely $0\le t< T_{\mathsf{mas},1}$. The analysis is similar to the large difficulty ratio regime, but with some modifications since $L_{i}$ with $i\ge 2$ could be relatively small and still at the constant-length regime.

\begin{lemma}\label{lemma-mixed-warmup-q-r-moderate}
If \Cref{induction-mixed-moderate} holds, then for all iterations $0\le t< T_{\mathsf{mas},1}$:
\begin{enumerate}[(a)]
\item $0\le q^{(t)}\le O\left(\frac{L_1}{\log d}\right)$, and $q^{(t)}$ is monotonically increasing in $t$.
\item $|r^{(t)}|\le O\left(\frac{1}{\dpos}\right) q^{(t)}$.
\end{enumerate}
\end{lemma}

\begin{lemma}\label{lem-attn-moderate-warmup}
If \Cref{induction-mixed-moderate} holds, then for all iterations $0\le t< T_{\mathsf{mas},1}$, and for any $L_i\in\cL_R$ with $i\ge 2$, we have
\begin{enumerate}[(a)]
\item if $L_i=O(1)$, then $1-C_B\attn^{(t)}_{L_i}\ge 1-\frac{1}{0.99 R^{i-1}+1}$;
\item else, $\attn^{(t)}_{L_i}\le O\left(\frac{1}{L_i}\right)=o(1)$.
\end{enumerate}
\end{lemma}

\begin{proof}
The second item is similar to the large difficulty ratio regime. For the first item, by \Cref{lemma-mixed-warmup-q-r-moderate}, we have
\begin{align*}
    \attn^{(t)}_{L_i}\le \frac{1}{(L_i-1)\cdot e^{-O(L_1/\log d)}+1}=\frac{1}{R^{i-1}C_B\cdot e^{-O(L_1/\log d)}+1}.
\end{align*}
Hence,
\begin{align*}
    1-C_B\attn^{(t)}_{L_i}\ge 1-\frac{1}{0.99 R^{i-1}+1}.
\end{align*}
\end{proof}

Similarly as \Cref{lem-large-warmup-flat}, the condition \cref{eq:flat-region-cond} in \Cref{prop-flat-gen} holds for $L_i=\omega(1)$, which yields the following.

\begin{lemma}\label{lem-moderate-warmup-flat}
If \Cref{induction-mixed-moderate} holds, then for all iterations $0\le t< T_{\mathsf{mas},1}$
and for any horizon $L_i=\omega(1)$, we have
\[
\big|\nabla_{q} \widetilde\cJ_{L_{i}}^{(t)}\big|
\le \tilde{O}\left(\frac{1}{\dpos}\right)\cdot d^{-\Omega(L_{i})},
\qquad
\big|\nabla_{r} \widetilde\cJ_{L_{i}}^{(t)}\big|
\le \tilde{O}\left(\frac{1}{\dpos^2}\right)\cdot d^{-\Omega(L_{i})}.
\]
\end{lemma}

\Cref{lem-attn-moderate-warmup} shows that even some longer horizons $L_i$  are still at the constant-length regime, its target attention scores are still below $\frac{1}{C_B}$, which means $p^{(t)}_{L_i,1}\le d^{-\Omega(1)}$ is still close to $0$. However, directly applying \Cref{prop-flat-gen} to $L_i=O(1)$ only gives a bound of $d^{-\Omega(L_i)}$, which may be too loose in the constant-length regime. Thus we use a variant of \Cref{prop-flat-gen} and precise characterization of $1-C_B\attn^{(t)}_{L_i}$ to get a more precise bound.
\begin{lemma}\label{lem-moderate-warmup-flat-2}
    If \Cref{induction-mixed-moderate} holds, then for all iterations $0\le t< T_{\mathsf{mas},1}$
    and for any horizon $L_i=O(1)$ with $i\ge 2$, we have
    \[
    \big|\nabla_{q} \widetilde\cJ_{L_{i}}^{(t)}\big|
    \le \tilde{O}\left(\frac{1}{\dpos}\right)\cdot d^{-L_{i}(1-\frac{1}{0.99 R^{i-1}+1})},
    \qquad
    \big|\nabla_{r} \widetilde\cJ_{L_{i}}^{(t)}\big|
    \le \tilde{O}\left(\frac{1}{\dpos^2}\right)\cdot d^{-L_{i}(1-\frac{1}{0.99 R^{i-1}+1})}.
    \]
    \end{lemma}
Since $L_i\ge R^{i-1}C_B$, we have $L_{i}(1-\frac{1}{0.99 R^{i-1}+1})\ge C_B+\Omega(1)$. Compared with \Cref{lem-gd-con-1} in \Cref{sec:constant-len},
\Cref{lem-moderate-warmup-flat} and \Cref{lem-moderate-warmup-flat-2} show that during warm-up the gradients contributed by
longer horizons $L_i$ (for $i\ge 2$) are negligible relative to the shortest horizon $L_1=C_B$.
Therefore, we can apply the constant-length analysis from \Cref{sec:constant-len}
to the warm-up stage for $L_1$, which yields the following characterization at the end of warm-up.

\begin{lemma}\label{lem-moderate-warmup-end}
\Cref{induction-mixed-moderate} holds through $0\le t< T_{\mathsf{mas},1}$ with
\[
T_{\mathsf{mas},1}= O\left(\frac{K L_{\max} L_1}{\eta \log^2 d}\right),
\]
and at time $T_{\mathsf{mas},1}$ we have
$q^{(T_{\mathsf{mas},1})}\ge \Omega\left(\frac{\log L_1}{\log d}\right)$.
\end{lemma}

\subsubsection{Transition between mastery states}\label{sec-mixed-moderate-transit}
In this part, we analyze the transition of the mastery state across consecutive horizons.
Concretely, we focus on the time interval
\([T_{\mathsf{mas},i},\, T_{\mathsf{mas},i+1})\)
for \(i\in\{1,\dots,K-2\}\).
As before, we fix an arbitrary \(i^\star\in\{1,\dots,K-2\}\) for the remainder of the analysis (The restriction \(i\le K-2\) excludes the final truncated step where
\(L_K/L_{K-1}\) is not necessarily \(R\).)

By the critical threshold of $q$ in \Cref{lemma-mixed-q-critical-moderate}, we have the following characterization of the attention scores:
\begin{lemma}\label{lemma-mixed-transition-attn-moderate}
    If \Cref{induction-mixed-moderate} holds, then for all iterations $T_{\mathsf{mas},i^\star}\le t< T_{\mathsf{mas},i^\star+1}$:
    \begin{enumerate}[(a)]
    \item if $i^\star>1$, then for any $i<i^\star$, we have
    \begin{align*}
        \attn_{L_{i}}^{(t)}-\attn_{L_{i^\star}}^{(t)}\ge \Omega(1).%\ge 1-O\Big(\frac{1}{R^{i^\star-i}}\Big)\cdot (1-\attn_{L_{i^\star}}^{(t)})=1-o(1).
    \end{align*}
    \item for  $i=i^\star$, we have
    \begin{align*}
     \frac{1}{C_B}+\Omega\Big(\frac{\log L_i}{\log d}\Big)<   \attn_{L_{i}}^{(t)}\le 1-\Omega(1 ).
    \end{align*}
    \item if $i=i^\star+1$, we have
    \begin{align*}
        \frac{1}{RC_B}\le   \attn_{L_{i}}^{(t)}\le \frac{1}{C_B}+O\Big(\frac{\log L_i}{\log d}\Big).
       \end{align*}
    \item if $i^\star<K-2$, then for any $i>i^\star+1$, we have
 \begin{itemize}
        \item if $L_i=O(1)$, then $1-C_B\attn^{(t)}_{L_i}\ge 1-\frac{1}{0.99 R^{i-i^\star-1}+1}$;
        \item else, $\attn^{(t)}_{L_i}\le O\left(\frac{1}{L_i}\right)=o(1)$.
        \end{itemize}
    \end{enumerate}
    \end{lemma}
This immediately implies the following characterization of the logits:
\begin{lemma}\label{lemma-mixed-transition-logit-moderate}
    If \Cref{induction-mixed-moderate} holds, then for all iterations $T_{\mathsf{mas},i^\star}\le t< T_{\mathsf{mas},i^\star+1}$:
    \begin{enumerate}[(a)]
    \item if $i^\star>1$, then for any $i<i^\star$, we have $ (p_{L_{i},1}^{(t)})^{L_i}\ge \Omega(\cJ^{(t)}_{L_{i^\star}})=\Omega(1)$, and also
    \begin{align*}
   1-p_{L_{i},1}^{(t)}\le d^{-\Omega(1)}\Big(1-p_{L_{i^\star},1}^{(t)}\Big).
    \end{align*}
    \item for  $i=i^\star$, we have
  \begin{align*}
     1-p_{L_{i},1}^{(t)}= \Theta\Big({d^{-\big(1-\frac{1}{e^{q^{(t)}}/(L_i-1)+1}\big)C_B+1}}\Big)\ge \Omega(d^{-(1-\Omega(1))C_B+1}).
    \end{align*}
    \item if $i=i^\star+1$, we have
    \begin{align*}
p_{L_{i},1}^{(t)}/p_{L_{i},2}^{(t)}\ge d^{\Omega(1)}.
    \end{align*}
    \item if $i^\star<K-2$, then for any $i>i^\star+1$, we have
    \begin{itemize}
        \item if $L_i=O(1)$, then $p_{L_{i},1}^{(t)}\le O\bigg(d^{-\big(1-\frac{1}{0.99 R^{i-i^\star-1}+1}\big)}\bigg)$;
        \item else, $p_{L_{i},1}^{(t)}\le O\Big(1/d\Big)$.
    \end{itemize}
    \end{enumerate}
    \end{lemma}
    % \begin{proof}
    %     The proof is straightforward by \Cref{lemma-mixed-q-critical}.
    % \end{proof}
    The logit conditions also guarantee that for any $i\le i^\star+1$, we can invoke the gradient characterization in \Cref{prop:grad-char}, and for any $i>i^\star+1$, we can invoke the gradient characterization in \Cref{prop-flat-gen} and the variant in \Cref{lem-moderate-warmup-flat-2}. Therefore, we have the following characterization of the gradient:

\begin{lemma}\label{lemma-mixed-moderate-dominance}
    If \Cref{induction-mixed-large} holds, then for all iterations $T_{\mathsf{mas},i^\star}\le t< T_{\mathsf{mas},i^\star+1}$,
\begin{enumerate}[(a)]
    \item if $i^\star>1$, then for any $i<i^\star$, we have
    \begin{align*}
 \frac{\log d}{\dpos}\cdot \frac{1}{d^{C_B-1}}     \le \nabla_q \widetilde\cJ_{L_{i}}^{(t)}\le  d^{-\Omega(1)} \nabla_q \widetilde\cJ_{L_{i^\star}}^{(t)}. 
    \end{align*}
    \item for  $i=i^\star$, we have
    \begin{align*}
        \nabla_q \widetilde\cJ_{L_{i^\star}}^{(t)}= \Theta(1-p_{L_i,1}^{(t)})\cdot \frac{\log d}{\dpos}.
    \end{align*}
    \item if $i=i^\star+1$, we have
    \begin{align*}
        \nabla_q \widetilde\cJ_{L_{i^\star+1}}^{(t)}= \Theta\Bigg(\Big(1-{p_{L_{i^\star+1},1}^{(t)}}\Big)\Big(p_{L_{i^\star+1},1}^{(t)}\Big)^{L_{i^\star+1}}\Bigg)\cdot \frac{\log d}{\dpos}.
    \end{align*}
    \item if $i^\star<K-2$, then for any $i>i^\star+1$, we have
    \begin{align*}
       | \nabla_q \widetilde\cJ_{L_{i}}^{(t)}|\le d^{-\Omega(1)} \nabla_q \widetilde\cJ_{L_{i^\star}}^{(t)}.
    \end{align*}
\end{enumerate}
\end{lemma}

Thus, to control the gradient of $\cJ_{\mathrm{mix}, R}$ during $[T_{\mathsf{mas},i^\star}, T_{\mathsf{mas},i^\star+1})$, we only need to focus on the gradient of $L_{i^\star}$ and $L_{i^\star+1}$. Similar to \Cref{lemma-mixed-large-grok-gd}, we have the following lower bound for the period that the reward of $L_{i^\star+1}$ becomes visible.

\begin{lemma}\label{lemma-mixed-moderate-grok-gd}
    If \Cref{induction-mixed-moderate} holds, then for all iterations $T_{\mathsf{vis},i^\star+1}\le t< T_{\mathsf{mas},i^\star+1}$, we have 
    \begin{align*}
        \nabla_q \widetilde\cJ_{L_{i^\star}}^{(t)}+\nabla_q \widetilde\cJ_{L_{i^\star+1}}^{(t)}\geq \Omega\Big(\frac{\log d}{L_{i^\star+1}\dpos}\Big).
    \end{align*}
\end{lemma}

Now we turn to the period $[T_{\mathsf{mas},i^\star}, T_{\mathsf{vis},i^\star+1})$. The main difference from the large difficulty ratio regime is that in this stage, $L_{i^\star}$ and $L_{i^\star+1}$ will jointly decide a gradient lower bound for $\cJ_{\mathrm{mix}, R}$, which is significantly larger than the one in the long-plateau stage in the large difficulty ratio regime.
\begin{lemma}\label{lemma-mixed-moderate-plateau-gd}
    If \Cref{induction-mixed-moderate} holds, then during $[T_{\mathsf{mas},i^\star}, T_{\mathsf{vis},i^\star+1})$, we have 
    \begin{align*}
        \nabla_q \widetilde\cJ_{L_{i^\star}}^{(t)}+\nabla_q \widetilde\cJ_{L_{i^\star+1}}^{(t)}\ge \Omega \Big({d^{-\frac{RC_B}{R+C_B}+1}}\Big)\cdot \frac{\log d}{\dpos}.
    \end{align*}
\end{lemma}

\begin{proof}
    Notice that during $[T_{\mathsf{mas},i^\star}, T_{\mathsf{vis},i^\star+1})$, by \Cref{lemma-mixed-moderate-dominance}, $\nabla_q \widetilde\cJ_{L_{i^\star}}^{(t)}$ is dominated by the term  $1-p_{L_{i^\star},1}^{(t)}$. On the other hand, for $L_{i^\star+1}$, by \Cref{lemma-mixed-transition-attn-moderate} and \Cref{lemma-mixed-moderate-dominance}, we have $\nabla_q \widetilde\cJ_{L_{i^\star+1}}^{(t)}$ firstly dominated by the term $(p_{L_{i^\star+1},1}^{(t)})^{L_{i^\star+1}}$, which will increase as $q$ increases, and then by the term  $(1-p_{L_{i^\star+1},1}^{(t)})$ when 
$(p_{L_{i^\star+1},1}^{(t)})^{L_{i^\star+1}}$ reaches the constant level, and $(1-p_{L_{i^\star+1},1}^{(t)})$ is lower bounded by $\Omega(1/L_{i^\star+1})$. Therefore, to lower bound the gradient summation, we only need to find the time $t$ when $(p_{L_{i^\star+1},1}^{(t)})^{L_{i^\star+1}}$ reaches the same level as $1-p_{L_{i^\star},1}^{(t)}$. Thus, consider 
$$C_B \cdot \attn_{L_{i^\star}} - 1 = L_{i^\star+1} (1 - C_B \cdot \attn_{L_{i^\star+1}}),$$
which can be rewritten as
\begin{align}
    C_B \frac{e^q}{e^q + L_{i^{\star}} - 1} - 1 = L_{i^\star+1} \left( 1 - C_B \frac{e^q}{e^q + L_{i^{\star}+1} - 1} \right)\label{eq-mixed-moderate-grok-eq}
    \end{align}
which is a quadratic equation in $e^q$.  Denoting $W_L(x)=C_B \frac{x}{x + L - 1}$, then solving \cref{eq-mixed-moderate-grok-eq} is equivalent to finding the solution $x^{\star}$ of  $W_{L_{i^\star}}(x^{\star}) - 1 = L_{i^\star+1}(1 - W_{L_{i^\star+1}}(x^\star))$.  Consider the point $x_0=\frac{L_{i^\star+1}-1}{C_B-1}$. Note that $1-W_{L_{i^\star+1}}(x_0)=0$ and $W_{L}(x)$ is monotonically increasing. Thus $x^{\star}<x_0$. Hence,
\begin{align*}
    W_{L_{i^\star}}(x^{\star}) - 1 \leq W_{L_{i^\star}}(x_0) - 1 = \frac{C_B \frac{L_{i^\star+1}-1}{L_{i^\star}-1}}{\frac{L_{i^\star+1}-1}{L_{i^\star}-1}+(C_B-1)}-1. 
\end{align*}
Hence,
\begin{align*}
    \nabla_q \widetilde\cJ_{L_{i^\star}}^{(t)}+\nabla_q \widetilde\cJ_{L_{i^\star+1}}^{(t)}\ge \Omega \Big({d^{ W_{L_{i^\star}}(x^{\star}) - 1 }}\Big)\cdot \frac{\log d}{\dpos}\ge  \Big({d^{ -C_B(1-\frac{C_B}{R+C_B}) - 1 }}\Big)\cdot \frac{\log d}{\dpos}.
\end{align*}

\end{proof}

Putting everything together, we can then characterize the relay behaviour happening during the transition period $[T_{\mathsf{mas},i^\star}, T_{\mathsf{mas},i^\star+1})$.

\begin{lemma}\label{lemma-mixed-moderate-trans-end}
  \Cref{induction-mixed-large} holds through $[T_{\mathsf{mas},i^\star}, T_{\mathsf{mas},i^\star+1})$, where  $T_{\mathsf{mas},i^\star+1}=T_{\mathsf{mas},i^\star}+{O}\Big(\frac{d^{C_B-1}K\dpos\log R }{\eta i^\star \log d}\Big)$
  \begin{enumerate}[(a)]
    \item the reward of $J_{L_{i^\star+1}}$ only saturates below $0.01$ for a time period of at most 
    \begin{align*}
        \cT_{k}\le O\Big(\frac{d^{\frac{RC_B}{R+C_B}-1}K\dpos}{\eta\log d} \Big) \cdot \log R. 
    \end{align*}
    \item $T_{\mathsf{mas},i^\star+1}-T_{\mathsf{vis},i^\star+1}\leq O(\frac{L_{i^\star+1}\dpos K}{\eta\log d})$.  
  \end{enumerate}
\end{lemma}
\begin{proof}
    The proof is straightforward by \Cref{lemma-mixed-moderate-grok-gd} and \Cref{lemma-mixed-moderate-plateau-gd} and the fact that $q^{(t)}$ changes $\Theta(\log R)$ during $[T_{\mathsf{mas},i^\star}, T_{\mathsf{vis},i^\star+1})$ and $O(1)$ during $[T_{\mathsf{vis},i^\star+1}, T_{\mathsf{mas},i^\star+1})$ due to \Cref{lemma-mixed-q-critical}.
\end{proof}

\subsubsection{Proof of Theorem~\ref{thm:relay-1} and Corollary~\ref{cor-relay}}
\begin{proof}
\Cref{thm:relay-1} follows immediately from \Cref{lemma-mixed-moderate-grok-gd,lemma-mixed-moderate-trans-end}
together with the bound \(K = O(\log d)\).

For \Cref{cor-relay}, we apply \Cref{lemma-mixed-moderate-trans-end} iteratively
for \(K-2\) transitions (from horizon \(1\) up to horizon \(K-1\)):
\begin{align*}
T_{\mathsf{mas},K-1}-T_{\mathsf{mas},1}
&\le
O\Big(\frac{d^{\frac{RC_B}{R+C_B}-1}\,K\,\dpos}{\eta\log d}\Big)\cdot (K-2)\log R
\;+\;
O\Big(\frac{\dpos K}{\eta\log d}\Big)\cdot
L_1\Big(\frac{R^{K-1}-1}{R-1}\Big) \\
&\le
\tilde{O}\Big(\frac{\dpos}{\eta}\Big)\cdot d^{\frac{RC_B}{R+C_B}-1}
\;+\;
\tilde{O}\Big(\frac{\dpos}{\eta}\Big)\cdot L_{\max},
\end{align*}
where the last inequality uses \(K=O(\log d)\) and \(R^{K-1}=O(L_{\max})\). Combining this with the condition \(L_{\max}=O(d^{c_x})\) and
\(c_x< \frac{2C_B}{2+C_B}\le \frac{RC_B}{R+C_B}\),
we obtain
\begin{align*}
T_{\mathsf{mas},K-1}-T_{\mathsf{mas},1}
\le
\tilde{O}\Big(\frac{\dpos}{\eta}\Big)\cdot d^{\frac{RC_B}{R+C_B}-1}.
\end{align*}
Finally, by \Cref{lem-moderate-warmup-end}, the time spent in the warm-up stage
is negligible compared to \(T_{\mathsf{mas},K-1}-T_{\mathsf{mas},1}\).
Moreover, we can bound the final step by
\(T_{\mathsf{mas},K}-T_{\mathsf{mas},K-1}
\le O\big(T_{\mathsf{mas},K-1}-T_{\mathsf{mas},K-2}\big)\).
This completes the proof.
\end{proof}

\end{document}